\documentclass[10pt,twocolumn,letterpaper]{article}
\usepackage[letterpaper,top=0.75in,bottom=0.85in,left=0.75in,right=0.75in,columnsep=0.4in]{geometry}
\usepackage[T1]{fontenc}
\usepackage{newtxtext}
\usepackage{courier}
\usepackage{amsmath,amsthm}
\usepackage{newtxmath}
\usepackage{graphicx}
\graphicspath{{figures/}}
\usepackage{booktabs}
\usepackage{array}
\newcolumntype{L}[1]{>{\raggedright\arraybackslash}p{#1}}
\usepackage[font=small,labelfont=bf]{caption}
\usepackage{algorithm}
\usepackage{algorithmic}
\usepackage{natbib}
\usepackage[hyphens]{url}
\usepackage{titlesec}
\titleformat{\section}{\large\bfseries\raggedright}{\thesection}{0.6em}{}
\titleformat{\subsection}{\normalsize\bfseries\raggedright}{\thesubsection}{0.6em}{}
\titlespacing*{\section}{0pt}{2.0ex plus 0.6ex minus 0.3ex}{1.0ex plus 0.3ex}
\titlespacing*{\subsection}{0pt}{1.6ex plus 0.5ex minus 0.3ex}{0.8ex plus 0.2ex}
\titlespacing*{\paragraph}{0pt}{1.0ex plus 0.4ex minus 0.2ex}{0.8em}
\renewenvironment{abstract}{\section*{Abstract}\small}{\par\vskip 0.5ex}
\usepackage[hidelinks]{hyperref}
\hypersetup{pdftitle={Measured or Real? When Model-Derived Measurements Can Be Corrected or Compared},
  pdfauthor={Yifeng Li, Ziang Ni},bookmarksnumbered=true}
\newcommand{\rhat}{\hat{\rho}}
\newcommand{\robs}{r_{\mathrm{obs}}}
\newcommand{\rplug}{\hat r_{\mathrm{plug}}}
\newcommand{\rgate}{\hat r_{\mathrm{gate}}}
\newcommand{\rtrue}{r^{\star}}
\newcommand{\thetastar}{\theta^{\star}}
\newcommand{\qgate}{q_{\mathrm{gate}}}
\newcommand{\qucb}{q_{\mathrm{ucb}}}
\newcommand{\qalpha}{q_{\alpha}}
\newcommand{\MSW}{\mathrm{MSW}}
\newcommand{\MSB}{\mathrm{MSB}}
\newcommand{\ICC}[1]{\mathrm{ICC}(#1)}
\newcommand{\clip}{\mathrm{clip}}
\newcommand{\E}{\mathbb{E}}
\newcommand{\Certified}{\textsc{certified}}
\newcommand{\ProtDep}{\textsc{protocol-dependent}}
\newcommand{\Refused}{\textsc{refused}}
\newcommand{\Declined}{\textsc{declined}}
\theoremstyle{plain}
\newtheorem{proposition}{Proposition}
\newtheorem{lemma}{Lemma}
\theoremstyle{remark}
\newtheorem{remark}{Remark}

\begin{document}
\twocolumn[{\centering
  {\Large\bfseries Measured or Real?\\ When Model-Derived Measurements Can Be Corrected or Compared\par}
  \vskip 1.1em
  {\large Yifeng Li$^{1,*}$, Ziang Ni$^{2,*}$\par}
  \vskip 0.5em
  {\normalsize $^{1}$University of Trento\\ $^{2}$Delft University of Technology\\
    yifeng.li@studenti.unitn.it, ziangni@tudelft.nl\par}
  \vskip 1.8em}]
\begin{NoHyper}\renewcommand\thefootnote{$*$}\footnotetext{These authors contributed equally.}\end{NoHyper}
\begin{abstract}
Numbers computed from trained models, such as an example's difficulty or a model's benchmark lead, change with seed and prompt. We ask when such measurements can be corrected or compared. Noise shrinks correlations, and the textbook correction divides by the square root of two estimated reliabilities. On two CIFAR-100 difficulty scores, it returned correlations above one in 70 of 100 classes. We prove that its mean squared error is infinite even when the correlation and the reliabilities use separate seeds. Our rule corrects only when both reliabilities pass a significance test, then divides by their upper confidence bounds. Its output is always bounded, and its coverage guarantees assume Gaussian noise. It helps once reliability is well estimated, more than halving the error with 500 simulated units. Prompt settings can reverse rankings. On ARC, Qwen2-1.5B trails Llama-3.2-3B by 3.8 points with no examples in the prompt and leads by 16.6 with two. Across four benchmarks, averaging the two settings names one winner for 15 of 16 such reversals. We compare models within each setting, controlling family-wise error asymptotically in the number of items for a fixed prompt grid. Aligning embedding coordinates across seeds raises held-out per-dimension reliability from 0.003 to 0.813.
\end{abstract}
\section{Introduction}\label{sec:intro}

\begin{figure*}[t]
\centering
\includegraphics[width=\textwidth]{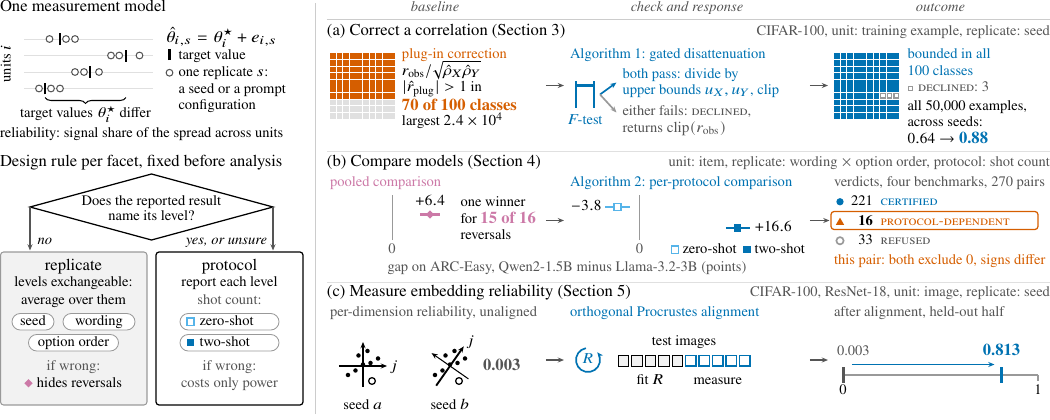}
\caption{In each of the three analyses the baseline fails, and the check in this paper gives an outcome that the data support. Left: all three rest on one measurement model and one design rule. Measuring unit $i$ under replicate $s$ gives its target value $\thetastar_i$ plus replicate noise $e_{i,s}$. A facet is a protocol, reported per level, when the reported result names its level or when we are unsure. Otherwise, if its levels are exchangeable, it is a replicate and is averaged over. (a) The plug-in correction of the correlation between EL2N and dynamic uncertainty, computed within each \mbox{CIFAR-100} class, exceeds 1 in magnitude in 70 of the 100 classes (vermilion cells). Algorithm~\ref{alg:gated} returns a bounded value in all 100 and declines in 3 (open cells). Cells count classes and are not matched between the two grids. This demonstration uses the same two seeds for the correlation and the reliabilities, with the gate at $\alpha{=}0.25$. With the correlation computed across seeds on all 50,000 training examples, Algorithm~\ref{alg:gated} corrects 0.64 to 0.88. (b) On ARC-Easy, the pooled comparison puts \mbox{Qwen2-1.5B} ahead of \mbox{Llama-3.2-3B}, while the simultaneous 95\% intervals of Algorithm~\ref{alg:perproto} exclude zero with opposite signs in the two protocols. Intervals are drawn to scale. The counts of reversals and verdicts cover 270 pairs on four benchmarks. (c) Without alignment, dimension $j$ of one seed need not match dimension $j$ of another. Orthogonal Procrustes alignment of \mbox{ResNet-18} features across seeds, fitted on half of the \mbox{CIFAR-100} test images, raises per-dimension reliability on the other half from 0.003 to 0.813.}
\label{fig:overview}
\end{figure*}

Machine learning papers often report numbers computed from trained models, such as how hard a training example is or how far one model leads another on a benchmark. These model-derived measurements change when we retrain with a new seed or reword the prompt. We ask when such measurements can be corrected for that noise, and when they can be compared.

\paragraph{A correction that returns impossible correlations.}
Data pruning ranks training examples by difficulty scores. EL2N (error L2 norm) \citep{paul2021} measures the size of the prediction error early in training, and dynamic uncertainty \citep{he2024} measures how much the true-class probability fluctuates across epochs. Within each \mbox{CIFAR-100} class, we correlate the two scores across about 500 examples, each averaged over two training seeds. Seed noise pulls this correlation toward zero. Spearman's textbook correction \citep{spearman1904} divides it by the square root of the two scores' reliabilities. A score's reliability is the share of the differences between units that would reappear if we retrained with a new seed, and here the units are training examples. Some estimated reliabilities fall near or below zero, down to $-$0.06 for the one-way intraclass correlation $\ICC{1,2}$ of Section~\ref{sec:reliability}. With these estimates plugged in, Spearman's formula (the plug-in correction) returns correlations above one in magnitude in 70 of the 100 classes (Figure~\ref{fig:overview}(a)). The largest is $2.4\times10^4$. Section~\ref{sec:real} gives the details, and Figure~\ref{fig:teaser}(a, b) shows the same failure in simulation.

\paragraph{A ranking that flips with the prompt.}
Our second example ranks language models on the 2,365 four-option test questions of ARC-Easy. We score each model under 16 prompt configurations, four instruction wordings crossed with four orders of the answer options. We do this in two settings, zero-shot and two-shot, where two-shot places the same two solved examples inline before every question. \mbox{Llama-3.2-3B} leads \mbox{Qwen2-1.5B} by 3.8 percentage points zero-shot. With two shots, \mbox{Qwen2-1.5B} leads by 16.6 points. Both gaps stay clear of zero even under simultaneous intervals, which are built to cover every gap on the benchmark at once (Figure~\ref{fig:overview}(b)). Averaging the two settings puts \mbox{Qwen2-1.5B} ahead by 6.4 points and hides the reversal. Across four benchmarks we find 16 such reversals, and a pooled comparison that averages the settings names a single winner for 15 of them.

\paragraph{Both examples ask how much of a result survives a change of replicate.}
Both measure units under replicates that we average over. The units are training examples or benchmark items, and the replicates are seeds or prompt configurations. The shot count is not a replicate, because papers report it by name, as in two-shot accuracy. We call such a setting a protocol and analyze each protocol separately. A correction needs both reliabilities to be clearly above zero, and a winner needs evidence pointing the same way in every protocol. When the data fall short, declining to correct and declining to name a winner are legitimate results. Our design rule is to name the unit, the replicates and the protocols before any correction or comparison (Figure~\ref{fig:overview}).

\paragraph{Contributions.}
\begin{enumerate}\itemsep1pt
\item \textbf{Separate seeds do not make the plug-in correction safe.} On the event that both estimated reliabilities are positive, the plug-in's mean squared error (MSE) is infinite even when the correlation and the reliabilities use separate seeds. The proof needs conditional independence and a density bound on the reliability estimate near zero, and no Gaussian assumption (Proposition~\ref{prop:split}). We respond with gated disattenuation (Algorithm~\ref{alg:gated}). Disattenuation corrects a correlation for the shrinkage that noise causes, and ours corrects only when both reliabilities pass a significance test and then divides by their upper confidence bounds. Its output is always bounded, and its only tuning setting is the significance level. We claim no optimality for it. It helps once reliability is well estimated, and with 500 simulated units its MSE is 0.048, against 0.105 for the uncorrected estimate. Across the settings of Table~\ref{tab:mse}, it is never more than 0.01 above the uncorrected estimate.
\item \textbf{Two uses of the design rule.} For model comparison, the rule leads to per-protocol comparison (Algorithm~\ref{alg:perproto}). Algorithm~\ref{alg:perproto} gives every pair of models a simultaneous interval for its accuracy gap in each protocol and so controls the family-wise error, the chance of any false claim on a benchmark. We call a direction \emph{certified} when, in every protocol, its simultaneous interval excludes zero with the same sign. The family-wise error level is $\alpha$, asymptotically in the number of items, for items drawn independently from an item population and the fixed prompt grid, with every model scored the same way on the same items. It does not extend to unseen prompts or other harnesses. For embeddings, a seed counts as a replicate only if each dimension means the same thing under every seed, so we align coordinates across seeds on held-out data before measuring per-dimension reliability.
\item \textbf{Findings.} One prompt configuration measures item difficulty poorly, with reliability at most 0.65 for every model. Most reversals involve \mbox{Llama-3.2-3B}, and on ARC its two-shot drop depends on the format of the two examples, because giving them as earlier chat turns removes it. For \mbox{ResNet-18} features on \mbox{CIFAR-100}, aligning embedding coordinates raises held-out per-dimension reliability from 0.003 to 0.813.
\end{enumerate}

\paragraph{What is standard and what is new.}
The tools we build on are standard. Correction uses the intraclass correlation and its exact $F$ distribution \citep{shrout1979}, Spearman's correction, a significance pre-test \citep{bancroft1944} and upper confidence bounds. Comparison uses max-$t$ intervals, built from the largest $t$ statistic \citep{hothorn2008}, and generalizability theory (G-theory) \citep{cronbach1972,brennan2001}. Embeddings use Procrustes alignment \citep{schonemann1966} and centered kernel alignment (CKA) \citep{kornblith2019}. Proposition~\ref{prop:known} restates known facts \citep{searle1992,zimmerman1997,kinal1980}. Proposition~\ref{prop:split} is new. We are not aware of a prior statement of this result. Algorithm~\ref{alg:gated} adds an explicit cap on the correction factor to a standard gate and standard upper bounds, and Algorithm~\ref{alg:perproto} adds a separate target for each protocol and three verdicts with measured error rates.

Figure~\ref{fig:overview} and Table~\ref{tab:roadmap} map each problem to its check and response. Readers interested only in model comparison can read Sections~\ref{sec:units}, \ref{sec:protocols} and~\ref{sec:compare}.
\section{Measurements, reliability, and what to average over}\label{sec:framework}

One measurement model covers both running examples. Appendix~\ref{app:notation} collects the notation (Table~\ref{tab:notation}) and the reliability estimate behind each analysis (Table~\ref{tab:analyses}).

\begin{table*}[t]\centering\scriptsize\setlength{\tabcolsep}{3pt}
\begin{tabular}{@{}L{1.8cm}L{3.4cm}L{1.35cm}L{1.8cm}L{1.8cm}L{2.75cm}L{2.6cm}L{0.75cm}@{}}
\toprule
Problem & Example in this paper & Unit & Replicate (protocol) & What we estimate & Check & Response & Sec. \\
\midrule
Noise shrinks a correlation &
EL2N vs.\ dynamic uncertainty per CIFAR-100 class: $|\rplug|>1$ in 70/100 classes &
training example & seed &
correlation of seed-mean scores &
$F$-test gate on each reliability; variance--mean check &
Alg.~\ref{alg:gated}: correct only if both gates pass, then divide by upper bounds &
\ref{sec:disatt} \\\addlinespace[3pt]
Errors from one run correlate &
the two scores from the same seed: error correlation +0.28 &
training example & seed &
as above &
error correlation within a seed &
compute $\robs$ from different seeds &
\ref{sec:real} \\\addlinespace[3pt]
A replicate shifts every unit &
one wording makes all items easier; synthetic case (App.~\ref{app:persist}): within-run error 0.44 vs.\ 0.0007 with independent readouts &
item &
wording $\times$ option order &
paired gap on one shared grid &
a model against itself on two halves of the grid: 3.6\% false winners (pooled 18.7\%) &
Alg.~\ref{alg:perproto} on one grid shared by all models &
\ref{sec:harness}, App.~\ref{app:persist} \\\addlinespace[3pt]
The setting changes the answer &
ARC, \mbox{Qwen2-1.5B} minus \mbox{Llama-3.2-3B}: $-$3.8 points zero-shot, +16.6 two-shot &
item &
configuration (protocol: zero- vs.\ two-shot) &
one gap per protocol, $\Delta_{p,\pi}$ &
sign of each protocol's simultaneous interval &
Alg.~\ref{alg:perproto}: three verdicts; pooling names a stable winner for 15/16 reversals &
\ref{sec:compare} \\\addlinespace[3pt]
Coordinates are arbitrary across runs &
ResNet-18 penultimate features on CIFAR-100 &
image (per dimension) & seed &
per-dimension reliability after alignment &
per-dimension reliability before vs.\ after alignment: 0.003 vs.\ 0.813 &
held-out Procrustes alignment; distances; CKA &
\ref{sec:embed} \\
\bottomrule
\end{tabular}
\caption{Each problem needs its own check and response.}
\label{tab:roadmap}
\end{table*}

\subsection{Units, replicates and the target}\label{sec:units}

A model-derived measurement assigns a number to each \emph{unit}, which may be a training example, a benchmark item or an image. The number changes with the \emph{replicate}, one instance of the randomness we average over, such as a training seed. Measuring unit $i$ under replicate $s$ gives
\begin{equation}\label{eq:model}
\hat\theta_{i,s}=\thetastar_i+e_{i,s},\qquad i=1,\dots,N,\quad s=1,\dots,K.
\end{equation}
The \emph{target value} $\thetastar_i$ is the mean of unit $i$'s measurement over the population of replicates, and $e_{i,s}$ is the replicate noise. Different seeds may find different solutions, and the target averages over them. The \emph{estimand} is the population quantity we want to estimate, such as the correlation between the target values of two measurements.

In LLM evaluation the unit is a benchmark item, and scoring is deterministic, so rerunning one prompt gives identical answers. The replicates are therefore prompt \emph{configurations}, each pairing one wording of the instructions with one order of the answer options.

\subsection{Reliability and how we estimate it}\label{sec:reliability}

Reliability measures how much of a measurement's spread across units is signal. Let $\tau^2$ be the variance of the target values across units and $\sigma^2$ the variance of the noise $e_{i,s}$. The \emph{single-replicate reliability} $\rho_1=\tau^2/(\tau^2+\sigma^2)$ is the share of one replicate's variance across units that comes from the target values. Averaging $K$ replicates divides the noise variance by $K$, so the reliability of a $K$-replicate mean follows the Spearman--Brown formula \citep{spearman1910,brown1910},
\begin{equation}\label{eq:rhoK}
\rho_K=\frac{\tau^2}{\tau^2+\sigma^2/K}=\frac{K\rho_1}{1+(K-1)\rho_1}.
\end{equation}

We estimate both reliabilities with the one-way intraclass correlation (ICC), which treats the replicates of each unit as interchangeable and ignores which seed or configuration produced them. The mean square within units, $\MSW$, measures how much the $K$ replicates of one unit disagree. The mean square between units, $\MSB$, measures how much the unit means differ. In the notation of \citet{shrout1979}, the one-way estimates are
\begin{equation}\label{eq:icc}
\begin{aligned}
\rhat_K&=\ICC{1,K}=1-\frac{\MSW}{\MSB},\\
\rhat_1&=\ICC{1,1}=\frac{\MSB-\MSW}{\MSB+(K-1)\MSW}.
\end{aligned}
\end{equation}
These two estimates obey the Spearman--Brown formula exactly, and for $K=2$ the formula inverts to $\rhat_1=\rhat_2/(2-\rhat_2)$. A correction must use a reliability for the same number of replicates as the measurement it corrects. A correlation of single-seed scores needs $\rho_1$, and a correlation of $K$-seed means needs $\rho_K$.

A replicate can also shift every unit at once, as when one prompt wording makes every item easier. To report how reliable a score or an item difficulty is, we use the crossed two-way form of G-theory, which separates such shared shifts from the noise \citep{cronbach1972,brennan2001}. The corrections in Section~\ref{sec:disatt} use the one-way estimates, because the $F$ distribution behind Algorithm~\ref{alg:gated} holds for the one-way model.

\subsection{Replicate or protocol? A rule fixed before analysis}\label{sec:protocols}

Some facets of an evaluation are noise to average over, and others change the question. A facet is a factor of the evaluation design, such as the instruction wording, the option order or the number of solved examples in the prompt. A facet is a replicate when the reported result does not name its level and the levels are exchangeable, meaning any one could stand in for another. Wording and option order are replicates, and the target is their average. A facet is a \emph{protocol} when the result names its level, as in ``zero-shot accuracy'' and ``two-shot accuracy''. Each protocol defines its own estimand and is reported separately. Generalizability theory calls these random and fixed facets \citep{brennan2001}. We fix the assignment before analysis by asking whether the result names the level, so it is never chosen from the data (Figure~\ref{fig:overview}, left).

The two possible mistakes cost different things. Treating the shot count as a replicate gives the pooled comparison, which names one winner for 15 of the 16 reversals in Section~\ref{sec:results}. Separating two protocols that agree costs only power (Table~\ref{tab:pooled}). When unsure, we therefore treat the facet as a protocol.

\subsection{Noise and four other problems}\label{sec:failures}

Let $\robs$ be the observed correlation across units between two measurements $X$ and $Y$, and let the \emph{true correlation} $\rtrue$ be the correlation of their target values. Noise shrinks $\robs$ toward zero, and the plug-in correction of \citet{spearman1904} divides by the estimated reliabilities,
\begin{equation}\label{eq:plugin}
\rplug=\frac{\robs}{\sqrt{\rhat_X\,\rhat_Y}}.
\end{equation}
The correction assumes that the noise averages to zero for every unit, that the errors in $X$ and $Y$ are independent, and that no replicate shifts every unit. Section~\ref{sec:disatt} shows how it fails when an estimated reliability is near zero. Table~\ref{tab:roadmap} lists four further problems, each with its own check.

\paragraph{Errors from one run correlate.} Two scores computed from the same training run share that run's noise, which inflates $\robs$. The cross-seed correction in Section~\ref{sec:real} computes $\robs$ from scores on different seeds.

\paragraph{A replicate shifts every unit.} When one prompt wording makes every item easier, the shift survives averaging over more items and survives dividing by a reliability. Our model comparisons score every model on one grid of configurations fixed in advance, so the target and the guarantee refer to that grid (Section~\ref{sec:harness}). Appendix~\ref{app:persist} gives a synthetic example in which noise shared within one run biases an estimate in a way that no reliability correction can undo.

\paragraph{The setting changes the answer.} Two-shot prompting reverses rankings, so Section~\ref{sec:compare} compares models within each protocol and names a winner only when the protocols agree in sign.

\paragraph{Coordinates are arbitrary across runs.} Two training runs can represent the same information in different coordinates, so dimension $j$ of one run need not match dimension $j$ of another. For the \mbox{ResNet-18} features we study, per-dimension reliability then mostly measures the arbitrary basis, and Section~\ref{sec:embed} aligns the coordinates on held-out data before measuring it.
\section{Correcting a correlation for noise}\label{sec:disatt}

If $\rhat$ is often near zero, $1/\rhat$ is sometimes enormous.
Sections~\ref{sec:plugin} and~\ref{sec:split} show how this breaks the plug-in correction, even with separate seeds, and Sections~\ref{sec:gated}--\ref{sec:whento} give and test Algorithm~\ref{alg:gated}, which corrects only after both reliabilities pass an $F$-test.

\subsection{Dividing by an estimated reliability}\label{sec:plugin}

The plug-in correction~\eqref{eq:plugin} divides $\robs$ by $\sqrt{\rhat_X\rhat_Y}$.
How often this denominator lands near or below zero has an exact answer in the simplest model of replicate noise, the balanced Gaussian one-way model.
In that model each unit has a Gaussian target value, each replicate adds independent Gaussian noise of the same variance, and every unit has $K$ replicates.
The estimate $\rhat_K=\ICC{1,K}$ then obeys the exact pivot
\begin{equation}\label{eq:pivot}
\frac{1-\rhat_K}{1-\rho_K}\sim F_{\nu_1,\nu_2},\quad \nu_1=N(K-1),\ \ \nu_2=N-1,
\end{equation}
where $F_{\nu_1,\nu_2}$ is the $F$ distribution with $\nu_1$ and $\nu_2$ degrees of freedom \citep{shrout1979}.
A pivot is a function of the estimate and the true value whose distribution is known exactly.
Because an $F$ variable can be arbitrarily large, \mbox{$\rhat_K=1-(1-\rho_K)F$} can take any value below one, negative values included.

Proposition~\ref{prop:known} collects three failures of the plug-in.
These facts are known \citep{searle1992,zimmerman1997,kinal1980}. We restate them for the ICC plug-in.

\begin{proposition}[Known finite-sample facts for the ICC plug-in]\label{prop:known}
Part (i) uses the pivot~\eqref{eq:pivot}, part (ii) uses only $\Pr(\rhat_K\le0)>0$, and part (iii) uses only the condition stated in it.
\par\noindent\textup{(i)} $\Pr(\rhat_K\le 0)=1-F_{\mathrm{cdf}}\big(\tfrac{1}{1-\rho_K};\nu_1,\nu_2\big)>0$, so $\sqrt{\rhat_K}$ is undefined as a real number on an event of positive probability.
\par\noindent\textup{(ii)} With floor $\epsilon$, the factor $g_\epsilon=\max(\rhat_K,\epsilon)^{-1/2}$ has $\E[g_\epsilon^2]=\Theta(1/\epsilon)\to\infty$ as $\epsilon\to0$. Its second moment is finite for each fixed $\epsilon>0$ but not uniformly bounded.
\par\noindent\textup{(iii)} Suppose that, for almost every value of $\robs$, the pair $(\rhat_X,\rhat_Y)$ lies in $(0,\eta)^2$ with positive conditional probability for every $\eta>0$. In the balanced Gaussian model this holds whether $\robs$ and the reliabilities use disjoint seed blocks or share replicates (Appendix~\ref{app:proof-known}). Then the events $\{|\rplug|>1,\ \rhat_X\rhat_Y>0\}$ and $\{0<\rhat_X\rhat_Y<\robs^2\}$ coincide, and for almost every value $\robs\neq0$ their conditional probability is positive.
\end{proposition}

The proofs are in Appendix~\ref{app:proof-known}.
On real difficulty scores the plug-in returned correlations above one in magnitude in 70 of 100 \mbox{CIFAR-100} classes (Section~\ref{sec:real}).

\subsection{Splitting seeds does not fix the denominator}\label{sec:split}

A natural remedy is to compute $\robs$ on seeds separate from those used for the reliabilities.
The numerator then carries no information about the noise in the denominator.
Proposition~\ref{prop:split} shows that the second moment of the plug-in stays infinite all the same.

\begin{proposition}[The second moment stays infinite when seeds are split]\label{prop:split}
Let $\robs$ be computed on seeds disjoint from those used for $\rhat_X$ and $\rhat_Y$, and let $\thetastar$ be the vector of target values. Assume
\textup{(H1)} given $\thetastar$, $\robs$ is independent of $(\rhat_X,\rhat_Y)$;
\textup{(H2)} there is a set $\Theta_0$ with $\Pr(\thetastar\in\Theta_0)>0$ such that, for every $\thetastar\in\Theta_0$, $\E[\robs^2\mid\thetastar]>0$, and for some $c,\delta>0$ (which may depend on $\thetastar$), $\Pr(\rhat_X\in(a,b),\,\rhat_Y>0\mid\thetastar)\ge c\,(b-a)$ whenever $0<a<b<\delta$;
\textup{(H3)} $\rhat_Y<1$ almost surely.
Then $\E\big[\rplug^{\,2}\,\mathbf 1\{\rhat_X>0,\rhat_Y>0\}\big]=\infty$.
\end{proposition}

\emph{Proof sketch.}
On $\{\rhat_X>0,\rhat_Y>0\}$ we have $\rplug^{\,2}\ge\robs^2/\rhat_X$, by (H3).
Given $\thetastar\in\Theta_0$, (H1) factorizes the conditional expectation, and (H2) bounds it below by $\E[\robs^2\mid\thetastar]\int_0^\delta c/t\,dt=\infty$.
Tonelli's theorem gives the unconditional claim (Appendix~\ref{app:proof-split}).

We are not aware of a prior statement of this result.
Because the true correlation lies in $[-1,1]$, an infinite second moment on this event means an infinite mean squared error there, the form stated in Section~\ref{sec:intro}.
In words, (H1) says that once the target values are fixed, the correlation carries no information about the reliability estimates.
(H2) says that, while the estimate for $Y$ stays positive, the estimate for $X$ lands in any small interval just above zero with probability at least proportional to the width of that interval.
The proposition needs no Gaussian assumption and allows $\rhat_X$ and $\rhat_Y$ to be dependent.
In the balanced Gaussian model with three disjoint seed blocks, one each for $\robs$, $\rhat_X$ and $\rhat_Y$, conditions (H1)--(H3) all hold (Appendix~\ref{app:proof-split} also gives a general sufficient condition for (H3)).
There, with floor $\epsilon$ on each reliability, the second moment of the plug-in grows like $\epsilon^{-2}$ (Appendix~\ref{app:proof-split}, with a synthetic check in Appendix~\ref{app:interval}).
A just-identified instrumental-variable (IV) estimator divides by an estimated first-stage coefficient and has no finite moments \citep{kinal1980}, which matters in practice when the first stage is weak \citep{staiger1997,bound1995}.
Split-sample IV also estimates the first stage on a separate sample \citep{angrist1995}.
Proposition~\ref{prop:split} is the analogous statement for the ICC plug-in.

\paragraph{The failure is a small-sample, low-reliability effect.}
At any fixed $\rho>0$ the plug-in is consistent and asymptotically normal, and $\Pr(\rhat\le0)$ decays exponentially in $N$.
Proposition~\ref{prop:split} holds at every finite $N$ and concerns the unclipped plug-in.
The event that drives it, an estimated reliability just above zero, also has a probability that decays exponentially in $N$, so the failure matters in practice at small $N$ and low reliability.
Clipping to $[-1,1]$ bounds the error but pulls estimates toward $\pm1$, and in the base row of Table~\ref{tab:full} the clipped plug-in has MSE 0.52, against 0.12 uncorrected (Appendix~\ref{app:baselines} evaluates a positive-part rule and other baselines).

\subsection{Gated disattenuation (Algorithm~\ref{alg:gated})}\label{sec:gated}

\begin{table*}[t]\centering\scriptsize\setlength{\tabcolsep}{3pt}
\begin{tabular}{@{}L{1.95cm}L{3.65cm}L{3.85cm}L{3.75cm}L{3.7cm}@{}}
\toprule
Result & Needs & Gives & Measured here & Not covered \\
\midrule
Prop.~\ref{prop:known} (known) & (i) Gaussian one-way model; (ii) $\Pr(\rhat\le0)>0$; (iii) $(\rhat_X,\rhat_Y)\in(0,\eta)^2$ with positive prob.\ given $\robs$ & closed form of $\Pr(\rhat_K\le0)>0$; floored second moment $\Theta(1/\epsilon)$; $\Pr(|\rplug|>1)>0$ & plug-in exceeds 1 in 70/100 CIFAR-100 classes (same seeds) & -- \\\addlinespace[2pt]
Prop.~\ref{prop:split} (new) & (H1)--(H3); no Gaussian assumption; $\rhat_X,\rhat_Y$ may be dependent & infinite second moment with $\robs$ on separate seeds; floored rate $\Theta(\epsilon^{-2})$ in the Gaussian disjoint-seed design & synthetic check (App.~\ref{app:interval}) & whether real scores meet (H2) \\\addlinespace[2pt]
Alg.~\ref{alg:gated}: bound and cap & none (algebra) & $|\rgate|\le1$; accepted factor $\le(1-1/(\qgate\qucb))^{-1}$ if $m{=}K$, 2.12 in the base setting (App.~\ref{app:cap}) & bounded in 100/100 classes & -- \\\addlinespace[2pt]
Alg.~\ref{alg:gated}: levels; Prop.~\ref{prop:magnitude} & balanced Gaussian one-way model (Gaussian target values and noise); max rule: for each measure, a bound that covers on the branch where it is used & gate level $\alpha$; each bound covers with prob.\ $\ge1-\alpha$; $|\rgate|\le$ correction with true reliabilities with prob.\ $\ge1-2\alpha$ & MSE (Table~\ref{tab:mse}, Fig.~\ref{fig:robust}); coverage given acceptance 0.90--1.00; exact-$F$ coverage 0.88--0.90 up to $|S|\approx0.5$, 0.79--0.84 near 0.74 & an interval for $\rtrue$; optimality; same-run correlated errors; MSE or coverage of the max rule \\\addlinespace[2pt]
Corrected value (0.88) & mean-zero errors; independent errors in $X$ and $Y$; exchangeable seeds & a corrected value with no Gaussian assumption; the Gaussian model enters only through the gate and $\qucb\approx1.01$ & agrees with the joint estimate (0.93) and the plug-in over 12 seed partitions (0.89--0.98) & accuracy on real data (no ground truth) \\\addlinespace[2pt]
Alg.~\ref{alg:perproto} & i.i.d.\ items; paired-item CLT; consistent $\hat R$; one shared grid fixed in advance; every model scored alike & simultaneous coverage $1-\alpha$ of all $2P$ gaps, asymptotic in items; family-wise error control of all sign claims & Tables~\ref{tab:verdicts} and~\ref{tab:pooled}; new wordings 1/93 (App.~\ref{app:transfer}) & unseen prompts, harnesses, scoring or tasks; finite samples; rank sets across wordings (0--2 of 6 splits); single configurations (5/93 pairs flip) \\\addlinespace[2pt]
$g(1)$, $g(32)$ & Gaussian G-theory components on 0/1 correctness & item-difficulty reliability of one configuration and of the mean of all 32 (protocols pooled) & beta-binomial check within 0.05 for 14/15 models & protocol-specific coefficients \\\addlinespace[2pt]
Aligned embedding reliability & alignment fit on a disjoint half; three seeds & per-dimension reliability in a common basis & stable across reference seed and fit fraction & architectures other than ResNet-18 \\
\bottomrule
\end{tabular}
\caption{What each result needs, what it gives, what we measured beyond its assumptions, and what it does not cover. Among the guarantees for Algorithm~\ref{alg:gated}, only the gate's level and the coverage levels need Gaussian noise; Proposition~\ref{prop:split} needs none.}
\label{tab:scope}
\end{table*}

Write $\clip(r)$ for $r$ truncated to $[-1,1]$.
A fixed floor on each reliability followed by clipping already keeps the squared error at most 4, because the clipped estimate and the true correlation both lie in $[-1,1]$.
The choice of floor, however, is arbitrary.
Algorithm~\ref{alg:gated} replaces the floor with a significance level $\alpha$.
Each reliability must pass a gate, the $F$-test of $\rho=0$, which fails when $\MSB/\MSW\le\qgate$, the upper-$\alpha$ quantile of $F_{\nu_2,\nu_1}$.
When $\rho=0$, the pivot~\eqref{eq:pivot} makes $\MSW/\MSB$ follow $F_{\nu_1,\nu_2}$, so $\MSB/\MSW$ follows $F_{\nu_2,\nu_1}$.
A failed gate means that the $F$-test cannot tell this reliability from zero.
If either reliability fails the gate, Algorithm~\ref{alg:gated} declines to correct. It returns the uncorrected estimate $\clip(\robs)$, flagged \Declined.
When both gates pass, the algorithm divides $\robs$ by exact-$F$ upper confidence bounds $u_X$ and $u_Y$.
Under the Gaussian model each bound covers its true reliability with probability $1-\alpha$, meaning that it lies at or above it.
The division therefore errs toward a smaller correction.
If $\robs$ averages $m\neq K$ replicates per unit, each bound is first mapped to $m$ replicates with the Spearman--Brown formula~\eqref{eq:rhoK}.
The map is increasing, so a mapped bound covers exactly when the original does (Appendix~\ref{app:cap} gives the cap after mapping).

\begin{algorithm}[t]\small
\caption{Gated disattenuation. Correct only when both reliabilities pass an $F$-test, then divide by upper confidence bounds.}\label{alg:gated}
\begin{algorithmic}[1]
\STATE \textbf{input} replicate matrices for $X$ and $Y$ ($N$ units $\times$ $K$ replicates); $\robs$, computed on $m$ replicates per unit and on seeds separate from these replicates when possible; level $\alpha{=}0.1$
\STATE $\qgate\gets$ upper-$\alpha$ quantile of $F_{\nu_2,\nu_1}$;\ \ $\qucb\gets$ upper-$\alpha$ quantile of $F_{\nu_1,\nu_2}$ \COMMENT{$\nu_1{=}N(K{-}1)$, $\nu_2{=}N{-}1$}
\FOR{$\ell\in\{X,Y\}$}
  \STATE $\rhat_\ell\gets 1-\MSW_\ell/\MSB_\ell$ \COMMENT{$\ICC{1,K}$}
  \IF{$\MSB_\ell/\MSW_\ell\le\qgate$}
    \STATE \textbf{return} $\clip(\robs)$, flagged \Declined \COMMENT{gate fails}
  \ENDIF
  \STATE $u_\ell\gets 1-(1-\rhat_\ell)/\qucb$ \COMMENT{exact-$F$ upper bound on the $K$-replicate reliability}
  \STATE $S_\ell\gets$ Spearman correlation, across units, of within-unit variance and unit mean
  \IF{$|S_\ell|>0.5$}
    \STATE $u^{B}_\ell\gets(1{-}\alpha)$ quantile of $\ICC{1,K}$ over $B{=}400$ resamples of units (all $K$ replicates kept)
    \STATE $u_\ell\gets\max(u_\ell,u^{B}_\ell)$ \COMMENT{max rule}
  \ENDIF
  \IF{$m\neq K$}
    \STATE $u_\ell\gets (m/K)\,u_\ell\,/\,\big(1-(1-m/K)\,u_\ell\big)$ \COMMENT{Spearman--Brown}
  \ENDIF
\ENDFOR
\STATE \textbf{return} $\clip\big(\robs/\sqrt{u_Xu_Y}\big)$
\end{algorithmic}
\end{algorithm}

\begin{figure}[t]
\centering
\includegraphics[width=\columnwidth]{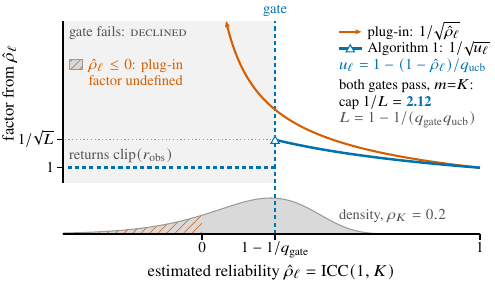}
\caption{Algorithm~\ref{alg:gated} declines wherever the plug-in factor is undefined or grows without bound, and dividing by an upper bound caps the factor that remains. The curves show the factor that one reliability estimate $\rhat_\ell=\ICC{1,K}$ contributes to the correction, computed exactly at $N{=}50$, $K{=}3$, $\alpha{=}0.1$. Left of the gate, where $\MSB_\ell/\MSW_\ell\le\qgate$, Algorithm~\ref{alg:gated} returns $\clip(\robs)$, so its factor is 1 (thick dashed line). Right of it, if the other gate also passes, Algorithm~\ref{alg:gated} divides by an upper bound that is at least $u_\ell$, so each factor is below $1/\sqrt{L}$. For $m{=}K$ the two factors together are therefore at most $1/L$, which is 2.12 here (Proposition~\ref{prop:magnitude}(a)). Bottom: the sampling density of $\rhat_\ell$ under the balanced Gaussian model when the reliability of the 3-replicate mean is 0.2, as in the base row of Table~\ref{tab:mse}. Most of it lies left of the gate, and in that row a gate fails in 87\% of runs.}
\label{fig:gate}
\end{figure}

\paragraph{When the noise grows with the score, divide by the larger bound.}
The exact-$F$ bound assumes that every unit has the same noise variance.
Algorithm~\ref{alg:gated} checks for heteroscedasticity, noise whose variance changes with the mean, through $S_\ell$, the Spearman correlation across units between each unit's variance over replicates and its mean.
The threshold $|S_\ell|>0.5$ is pre-specified from a calibration that computes $S$ against the simulated target values (Appendix~\ref{app:hetero}).
Above it, the algorithm also computes a bound $u^B_\ell$ by resampling units and divides by the larger of the two bounds, which we call the max rule.
The larger bound keeps the cap of Proposition~\ref{prop:magnitude}(a) and covers $\rho_\ell$ at least as often as either bound alone.
Every Algorithm~\ref{alg:gated} result we report uses the exact-$F$ branch, because the simulations have equal noise variances and both \mbox{CIFAR-100} scores have $|S|\le0.38$.
The max rule leaves the gate untouched, so the per-class analysis of Section~\ref{sec:real} declines in the same 3 classes under either branch.

Proposition~\ref{prop:magnitude} states what the gate and the upper bounds guarantee when Algorithm~\ref{alg:gated} accepts, that is, when both gates pass.

\begin{proposition}[Magnitude conservatism]\label{prop:magnitude}\hfill\par\nopagebreak\noindent
Suppose $\qgate\qucb>1$. (a) If both gates pass, then $u_X,u_Y\ge1-1/(\qgate\qucb)>0$ before any Spearman--Brown map, so for $m=K$ the correction factor is at most $(1-1/(\qgate\qucb))^{-1}$. This needs no distributional assumption.
\par\noindent (b) Suppose each $u_\ell$ is a level-$(1-\alpha)$ upper confidence bound for $\rho_\ell$, the reliability of the $m$-replicate measurement $\ell$. Under the balanced Gaussian model this holds for the exact-$F$ bound, and hence for the bound Algorithm~\ref{alg:gated} divides by on either branch, which is never smaller. Without the Gaussian model it holds when the bound used on each branch covers, for example with coverage at least $1-\alpha$ given the branch (Appendix~\ref{app:proof-magnitude}). Then, with probability at least $1-2\alpha$, $|\rgate|\le|\robs|/\sqrt{\rho_X\rho_Y}$, the magnitude of the correction with the true reliabilities. When a gate fails, $\rgate=\clip(\robs)$ meets this bound trivially.
\end{proposition}

Part (a) is algebra.
A passed gate means $1-\rhat_\ell=\MSW_\ell/\MSB_\ell<1/\qgate$, which gives the bound on $u_\ell$.
Part (b) is a union bound over the two coverage events (Appendix~\ref{app:proof-magnitude}).
The cap in part (a) rules out the failures of Section~\ref{sec:plugin}, because after a passed gate the correction factor is bounded by a constant that depends only on $N$, $K$, $m$ and $\alpha$ (Appendix~\ref{app:cap} and Figure~\ref{fig:gate}).

\paragraph{Algorithm~\ref{alg:gated} returns a bounded point estimate.}
The output always lies in $[-1,1]$, whatever the noise distribution, and Appendix~\ref{app:alpha} varies its only tuning setting, $\alpha$.
Take the base setting of our simulations ($N{=}50$, $K{=}3$, $\alpha{=}0.1$).
There $\qgate\qucb=1.89$, so an accepted correction multiplies $\robs$ by at most 2.12 (Figure~\ref{fig:gate}).
The cap loosens as $N$ grows, because both quantiles approach one and a smaller reliability can then pass the gate (Table~\ref{tab:caps}).

\paragraph{Proposition~\ref{prop:magnitude} bounds the size of the correction.}
With probability at least $1-2\alpha$, the corrected value is no larger in magnitude than the correction with the true reliabilities.
The proposition gives no interval for $\rtrue$, and the uncorrected estimate meets the same bound.
A Fieller-type interval for $\rtrue$ \citep{fieller1954} over-covered in our simulations, but at $\rho{=}0.2$ its mean width was about 1.9, nearly the whole range of a correlation (Appendix~\ref{app:interval}).
Declining to correct is the point-estimate counterpart of such an interval.
We claim no optimality for Algorithm~\ref{alg:gated}.
Table~\ref{tab:scope} lists what each result needs and what it leaves uncovered (its last three rows summarize Sections~\ref{sec:compare} and~\ref{sec:embed}).

\begin{figure*}[t]
\centering
\includegraphics[width=\textwidth]{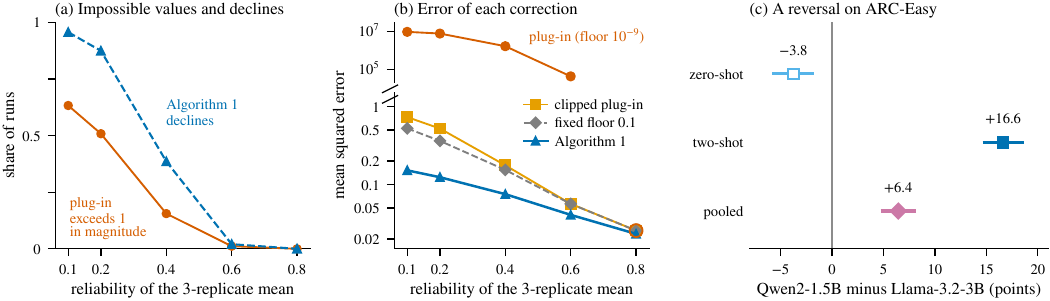}
\caption{Where the plug-in correction returns impossible values, Algorithm~\ref{alg:gated} mostly declines to correct (a, b), and averaging two prompt settings hides a ranking reversal (c). Panels (a) and (b) simulate a true correlation of 0.4 with $N{=}50$ units, $K{=}3$ replicates per unit and Gaussian noise, using 4,000 simulation runs per point. The correlation and the reliabilities use the same replicates, and Algorithm~\ref{alg:gated} uses $\alpha{=}0.1$. (a) Share of runs in which the plug-in exceeds 1 in magnitude, and in which Algorithm~\ref{alg:gated} declines. (b) Mean squared error on a broken log scale. The plug-in floors the product of the two reliabilities at $10^{-9}$. The clipped plug-in cuts its value back to $[-1,1]$, and the fixed floor raises each reliability to at least 0.1 before dividing and clipping. At reliability 0.8 the plug-in (open ring) matches both. (c) ARC-Easy, \mbox{Qwen2-1.5B} against \mbox{Llama-3.2-3B}, with the simultaneous 95\% interval of Algorithm~\ref{alg:perproto} in each setting and the pooled comparison's interval for their average. A gap is the first-named model's accuracy minus the second-named model's, in percentage points.}
\label{fig:teaser}
\end{figure*}

\subsection{Correction pays off once reliability is well estimated}\label{sec:sim}

Simulations let us measure accuracy against a known true correlation $\rtrue$.
Each simulation run computes $\robs$ from the $K$-replicate means of $N$ units and both reliabilities, $\ICC{1,K}$, from the same replicates, as in the per-class example of Section~\ref{sec:real}.
Table~\ref{tab:mse} reports the mean squared error (MSE) of each method (design in Appendix~\ref{app:sim-design}).

\begin{table}[t]\centering\footnotesize\setlength{\tabcolsep}{3.4pt}
\begin{tabular}{@{}lcccccc@{}}
\toprule
Scenario & Uncorr. & Joint & \shortstack{Boot.\\floor} & Alg.~\ref{alg:gated} & \shortstack{Best\\floor$^\dagger$} & \shortstack{Declined\\(\%)} \\
\midrule
Base & \textbf{.12} & .31 & .22 & \textbf{.12} & .14 & 87 \\
Asym.\ $.3/.8$ & \textbf{.06} & .11 & .10 & \textbf{.06} & .06 & 42 \\
$t_3$ noise & \textbf{.12} & .30 & .21 & \textbf{.12} & .13 & 84 \\
$\rtrue{=}0$ & \textbf{.02} & .28 & .19 & .03 & .08 & 87 \\
$\rtrue{=}.8$ & .43 & .45 & \textbf{.29} & .42 & .27 & 86 \\
$N{=}20$ & \textbf{.16} & .33 & .29 & \textbf{.16} & .26 & 94 \\
$N{=}100$ & \textbf{.11} & .27 & .19 & \textbf{.11} & .10 & 75 \\
$N{=}500$ & .11 & .08 & .08 & \textbf{.05} & .04 & 12 \\
\bottomrule
\end{tabular}
\caption{At low reliability Algorithm~\ref{alg:gated} stays within 0.01 of the uncorrected estimate, and with 500 units it more than halves the uncorrected error. Entries are the mean squared error (MSE) against the true correlation over 2,000 simulation runs per row, rounded to two decimals. Base row: $N{=}50$ units, $K{=}3$ replicates, reliability of the 3-replicate mean 0.2 for both measurements, true correlation 0.4, Gaussian noise. Each other row changes one setting. Asym.: reliabilities 0.3 and 0.8. $t_3$: Student-$t$ noise with 3 degrees of freedom. In every row $\robs$ and the reliabilities use the same replicates. Uncorr.: $\clip(\robs)$. Joint: moment estimator from covariances between different replicates, clipped. Boot.\ floor: floor chosen by a parametric bootstrap. Alg.~\ref{alg:gated}: $\alpha{=}0.1$. $^\dagger$Best post-hoc floor: the single floor from a grid of eight (Appendix~\ref{app:baselines}) with the lowest MSE against the truth. It is chosen after seeing the truth, so it is not deployable, and it can exceed the uncorrected MSE. Declined: percentage of runs in which a gate fails. Bold: lowest deployable MSE (ties bold). The plug-in with a $10^{-9}$ floor on the product of the reliabilities is omitted (MSE $10^6$--$10^7$). Monte Carlo s.e.\ of Alg.~\ref{alg:gated} cells $\le0.005$.}
\label{tab:mse}
\end{table}

\paragraph{The plug-in is unusable at low reliability.}
To keep the plug-in finite, we floor the product of the two estimated reliabilities at $10^{-9}$.
Its MSE is still about $10^7$, and it grows as the floor shrinks.
When $\rho_K$ is 0.1, the plug-in exceeds one in magnitude in more than 60\% of runs (Figure~\ref{fig:teaser}(a)).
We call the five rows of Table~\ref{tab:mse} with both reliabilities at 0.2 and at most 50 units the low-reliability rows.
In those rows the joint estimator, a moment estimator built from covariances between different replicates, is undefined or outside $[-1,1]$ before clipping in 48--68\% of runs (Appendix~\ref{app:fulltable}).

\paragraph{At low reliability Algorithm~\ref{alg:gated} mostly declines to correct.}
In the low-reliability rows a gate fails in 84--94\% of runs, and Figure~\ref{fig:gate} (bottom) shows why for the base row.
A declined run returns the uncorrected estimate $\clip(\robs)$ exactly, so any MSE difference between Algorithm~\ref{alg:gated} and that estimate comes from the runs in which both gates pass.
Across all of Table~\ref{tab:mse}, Algorithm~\ref{alg:gated} is never more than 0.01 above the uncorrected estimate.

\paragraph{With 500 units the correction more than halves the error.}
With 500 units, Algorithm~\ref{alg:gated} has MSE 0.048, against 0.105 for the uncorrected estimate.
Among deployable methods, Algorithm~\ref{alg:gated} is best or tied in the asymmetric row and with 100 or 500 units, and the uncorrected estimate is best or tied in the low-reliability rows except at $\rtrue=0.8$.
The bootstrap floor does better when the true correlation is large, with 0.29 against 0.42 at $\rtrue=0.8$.
Every deployable correction other than Algorithm~\ref{alg:gated}, the bootstrap floor included, is at least 0.10 worse than the uncorrected estimate in some row (Appendix~\ref{app:fulltable}).

\paragraph{The gate does most of the work.}
Removing the gate raises the MSE by up to 2.5$\times$ at low reliability (Appendix~\ref{app:ablation}).
Levels from 0.05 to 0.25 give MSE within about 25\% of each other in every cell (Appendix~\ref{app:alpha}).
Every other simulation and the cross-seed correction use $\alpha=0.1$, which we did not tune.

\subsection{Real difficulty scores}\label{sec:real}

\paragraph{The plug-in fails on real scores.}
We return to the difficulty scores of Section~\ref{sec:intro}.
EL2N is taken at epoch 10, and dynamic uncertainty in its single-window form, the variability score of dataset cartography \citep{swayamdipta2020} after a 5-epoch warm-up (Appendix~\ref{app:score-rel}).
Within each \mbox{CIFAR-100} class we correlate two-seed means of the two scores across the class's roughly 500 training examples, with reliabilities $\ICC{1,2}$ from the same two seeds.
Several reliabilities are near zero, and the plug-in exceeds one in magnitude in 70 of the 100 classes (Figure~\ref{fig:overview}(a)).
Algorithm~\ref{alg:gated} returns a bounded value in all 100 classes.
In the 3 classes where either reliability is not significant at $\alpha=0.25$, a looser level used only for this two-seed demonstration, the gate fails and Algorithm~\ref{alg:gated} declines to correct.
The correlation and the reliabilities share the same two seeds, so these values illustrate the failure and are not offered as corrected estimates (Appendix~\ref{app:perclass}).

\paragraph{Dynamic uncertainty is the more reliable score.}
With eight seeds we estimate single-seed reliability on all 50,000 training examples, using the crossed form of Section~\ref{sec:reliability}.
On \mbox{CIFAR-100} it is 0.58 for EL2N and 0.84 for dynamic uncertainty.
On \mbox{CIFAR-10} the two values are 0.55 and 0.79, and on both datasets the seed-subset intervals of the two scores do not overlap (Table~\ref{tab:scorerel}).

\paragraph{Correcting across seeds gives 0.88 on all 50,000 examples.}
On \mbox{CIFAR-100}, the errors of the two scores from the same seed correlate at +0.28.
In a two-seed simulation, an error correlation of 0.3 adds +0.15 to a same-seed estimate (Appendix~\ref{app:errcorr}).
An error correlation of 0.6 adds +0.29.
We therefore compute $\robs$ from EL2N on seed 2 and dynamic uncertainty on seed 3.
The single-seed reliabilities $\ICC{1,1}$ come from seeds 0 and 1.\footnote{Our run plugged these single-seed estimates into the exact-$F$ bound for a two-seed mean, which is conservative. Running Algorithm~\ref{alg:gated} as written, with its Spearman--Brown step, also gives 0.88 (Appendix~\ref{app:pooled}).}
The corrected value itself, as opposed to the coverage of its bounds, assumes errors with mean zero, exchangeable seeds and independent errors in $X$ and $Y$.
This design provides the last, and it also computes $\robs$ on seeds disjoint from the reliability seeds.
Algorithm~\ref{alg:gated} corrects $\robs=0.64$ to 0.88.

The joint estimator agrees at 0.93, and over the 12 ways to choose the $X$ seed and the $Y$ seed from the four, with reliabilities from the other two, the plug-in ranges from 0.89 to 0.98 (Appendix~\ref{app:pooled}).
These partitions share seeds, so their spread is not an error bar.
Real data have no ground truth, so these numbers show agreement, and Table~\ref{tab:mse} is where accuracy is measured.
The Gaussian model enters only through the gate and the quantile $\qucb$, which is 1.01 at $N=50{,}000$, so the upper bounds nearly equal the point estimates.

\paragraph{The \mbox{CIFAR-100} scores pass the noise check.}
On \mbox{CIFAR-100}, the variance--mean correlation $S$ of Section~\ref{sec:gated}, computed over the eight seeds, is +0.15 for EL2N and $-$0.38 for dynamic uncertainty.
Negating every score gives the mirror-image design with $-S$ and the same ICC and bounds, so only $|S|$ matters, and both values lie in the range where exact-$F$ bounds covered near the nominal 0.90 (Appendix~\ref{app:hetero}).
The strongest value, +0.74, is for EL2N on \mbox{CIFAR-10}.
In simulations with heteroscedasticity this strong, exact-$F$ coverage fell to 0.79--0.84, and a unit-bootstrap bound raised it to 0.87.
Neither real-data finding depends on exact-$F$ coverage, since the 70 impossible values are a count and 0.88 needs only the conditions listed above.

\subsection{When to correct}\label{sec:whento}

The evidence in Sections~\ref{sec:gated}--\ref{sec:real} suggests six steps.
\begin{enumerate}\setlength{\itemsep}{0pt}\setlength{\parsep}{0pt}
\item Compute $\robs$ across seeds when $K\le2$ or when errors from the same run correlate strongly (Appendix~\ref{app:errcorr}).
\item Gate both reliabilities at $\alpha=0.1$. If either fails, report the uncorrected estimate, flagged \Declined.
\item Check the variance--mean correlation $S$. Use the max rule above \mbox{$|S|=0.5$}.
\item Expect the correction to change the answer only when reliability is precisely estimated (Table~\ref{tab:mse}). At low reliability, expect the uncorrected estimate back.
\item For an interval for $\rtrue$, use a Fieller-type interval \citep{fieller1954}, which over-covered in our simulations but is wide at low reliability.
\item At large true correlations the bootstrap floor did better, as in the row $\rtrue=0.8$ of Table~\ref{tab:mse}.
\end{enumerate}

\section{Comparing models across prompt settings}\label{sec:compare}

A ranking of models must name its prompt setting (Figures~\ref{fig:overview}(b) and~\ref{fig:teaser}(c)).

\subsection{What we compare}\label{sec:estimand}

We compare 15 open Instruct checkpoints from four families, Qwen2 and Qwen2.5 \citep{yang2024a,yang2024b}, SmolLM2 \citep{benallal2025}, Falcon3 \citep{falcon2024} and \mbox{Llama-3.2} \citep{grattafiori2024,meta2024}, ranging from 135M to 7B parameters (Table~\ref{tab:models} lists them).
A model's answer is the option letter with the largest next-token logit after the model's own chat template, so scoring is deterministic.
Our main benchmark is ARC-Easy \citep{clark2018}, and its items are the 2,365 four-option test questions.
Eleven of the models are also compared on HellaSwag \citep{zellers2019}, MMLU \citep{hendrycks2021} and OpenBookQA \citep{mihaylov2018}, with 2,365 items each.
Appendix~\ref{app:models} describes these subsets and explains the choice of models, two of which score near chance and serve as negative controls.

Every model answers every item under the same prompt grid, 4 instruction wordings crossed with 4 option orders, so each protocol has 16 configurations.
In the two-shot protocol the same two solved examples are written inline before every question.
Following Section~\ref{sec:protocols}, wording and option order are replicates, and the number of solved examples is a protocol with its own target.

Let $C_A[c,\pi,i]$ equal 1 when model $A$ answers item $i$ correctly under configuration $c$ of protocol $\pi$, and 0 otherwise.
For a pair $p=(A,B)$, the per-item gap $d^{p,\pi}_i$ averages $C_A-C_B$ over the 16 configurations of protocol $\pi$.
The target is the mean of this gap over the population of items,
\begin{equation}
\begin{split}
\Delta_{p,\pi}&=\E_i\big[d^{p,\pi}_i\big],\\
d^{p,\pi}_i&=\operatorname{mean}_c\{C_A[c,\pi,i]-C_B[c,\pi,i]\},
\end{split}
\label{eq:delta}
\end{equation}
and its estimate $\hat\Delta_{p,\pi}$ is the mean of $d^{p,\pi}_i$ over the benchmark's items.
A gap is the first-named model's accuracy minus the second-named model's, in percentage points.
Each benchmark forms one comparison family, which holds every pair of its models in both protocols, so $P$ pairs give $2P$ contrasts.
ARC has 105 pairs, and each of the other benchmarks has 55.
The three 7B models and \mbox{Qwen2.5-3B} are included only in the ARC comparison family.

\subsection{Algorithm~\ref{alg:perproto} and its three verdicts}\label{sec:alg2}

Algorithm~\ref{alg:perproto} builds one interval for every pair and protocol and reads a verdict from their signs (Figure~\ref{fig:alg2}).
All $2P$ intervals on a benchmark share one critical value $\qalpha$ from the max-$t$ method \citep{hothorn2008}.
The method treats the $2P$ standardized gaps as one Gaussian vector whose correlation $\hat R$ is estimated from the per-item gaps.
The value $\qalpha$ is the $(1-\alpha)$ quantile of the vector's largest absolute entry, so with probability about $1-\alpha$ every interval covers its gap at the same time.
Such intervals are called simultaneous, and we use $\alpha=0.05$ throughout.

\begin{algorithm}[t]\small
\caption{Per-protocol simultaneous comparison. A direction is \Certified{} only if every protocol's simultaneous interval excludes zero with the same sign; the family-wise level is $\alpha$, asymptotically in items drawn from an item population, for the fixed prompt grid.}\label{alg:perproto}
\begin{algorithmic}[1]
\STATE \textbf{input} correctness $C_m[c,\pi,i]\in\{0,1\}$ for $M$ models, configurations $c$, protocols $\pi$ and $N$ shared items; level $\alpha$
\FOR{each pair $p=(A,B)$ and protocol $\pi$}
  \STATE $d^{p,\pi}_i\gets\operatorname{mean}_c\{C_A[c,\pi,i]-C_B[c,\pi,i]\}$
  \STATE $\hat\Delta_{p,\pi}\gets\operatorname{mean}_i d^{p,\pi}_i$;\ \ $\widehat{\mathrm{se}}_{p,\pi}\gets\operatorname{sd}_i(d^{p,\pi}_i)/\sqrt N$
\ENDFOR
\STATE $\hat R\gets$ correlation matrix of the centered vectors $d^{p,\pi}_i-\hat\Delta_{p,\pi}$
\STATE $\qalpha\gets(1{-}\alpha)$ quantile of $\max_j|Z_j|$, $Z\sim\mathcal N(0,\hat R)$ \COMMENT{max-$t$}
\STATE $I_{p,\pi}\gets\hat\Delta_{p,\pi}\pm\qalpha\,\widehat{\mathrm{se}}_{p,\pi}$ for every $(p,\pi)$
\FOR{each pair $p$}
  \IF{every $I_{p,\pi}$ excludes $0$, all with the same sign}
    \STATE $V_p\gets$ \Certified{} (with that sign)
  \ELSIF{every $I_{p,\pi}$ excludes $0$ and the signs differ}
    \STATE $V_p\gets$ \ProtDep
  \ELSE
    \STATE $V_p\gets$ \Refused
  \ENDIF
\ENDFOR
\STATE \textbf{return} the intervals $I_{p,\pi}$, the verdicts $V_p$, and the rank sets they induce
\end{algorithmic}
\end{algorithm}

\begin{figure}[t]
\centering
\includegraphics[width=\columnwidth]{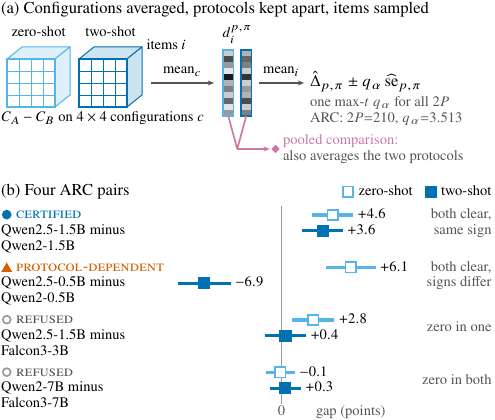}
\caption{Algorithm~\ref{alg:perproto} averages over prompt configurations, keeps the protocols apart and reads one verdict per pair from the signs of its simultaneous intervals. (a) For a pair $p=(A,B)$, the correctness difference on each item is averaged over the 16 configurations of a protocol, 4 wordings $\times$ 4 option orders, giving the per-item gap $d^{p,\pi}_i$. Its mean and standard error over items give one interval per pair and protocol. All $2P$ intervals on a benchmark share one max-$t$ critical value $\qalpha$, which controls the family-wise error asymptotically in the number of items, for the fixed prompt grid. The pooled comparison also averages over the two protocols, so it cannot return \ProtDep{}. (b) Four ARC pairs with their simultaneous 95\% intervals, drawn to scale, in percentage points. A pair is \Certified{} when both intervals exclude zero with the same sign, \ProtDep{} when both exclude zero with opposite signs, and \Refused{} when zero lies inside one interval or both.}
\label{fig:alg2}
\end{figure}

We treat the items as an independent sample from a population of items.
Coverage is then asymptotic in the number of items, under a central limit theorem for the paired per-item gaps and a consistent $\hat R$.
The coverage holds conditionally on the fixed grid and needs no model of prompt noise, and Appendix~\ref{app:alg2-proc} gives the numerical details.

We call a direction \emph{certified} when, in every protocol, its simultaneous interval excludes zero with the same sign. The family-wise error level is $\alpha$, asymptotically in the number of items, for items drawn independently from an item population and the fixed prompt grid, with every model scored the same way on the same items. It does not extend to unseen prompts or other harnesses.
\ProtDep{}: every protocol's interval excludes zero, and the signs differ.
We also call such a pair a reversal.
\Refused{}: in at least one protocol the simultaneous interval contains zero, so no winner is certified at level $\alpha$. A refusal is not evidence of a tie, and the intervals are still reported.
A pair is resolved in a protocol when its interval there excludes zero.
A method names a stable winner for a pair when it reports one model as better without reference to protocol.
Algorithm~\ref{alg:perproto} does so only for a \Certified{} pair, and the pooled comparison does so whenever its interval for the average gap excludes zero.
A \Refused{} pair is either unresolved in both protocols or resolved in only one protocol (Figure~\ref{fig:alg2}(b)).
Algorithm~\ref{alg:perproto}'s \Refused{}, a verdict about a pair, is unrelated to Algorithm~\ref{alg:gated}'s \Declined{}, a decision not to correct.
Table~\ref{tab:scope} lists what Algorithm~\ref{alg:perproto} needs and what it leaves uncovered.

The pair intervals also give each model a rank set within each protocol (Appendix~\ref{app:alg2-proc}).
These rank sets are simultaneous only on the fixed grid, and across held-out wordings they covered the ranks of all models in only 0--2 of the 6 splits (Appendix~\ref{app:transfer}).

\paragraph{How Algorithm~\ref{alg:perproto} relates to standard methods.}
\looseness=1 \mbox{Max-$t$} intervals over the $2P$ contrasts are standard \citep{hothorn2008,westfall1993}, and we add the per-protocol target and the three-verdict rule.
Because all intervals cover at once, every \Certified{} and \ProtDep{} sign claim keeps the family-wise error rate (FWER), the chance of at least one false claim on a benchmark, at $\alpha$ asymptotically.
Requiring every protocol to agree is intersection--union logic \citep{berger1982}, and flagging opposite signs resembles a test for a qualitative interaction \citep{gailsimon1985}.
Appendix~\ref{app:alg2-mult} places Bonferroni and \v{S}id\'ak intervals \citep{sidak1967}, resampling max-$t$ and step-down methods \citep{holm1979} under the same verdict rule.
\subsection{Results}\label{sec:results}

\begin{table}[t]\centering\footnotesize\setlength{\tabcolsep}{1.8pt}
\begin{tabular}{@{}lrrrr@{}}
\toprule
 & ARC & HellaSwag & MMLU & OpenBookQA \\
\midrule
models / pairs & 15/105 & 11/55 & 11/55 & 11/55 \\
$\qalpha$ & 3.513 & 3.336 & 3.344 & 3.380 \\
\Certified & 93 & 42 & 43 & 43 \\
\ProtDep & 5 & 4 & 3 & 4 \\
\Refused & 7 & 9 & 9 & 8 \\
\midrule
\multicolumn{5}{@{}l}{\emph{Measured error rates (\%)}} \\
held-out wording & 2.0 & 1.6 & 2.4 & 5.0 \\
held-out order & 0.0 & 0.0 & 0.4 & 0.0 \\
perm.-null FWER & 6.5 & 6.0 & 7.0 & 5.0 \\
self-comp., Alg.~2 & 5.2 & 1.6 & 1.5 & 5.4 \\
self-comp., pooled & 19.4 & 13.3 & 11.5 & 30.5 \\
\bottomrule
\end{tabular}
\caption{Most pairs are certified, and the measured error rates of Algorithm~\ref{alg:perproto} are low. Level $\alpha{=}0.05$ and 2,365 items per benchmark. $\qalpha$: max-$t$ critical value. Held-out wording (order): share of certified directions whose sign flips on the two wordings (orders) left out of calibration. Perm.-null FWER: family-wise error rate under a permutation null, the share of 200 label-permuted benchmarks in which some pair is falsely resolved (Monte Carlo s.e.\ \mbox{1.5--1.8\%}). Self-comp.: split-grid self-comparison, a model compared with itself on two disjoint halves of its grid, where any stable winner is an error. The last row applies the same test to the pooled comparison.}
\label{tab:verdicts}
\end{table}

\paragraph{Most pairs are certified, and measured error rates are low.}
Algorithm~\ref{alg:perproto} certifies 221 of the 270 pairs across the four benchmarks (Table~\ref{tab:verdicts}).
To see how often it names a winner where none exists, we build null benchmarks by shuffling the model labels within each item.
At least one pair is falsely resolved in 5.0--7.0\% of them (Monte Carlo s.e.\ 1.5--1.8\%), consistent with the nominal 5\% within Monte Carlo error.
In a separate null simulation of Algorithm~\ref{alg:perproto} alone, over 2,000 simulated benchmarks with binary items, the family-wise error rate is 4.3\% (s.e.\ 0.5\%).
The 44 pairs certified on two or more benchmarks agree in direction wherever they are certified.

\paragraph{Few certified directions change sign on held-out prompts.}
We refit Algorithm~\ref{alg:perproto} on two of the four wordings and check each certified direction on the other two.
The sign flips in 34 of 1,324 checks, or 1.6--5.0\% depending on the benchmark.
With option orders held out instead, 1 of 1,305 fails.
We also wrote four new ARC wordings after all verdicts were fixed, and 1 of the 93 certified ARC directions reverses on them.
These are measured error rates, outside the guarantee.
Certification concerns the mean direction within each protocol, and 5 of the 93 certified ARC pairs flip on some single configurations (Appendix~\ref{app:verdicts}).

\paragraph{Two-shot prompting reverses sixteen rankings.}
Sixteen pairs are \ProtDep{}, so the zero-shot and two-shot protocols name different winners (Appendix~\ref{app:pooled-baseline}, Figure~\ref{fig:reversals}).
On ARC, \mbox{Qwen2.5-0.5B} leads \mbox{Qwen2-0.5B} by 6.1 points at zero-shot and trails it by 6.9 points with two shots (Figure~\ref{fig:alg2}(b), intervals in Appendix~\ref{app:verdicts}).
Each of the other 15 reversals pairs \mbox{Llama-3.2-3B} with another model.
Averaged over its 16 configurations in each protocol, the ARC accuracy of \mbox{Llama-3.2-3B} is 0.885 at zero-shot and 0.677 at two-shot.
On HellaSwag, MMLU and OpenBookQA, every reversal also involves \mbox{Llama-3.2-3B}, with the same sign pattern against \mbox{1--3B} models of the other families.

\paragraph{On ARC, the \mbox{Llama-3.2-3B} drop depends on how the examples are formatted.}
Across the four ARC wordings, the inline two-shot block costs \mbox{Llama-3.2-3B} between 10 and 51 points.
Given the same two examples as earlier chat turns, its two-shot accuracy with the options in their original order rises from 0.67 to 0.89, averaged over the four wordings.
\mbox{Falcon3-3B} barely changes (Figure~\ref{fig:shotformat}, Appendix~\ref{app:shotformat}).
The control re-evaluates four models on ARC only.

\begin{table}[t]\centering\footnotesize\setlength{\tabcolsep}{4pt}
\begin{tabular}{@{}lcc@{}}
\toprule
 & Alg.~2 & Pooled \\
\midrule
\multicolumn{3}{@{}l}{\emph{Real data, 270 pairs}} \\
pairs with a stable winner & 221 & 254 \\
reversals given a stable winner & 0/16 & 15/16 \\
split-grid self-comparison, false winners & 3.6\% & 18.7\% \\
\midrule
\multicolumn{3}{@{}l}{\emph{Simulated benchmarks, 1,000 per scenario}} \\
stable rankings: correct verdict / resolved & 89.5\% & $\approx$100\% \\
near-ties: correct verdict / resolved & 4.5\% & 33.3\% \\
reversals detected & 68.3\% & -- \\
reversals given a false stable winner & 0\% & 64.6\% \\
simulated-null FWER & 5.1\% & 4.3\% \\
\bottomrule
\end{tabular}
\caption{Pooling buys power and hides reversals. Alg.~2: per-protocol comparison. Pooled: the pooled comparison, which averages the two protocols per item before the same \mbox{max-$t$} step. Simulated benchmarks have six models and 2,365 binary items. For stable rankings and near-ties, the Alg.~2 column gives the share of pairs with the correct verdict, and the pooled column the share resolved. --: the pooled comparison cannot return \ProtDep{}. Simulated-null FWER: family-wise error rate under the global null, in which all models are equal in both protocols (Monte Carlo s.e.\ about 0.7\%).}
\label{tab:pooled}
\end{table}

\paragraph{Pooling hides reversals.}
The pooled comparison averages each item's gap over the two protocols and then applies the same paired max-$t$ intervals.
On the real benchmarks it resolves 254 of the 270 pairs.
Among them are 15 of the 16 reversals, each given a stable winner.
Simulated benchmarks with planted reversals show what this costs (Table~\ref{tab:pooled}).
There the pooled comparison names a stable winner for 64.6\% of the reversed pairs, which is always wrong, and Algorithm~\ref{alg:perproto} never does, although both have similar error rates under the null.
Algorithm~\ref{alg:perproto} detects 68.3\% of the reversals and returns \Refused{} for the rest.

\paragraph{Separating protocols costs some power.}
When the protocols agree, the pooled comparison finds more winners.
On simulated benchmarks with stable rankings, Algorithm~\ref{alg:perproto} returns the correct verdict for 89.5\% of pairs, and the pooled comparison resolves about 100\%.
When the true gaps are near zero, the rates are 4.5\% and 33.3\%.
Appendix~\ref{app:pooled-baseline} discusses three variants, none of which we ran, that could recover some of this power at a cost.

\paragraph{Refused pairs.}
Seven ARC pairs are \Refused{}.
Two are unresolved in both protocols, and five are resolved in only one, for example \mbox{Qwen2.5-1.5B} against \mbox{Falcon3-3B} (+2.8 and +0.4 points, Figure~\ref{fig:alg2}(b)).
Across the four benchmarks 33 pairs are \Refused{}, and Algorithm~\ref{alg:perproto} reports their intervals like those of any other pair (Appendix~\ref{app:verdicts} lists the ARC pairs).

\paragraph{One prompt configuration measures item difficulty poorly.}\label{sec:gcoef}
With configurations as replicates, the crossed item $\times$ configuration form of G-theory (Section~\ref{sec:reliability}) gives $g(1)$, the reliability of item difficulty under one configuration, and $g(32)$, that of its mean over all 32.
Because $g(32)$ averages over both protocols, it pools a facet that Section~\ref{sec:protocols} treats as a protocol.
We use it only to show how much averaging over configurations helps, and we leave protocol-specific coefficients open (Table~\ref{tab:scope}).
For every model, $g(1)$ is at most 0.65 on ARC and at most 0.62 on the other benchmarks (Appendix~\ref{app:gtheory}, Figure~\ref{fig:gcoef}).
On HellaSwag, MMLU and OpenBookQA, all 24 model--benchmark combinations with $g(1)\ge0.2$ reach $g(32)\ge0.91$.

\subsection{One grid shared by all models}\label{sec:harness}
With the same items and identical configurations fixed in advance, the intervals of Algorithm~\ref{alg:perproto} cover the gaps on the configurations that all models share.
A configuration that some models never see adds a bias that more items cannot shrink.
Bounding it would need configurations sampled at random from a stated population, or an assumed cap on their effects.
We therefore state no theorem for comparisons across evaluation harnesses, meaning scoring pipelines that differ in template, shot format, scoring rule or grid.
We measure the bias instead with a split-grid self-comparison, a deliberate stress test outside the guarantee.
This test compares each model's scores on one half of its 16 wording-and-order configurations with its scores on the other half, so any winner it finds is false.
Algorithm~\ref{alg:perproto} names a false stable winner in 3.6\% of splits (342 of 9,600), and the pooled comparison in 18.7\%.
These false winners are more frequent for models whose accuracy varies more across configurations (Figure~\ref{fig:selfcomp}).
Any grid can serve, provided it is fixed in advance and shared by all models.
\section{Embeddings: align coordinates before measuring reliability}\label{sec:embed}

A common audit of representation stability correlates each embedding dimension across training seeds.
For the penultimate features of a \mbox{ResNet-18}, this audit mostly measures the arbitrary choice of coordinates.
The training loss sees these features only through a linear output layer, so many bases give nearly the same predictions and nothing puts the same feature in dimension $j$ of every seed.
Appendix~\ref{app:embed} explains why the choice of basis is only approximately free.

\paragraph{After alignment on held-out images, the features are reliable.}
We train three \mbox{ResNet-18} \citep{he2016} seeds on \mbox{CIFAR-100} \citep{krizhevsky2009}, with test images as units and seeds as replicates.
Per-dimension reliability correlates one feature dimension across images under two seeds, averaged over dimensions and seed pairs.
Before alignment, this reliability is 0.003.
Orthogonal Procrustes alignment finds the rotation or reflection that best maps one seed's features onto another's \citep{schonemann1966}.
We fit it on a random half of the test images and measure on the other half, where the aligned features reach 0.813 (Figure~\ref{fig:overview}(c)).
This value barely moves with the reference seed or with the share of images used for the fit (Appendix~\ref{app:embed}).

\paragraph{Measures that ignore the coordinates agree.}
Distances between pairs of images have a cross-seed reliability of 0.823.
Linear centered kernel alignment (CKA), a similarity index between two sets of features \citep{kornblith2019}, gives an agreement of 0.852.
\mbox{CIFAR-10} shows the same pattern (Appendix~\ref{app:embed}, Table~\ref{tab:embed}).
We therefore measure embedding reliability after an alignment fitted on held-out images, or on quantities such as distances that do not depend on the basis.
\section{Related work}\label{sec:related}

\paragraph{Training randomness and benchmark variance.}
Seeds and pipeline choices move reported results \citep{picard2021,bouthillier2021,sellam2022}, and equally accurate pipelines can return predictors that behave differently \citep{damour2022}.
\citet{dror2018} give a practical guide to significance testing in NLP.
These works do not report the reliability of a measurement across seeds or prompts, and we ask when that reliability supports a correction or a comparison.

\paragraph{Measurement theory.}
Correcting correlations for measurement error is classical \citep{spearman1904,fuller1987,carroll2006}.
\citet{searle1992} derive the probability of a negative variance-component estimate, and \citet{zimmerman1997} show by simulation that corrected correlations can exceed one.
Our gate is a preliminary test in the sense of \citet{bancroft1944}, and positive-part estimators \citep{baranchik1964} and Fieller's interval \citep{fieller1954} serve as comparisons.
\citet{jacobs2021} bring measurement theory to ML through construct validity.

\paragraph{Training-dynamics scores.}
Forgetting events \citep{toneva2019}, dataset cartography \citep{swayamdipta2020}, EL2N \citep{paul2021} and dynamic uncertainty \citep{he2024} score training examples from a single run and are used to prune datasets \citep{sorscher2022}.
Their reliability across seeds is usually assumed, and we measure it.

\paragraph{Item response theory and multi-prompt evaluation.}
\looseness=-1 Item response theory (IRT) models the difficulty of benchmark items and has been used to rank leaderboard entries \citep{martinez2019,rodriguez2021}.
PromptEval \citep{polo2024}, a multi-prompt variant, models item and prompt parameters and efficiently estimates each model's performance distribution across many prompt templates.
We target paired per-protocol gaps on shared items, with three verdicts and family-wise error control.

\paragraph{Prompt sensitivity and evaluation statistics.}
LLM accuracy shifts with prompt formatting \citep{sclar2024}, paraphrased instructions \citep{mizrahi2024} and the position of the answer option \citep{zheng2024}.
\citet{miller2024} gives standard errors for scores from one evaluation setting, including paired standard errors for the difference between two models.
\citet{avelar2026} build prompt-specific confidence sets for ranks under a contextual Bradley--Terry model of human preferences.
Shared harnesses \citep{liang2023,gao2023} standardize prompt settings across models but do not test rankings against them.

\paragraph{Simultaneous inference and representation similarity.}
Betting gives finite-sample intervals for bounded means \citep{waudbysmith2024}, and \citet{mogstad2024} give confidence sets for ranks.
Features from different seeds match only partly unit by unit \citep{li2016}, one recipe can yield models that generalize differently \citep{mccoy2020}, and CKA \citep{kornblith2019}, shape metrics \citep{williams2021} and their statistical tests \citep{ding2021} quantify agreement between representations.
These indices are not calibrated as reliabilities, so Section~\ref{sec:embed} uses them only as corroboration.

\section{Limitations}\label{sec:limits}

\begin{itemize}\setlength{\itemsep}{0pt}\setlength{\parsep}{0pt}
\item \textbf{Algorithm~\ref{alg:gated}.}
The gate's level, the coverage of the upper bounds and Proposition~\ref{prop:magnitude}(b) assume the balanced Gaussian one-way model.
Strong heteroscedasticity lowers the coverage of the exact-$F$ bound, and we report no MSE or coverage for the max rule that widens it.
Errors correlated within one training run are handled only by computing $\robs$ across runs.

\item \textbf{Algorithm~\ref{alg:perproto}.}
Coverage is asymptotic in the number of items, and requiring evidence in every protocol costs power when the protocols agree (Table~\ref{tab:pooled}).
The guarantee needs one prompt grid shared by all models (Section~\ref{sec:harness}), and we wrote ours, with 16 configurations in each of two protocols.
Our held-out wordings test transfer only within this multiple-choice format.
The other three benchmarks include 11 of the 15 models and use fixed subsets of 2,365 items, and OpenBookQA items come from its training pool (Appendix~\ref{app:models}).

\item \textbf{Closed models.}
We evaluate open checkpoints only.
Not every API exposes option-letter logits or a pinned, inspectable chat template, and an unpinned model alias can change partway through a grid.
Algorithm~\ref{alg:perproto} needs only 0/1 correctness, so API models fit when every model is scored the same way on one grid.
Mixing scoring methods makes the comparison cross-harness, which the guarantee does not cover.

\item \textbf{Vision.}
These studies are proofs of concept with \mbox{ResNet-18} on \mbox{CIFAR-10} and \mbox{CIFAR-100}, using two to eight seeds (Table~\ref{tab:analyses}).
Eight seeds already separate the reliability of EL2N from that of dynamic uncertainty.
Distance and CKA checks run on any encoder, but we have not tested the size of the coordinate effect beyond \mbox{ResNet-18}.

\item \textbf{Reproducibility.}
The per-class count of Section~\ref{sec:real} and the embedding values of Table~\ref{tab:embed} need \mbox{CIFAR-100} labels and features regenerated from the public dataset (Appendix~\ref{app:code}).

\item \textbf{Populations.}
Every reliability we report refers to the stated population of units and replicates.
A different item pool, training recipe or prompt set has its own reliability.
\end{itemize}
\section{Conclusion}\label{sec:concl}
Measurements computed from trained models carry noise from training seeds and prompts. Before correcting a correlation for this noise, check that both reliabilities are clearly above zero, and then divide by their upper confidence bounds. Compare models within each prompt setting, on one grid of prompt configurations that every model shares. When the data cannot support more, a declined correction or a refused pair is the result to report.

\appendix\raggedbottom\renewcommand{\dbltopfraction}{0.85}
\makeatletter\setlength{\@dblfptop}{0pt}\makeatother
\counterwithin{table}{section}\counterwithin{figure}{section}
\section{Notation, vocabulary, and which reliability each analysis uses}\label{app:notation}

\subsection{Notation}

Table~\ref{tab:notation} lists the symbols that appear in more than one section. Each keeps the meaning given here throughout the paper and the appendix.

\begin{table*}[t]\centering\footnotesize\setlength{\tabcolsep}{3pt}
\def\rag{\raggedright\let\\\tabularnewline}
\begin{minipage}[t]{0.49\textwidth}\centering
\begin{tabular}[t]{@{}p{2.2cm}p{6.2cm}@{}}
\toprule
Symbol & Meaning \\
\midrule
\multicolumn{2}{@{}l}{\textit{Measurements and reliability}} \\
$i=1,\dots,N$ & \rag unit \\
$s=1,\dots,K$ & \rag replicate (seed or configuration); never the protocol \\
$\hat\theta_{i,s}$ & \rag measurement $\thetastar_i+e_{i,s}$, target value plus replicate noise \\
$\tau^2$, $\sigma^2$ & \rag variance of $\thetastar$ across units; variance of $e$ \\
$\rho_1$ & \rag single-replicate reliability, $\tau^2/(\tau^2+\sigma^2)$ \\
$\rho_K$ & \rag reliability of a $K$-replicate mean, $\tau^2/(\tau^2+\sigma^2/K)=K\rho_1/(1+(K-1)\rho_1)$ (Spearman--Brown) \\
$\MSW$, $\MSB$ & \rag mean square within units (how much the replicates of one unit disagree); mean square between units \\
$\rhat_K$ & \rag one-way estimate $\ICC{1,K}=1-\MSW/\MSB$ \citep{shrout1979} \\
$\rhat_1$ & \rag one-way estimate $\ICC{1,1}=(\MSB-\MSW)/(\MSB+(K-1)\MSW)$ \\
$\nu_1$, $\nu_2$ & \rag degrees of freedom $N(K-1)$ and $N-1$ of the $F$ pivot \eqref{eq:pivot} \\
\bottomrule
\end{tabular}
\end{minipage}\hfill
\begin{minipage}[t]{0.49\textwidth}\centering
\begin{tabular}[t]{@{}p{2.2cm}p{6.2cm}@{}}
\toprule
Symbol & Meaning \\
\midrule
\multicolumn{2}{@{}l}{\textit{Correction}} \\
$\robs$, $\rtrue$ & \rag observed correlation; true correlation of the target values \\
$\rplug$, $\rgate$ & \rag plug-in correction $\robs/\sqrt{\rhat_X\rhat_Y}$; output of Algorithm~\ref{alg:gated} \\
$\ell\in\{X,Y\}$ & \rag which measurement \\
$\epsilon$ & \rag floor in $\max(\rhat,\epsilon)$, and nothing else \\
$\qgate$, $\qucb$ & \rag upper-$\alpha$ quantiles of $F_{\nu_2,\nu_1}$ and $F_{\nu_1,\nu_2}$ \\
$u_\ell$, $u^F_\ell$, $u^B_\ell$ & \rag bound Algorithm~\ref{alg:gated} divides by; exact-$F$ bound $1-(1-\rhat_\ell)/\qucb$; unit-bootstrap bound \\
$S_\ell$ & \rag Spearman correlation, across units, of within-unit variance and unit mean \\
$m$ & \rag replicates per unit behind $\robs$ \\
\midrule
\multicolumn{2}{@{}l}{\textit{Model comparison}} \\
$\alpha$ & \rag significance level only; the AR(1) coefficient in Appendix~\ref{app:persist} is $\phi$ \\
$A$, $B$; $p$ & \rag models; a pair $p=(A,B)$ \\
$P$; $M$ & \rag number of pairs; number of models \\
$\pi$ & \rag protocol, zero-shot or two-shot \\
$c$ & \rag configuration (16 per protocol) \\
$C_A[c,\pi,i]$ & \rag correctness, $1$ if model $A$ answers item $i$ correctly under $(c,\pi)$ and $0$ otherwise \\
$d^{p,\pi}_i$ & \rag per-item gap, averaged over the protocol's 16 configurations \\
$\Delta_{p,\pi}$, $\hat\Delta_{p,\pi}$ & \rag target mean gap; its estimate ($A$ minus $B$, in accuracy points) \\
$\hat R$ & \rag correlation of the $2P$ centered per-item gaps \\
$\qalpha$ & \rag max-$t$ critical value \\
$g(1)$, $g(32)$ & \rag G-theory reliability of item difficulty for one configuration and for the mean of the 32 configurations (both protocols pooled) \\
\bottomrule
\end{tabular}
\end{minipage}
\caption{Each symbol keeps one meaning throughout the paper. Replicates are indexed by $s$ alone, and $e$ is the replicate noise.}
\label{tab:notation}
\end{table*}

\subsection{Vocabulary}

Each object has one name. The definitions below are the ones used in the text, tables and figures.

\begin{description}\setlength{\itemsep}{1pt}
\item[Measurement.] A number computed from a trained model for each unit, such as a difficulty score, an embedding coordinate or a model's correctness on one item. ``Score'' refers only to difficulty scores.
\item[Unit.] The thing measured, such as a training example, a benchmark item or an image.
\item[Replicate.] One instance of the randomness that the target averages over, meaning one training seed or one prompt configuration within a protocol. In Appendix~\ref{app:persist} a replicate is one readout of a simulated trajectory.
\item[Target value and estimand.] The target value $\thetastar_i$ is the mean of unit $i$'s measurement over the replicate population. The estimand is the population quantity we want to estimate, such as the true correlation $\rtrue$ or a target mean gap $\Delta_{p,\pi}$.
\item[Protocol.] A setting that the reported result names and that we report separately (Section~\ref{sec:protocols}). Our protocols are zero-shot and two-shot, the latter with the same two solved examples written inline before every question.
\item[Configuration.] One wording of the instructions combined with one order of the answer options.
\item[Prompt grid.] The 32 configurations, 16 per protocol, from 4 wordings and 4 option orders.
\item[Evaluation harness.] The whole scoring pipeline, made of the template, the shot format, the scoring rule and the grid. We use the term only when two pipelines are compared.
\item[Split-grid self-comparison.] A model compared with itself on two disjoint halves of its 16 configurations. Any stable winner in this comparison is an error.
\item[Simulation run.] One simulated dataset in the disattenuation studies of Section~\ref{sec:sim} and Appendix~\ref{app:sim}.
\item[Simulated benchmark.] One whole simulated benchmark in the studies of Algorithm~\ref{alg:perproto} (Appendix~\ref{app:alg2-sim}).
\item[Comparison family.] All $2P$ contrasts on one benchmark, one per pair and protocol. Algorithm~\ref{alg:perproto} covers the family with simultaneous intervals that share the critical value $\qalpha$.
\item[\Certified.] We call a direction \emph{certified} when, in every protocol, its simultaneous interval excludes zero with the same sign. The family-wise error level is $\alpha$, asymptotically in the number of items, for items drawn independently from an item population and the fixed prompt grid, with every model scored the same way on the same items. It does not extend to unseen prompts or other harnesses.
\item[\ProtDep.] Every protocol's interval excludes zero, and the signs differ. We also call such a pair a reversal.
\item[\Refused.] In at least one protocol the simultaneous interval contains zero, so no winner is certified at level $\alpha$. A refusal is not evidence of a tie, and the intervals are still reported. A refused pair is either unresolved in both protocols or resolved in only one protocol.
\item[\Declined.] The flag that gated disattenuation (Algorithm~\ref{alg:gated}) attaches to its output when either reliability fails the gate. Algorithm~\ref{alg:gated} then returns the uncorrected estimate $\clip(\robs)$, the observed correlation clipped to $[-1,1]$. \Declined{} is a decision about a correction, and \Refused{} is a verdict about a pair of models.
\end{description}

\subsection{Which reliability each analysis uses}

Table~\ref{tab:analyses} lists every analysis that reports or uses a reliability, with the estimate behind it. Rows (1)--(3) and (5) feed a correction, so their reliability refers to the same number of replicates as the correlation being corrected (Section~\ref{sec:reliability}). Rows (4), (6) and (7) report a reliability as a finding and correct nothing.

\begin{table*}[t]\centering\scriptsize\setlength{\tabcolsep}{3pt}
\def\rag{\raggedright\let\\\tabularnewline}
\begin{tabular}{@{}p{3.1cm}p{2.6cm}p{2.3cm}p{3.1cm}p{4.6cm}p{0.9cm}@{}}
\toprule
Analysis & Units & Replicates & Numerator & Reliability estimate & Gate $\alpha$ \\
\midrule
\rag (1) Per-class demo (Sec.~\ref{sec:real}, App.~\ref{app:perclass}) & \rag about 500 examples per class & \rag 2 seeds & \rag correlation of two-seed means & \rag one-way $\ICC{1,2}$ on the same two seeds & .25 \\
\rag (2) Cross-seed correction (Sec.~\ref{sec:real}, App.~\ref{app:pooled}) & \rag 50,000 examples & \rag seeds 0--3 & \rag EL2N on seed 2 vs.\ dynamic uncertainty on seed 3 & \rag one-way $\ICC{1,1}$ from seeds 0 and 1 (in the $K{=}2$ pivot) & .1 \\
\rag (3) Same-seed comparison (App.~\ref{app:pooled}) & \rag 50,000 examples & \rag seeds 0--3 & \rag two-seed means of seeds 2 and 3 & \rag one-way $\ICC{1,2}$ from seeds 0 and 1 & .1 \\
\rag (4) Score reliability (Sec.~\ref{sec:real}, App.~\ref{app:score-rel}) & \rag 50,000 examples & \rag 8 seeds & -- & \rag crossed two-way $\mathrm{ICC}(C,1)$; intervals from one-way $\ICC{1,1}$ over random 4-of-8 seed subsets & -- \\
\rag (5) Simulations (Sec.~\ref{sec:sim}, App.~\ref{app:sim}) & \rag $N$ as stated in each row & \rag $K{=}3$ & \rag correlation of $K$-replicate means & \rag one-way $\ICC{1,K}$ on the same replicates & .1 \\
\rag (6) LLM item difficulty, $g(1)$ and $g(32)$ (Sec.~\ref{sec:gcoef}, App.~\ref{app:gtheory}) & \rag 2,365 items & \rag 32 configurations, both protocols pooled & -- & \rag crossed item $\times$ configuration G-theory & -- \\
\rag (7) Embeddings (Sec.~\ref{sec:embed}, App.~\ref{app:embed}) & \rag test images (held-out half) $\times$ 512 dimensions & \rag 3 seeds & -- & \rag per-dimension cross-seed reliability after a Procrustes fit on the other half & -- \\
\bottomrule
\end{tabular}
\caption{Each correction divides by a reliability for the same number of replicates as its numerator. Numerator: the observed correlation $\robs$ that a correction divides. --: not applicable, because the analysis reports a reliability and corrects nothing. $\mathrm{ICC}(C,1)$: single-seed reliability from the crossed example $\times$ seed design, which removes the shift that each seed applies to all examples.}
\label{tab:analyses}
\end{table*}
\section{Shared bias: the persistence study}\label{app:persist}

This synthetic study, in which the true persistence is known, shows a bias that dividing by a reliability cannot fix. Noise shared within one readout of a trajectory makes a persistent process look anti-persistent, and pairing two independent readouts recovers the true persistence (Figure~\ref{fig:persist}). In a persistent process, consecutive steps tend to move in the same direction. In an anti-persistent one, they tend to reverse.

\paragraph{Generator.} We simulate $N{=}400$ units, each a trajectory of $T{=}20$ steps in $D{=}16$ dimensions. The velocity follows an AR(1) process with coefficient $\phi$, so each step's velocity is $\phi$ times the previous one plus fresh Gaussian noise, and $\phi$ is the true persistence. Two independent readouts observe each trajectory. Each readout adds Gaussian noise to every position, with standard deviation 0.5, 1 or 2. A readout plays the role of a replicate here. Both readouts observe the same trajectories, so the study illustrates shared bias and does not audit training noise. The 27 regimes cross 9 values of $\phi$ in [$-$0.2, 0.6] with the 3 noise levels. Each regime has 40 simulation runs from a deterministic PCG32 random number generator.

\paragraph{Two estimators.} Both estimators work with displacements, the changes in observed position from one step to the next. The within-run estimator correlates consecutive displacements of the same readout. The noise at one observed position enters two consecutive displacements with opposite signs, which adds a negative term to their correlation. Dividing by a reliability rescales an estimate but cannot remove an additive term of this kind. The cross-replicate estimator with disattenuation correlates each displacement in one readout with the next displacement in the other readout, so the two noise terms are independent. The estimator then divides this correlation by the reliability of a single displacement, which is estimated as the correlation between the two readouts' displacements at the same step.

\paragraph{Only the cross-replicate estimator recovers the true persistence.} For each regime we average the estimates over its 40 runs and take the central 95\% of the run estimates as its range. The mean absolute error averages, over the 27 regimes, the distance between this mean and the true $\phi$. A regime is covered when the true $\phi$ lies inside its range. The within-run estimator has mean absolute error 0.44 and covers the true value in 0/27 regimes. The cross-replicate estimator with disattenuation has mean absolute error 0.0007 and covers it in 27/27. Table~\ref{tab:roadmap} quotes these two errors.

\begin{figure*}[t]\centering
\includegraphics[width=\textwidth]{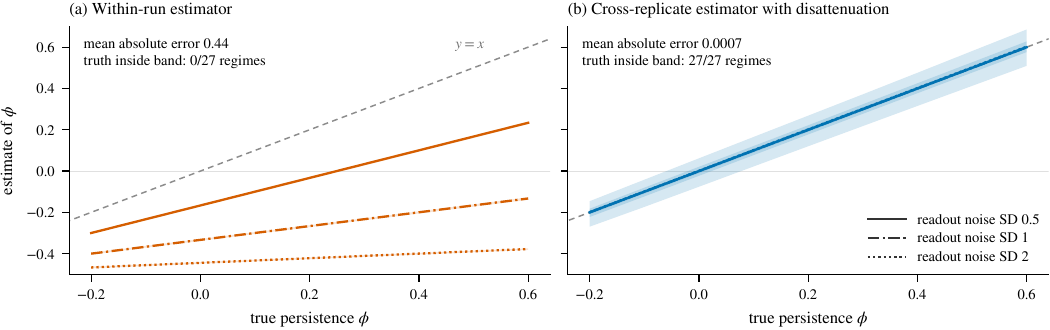}
\caption{With shared readout noise the within-run estimator reports spurious anti-persistence, and independent readouts recover the truth (synthetic). (a)~Within-run estimator. (b)~Cross-replicate estimator with disattenuation. Each line is the mean estimate over the 40 simulation runs of a regime, plotted against the true persistence $\phi$, with one line style per readout noise standard deviation (SD). Bands cover the central 95\% of runs. In (a) the bands are narrower than the lines. In (b) the three lines coincide, and the bands widen with the readout noise. Dashed line: $y{=}x$.}
\label{fig:persist}
\end{figure*}
\section{Proofs and the correction cap}\label{app:proofs}

This appendix proves Propositions~\ref{prop:known}--\ref{prop:magnitude} and tabulates the cap of Proposition~\ref{prop:magnitude}(a).
Expectations of nonnegative random variables take values in $[0,\infty]$, with $0\cdot\infty=0$.
The vector $\thetastar$ collects the target values of both measurements.
Conditional laws given $\thetastar$ are regular conditional distributions, which exist because every variable is real-valued.

\subsection{Proof of Proposition~\ref{prop:known}}\label{app:proof-known}

Proposition~\ref{prop:known} restates known facts.
Estimates of a variance component can be negative, so $\Pr(\rhat\le0)>0$ is a standard fact \citep{searle1992}.
\citet{zimmerman1997} gave Monte Carlo evidence that corrected correlations exceed one and vary widely.
Missing moments of an estimator that divides by a noisy coefficient are known from instrumental-variable estimators with one instrument \citep{kinal1980}.
The exact pivot~\eqref{eq:pivot} is the standard basis for $F$ intervals on the one-way ICC \citep{shrout1979}.
We give short proofs in our notation.
Throughout, $\rho_K<1$ and $F=(1-\rhat_K)/(1-\rho_K)$, which follows $F_{\nu_1,\nu_2}$ by~\eqref{eq:pivot}.

\paragraph{Part (i).}
The estimate is non-positive exactly when $1-\rhat_K\ge1$, that is, when $F\ge1/(1-\rho_K)$.
The $F$ distribution is continuous with an unbounded upper tail, so this event has probability $1-F_{\mathrm{cdf}}\big(1/(1-\rho_K);\nu_1,\nu_2\big)>0$.
On this event, $\sqrt{\rhat_K}$ is zero or not a real number.

\paragraph{Part (ii).}
Let $p_0=\Pr(\rhat_K\le0)>0$.
Because $\max(\rhat_K,\epsilon)\ge\epsilon$, we have $g_\epsilon^2\le1/\epsilon$ everywhere, with equality on $\{\rhat_K\le0\}$.
Hence
\[
p_0/\epsilon\;\le\;\E[g_\epsilon^2]\;\le\;1/\epsilon,
\]
so $\E[g_\epsilon^2]=\Theta(1/\epsilon)$, finite for each $\epsilon>0$ and unbounded as $\epsilon\to0$.
The divergence does not come only from non-positive estimates.
By~\eqref{eq:pivot}, $\rhat_K$ has density $f(t)=f_F\big((1-t)/(1-\rho_K)\big)/(1-\rho_K)$ for $t<1$, where $f_F$ is the $F_{\nu_1,\nu_2}$ density.
This density is continuous and positive on $(-\infty,1)$.
Fix $\delta\in(0,1)$, and let $f_{\min}>0$ and $f_{\max}<\infty$ be the minimum and maximum of $f$ on $[0,\delta]$.
For $\epsilon<\delta$,
\[
\E\big[g_\epsilon^2\,\mathbf 1\{\rhat_K>0\}\big]\;\ge\;\int_\epsilon^\delta\frac{f(t)}{t}\,dt\;\ge\;f_{\min}\log(\delta/\epsilon),
\]
which also diverges.
On the same event the first moment stays bounded, uniformly in $\epsilon$, because $g_\epsilon\le\rhat_K^{-1/2}$ there:
\[
\E\big[g_\epsilon\,\mathbf 1\{\rhat_K>0\}\big]\;\le\;f_{\max}\int_0^\delta t^{-1/2}\,dt+\delta^{-1/2}\;<\;\infty.
\]
The singularity of $1/\rhat_K$ at zero is thus just strong enough to break the second moment.

\paragraph{Part (iii).}
On $\{\rhat_X\rhat_Y>0\}$, $|\rplug|>1$ holds if and only if $\rhat_X\rhat_Y<\robs^2$, so the two events in (iii) coincide.
Fix a value $r\neq0$ of $\robs$ and let $\eta=|r|\in(0,1]$.
The event $\{0<\rhat_X\rhat_Y<r^2\}$ contains $\{\rhat_X\in(0,\eta),\,\rhat_Y\in(0,\eta)\}$, which has positive conditional probability by assumption.
The assumption holds when $\robs$, $\rhat_X$ and $\rhat_Y$ come from three disjoint seed blocks of the balanced Gaussian model.
By Lemma~\ref{lem:gauss} in Appendix~\ref{app:proof-split}, the three statistics are then independent given $\thetastar$, and each $\rhat_\ell$ has a conditional density that is positive on $(0,1)$.
Hence, for every $\thetastar$,
\[
\begin{aligned}
&\Pr\big(\rhat_X\in(0,\eta),\,\rhat_Y\in(0,\eta)\,\big|\,\thetastar,\robs\big)\\
&\qquad=\prod_{\ell\in\{X,Y\}}\Pr\big(\rhat_\ell\in(0,\eta)\,\big|\,\thetastar\big)>0,
\end{aligned}
\]
and averaging over $\thetastar$ given $\robs$ keeps the probability positive.

The assumption also holds when $\robs$ and the reliabilities share replicates, as in the per-class analysis and the simulations.
Take the balanced Gaussian model of Lemma~\ref{lem:gauss} with a single block of $K\ge2$ replicates, so that $\robs$ correlates the $K$-replicate means of $X$ and $Y$ and each $\rhat_\ell$ is $\ICC{1,K}$ or $\ICC{1,1}$ on the same replicates.
For each measurement the unit means are independent of the within-unit deviations, so $\MSW_X$ and $\MSW_Y$ are independent of all unit means and of each other, and $\nu_1\MSW_\ell/\sigma_\ell^2$ follows $\chi^2_{\nu_1}$.
The unit means determine $\robs$, $\MSB_X$ and $\MSB_Y$, and each $\MSB_\ell$ is positive almost surely.
Given the unit means, $\rhat_\ell\in(0,\eta)$ holds exactly when $\MSW_\ell$ falls in an open interval of positive length below $\MSB_\ell$, namely $\big((1-\eta)\MSB_\ell,\MSB_\ell\big)$ for $\ICC{1,K}$ and $\big((1-\eta)\MSB_\ell/(1+(K-1)\eta),\MSB_\ell\big)$ for $\ICC{1,1}$.
The $\chi^2$ density is positive on $(0,\infty)$, so the conditional probability of $\{\rhat_X\in(0,\eta),\,\rhat_Y\in(0,\eta)\}$ given the unit means is a product of two positive factors.
Averaging over the unit means given $\robs$ keeps it positive for almost every value of $\robs$.

\subsection{Proof of Proposition~\ref{prop:split}}\label{app:proof-split}

A measure-theory step first turns the interval bound in (H2) into a bound on integrals.

\begin{lemma}\label{lem:borel}
Let $\mu$ be a finite Borel measure on $(0,\delta)$ with $\mu((a,b))\ge c\,(b-a)$ whenever $0<a<b<\delta$, and let $\lambda$ be Lebesgue measure.
Then $\mu(B)\ge c\,\lambda(B)$ for every Borel set $B\subseteq(0,\delta)$, and $\int_{(0,\delta)}h\,d\mu\ge c\int_0^\delta h(t)\,dt$ for every Borel function $h\ge0$.
\end{lemma}

\begin{proof}
An open set $U\subseteq(0,\delta)$ is a countable disjoint union of open intervals $(a_j,b_j)$ with $0\le a_j<b_j\le\delta$.
Each of them is an increasing union of intervals $(a,b)$ with $0<a<b<\delta$, so continuity from below gives $\mu((a_j,b_j))\ge c\,(b_j-a_j)$, and summing gives $\mu(U)\ge c\,\lambda(U)$.
Now fix a Borel set $B\subseteq(0,\delta)$ and $\eta>0$.
A finite Borel measure on a metric space is outer regular, so there is an open set $U$ with $B\subseteq U\subseteq(0,\delta)$ and $\mu(U\setminus B)<\eta$.
Then $\mu(B)=\mu(U)-\mu(U\setminus B)>c\,\lambda(U)-\eta\ge c\,\lambda(B)-\eta$, and letting $\eta\to0$ gives $\mu(B)\ge c\,\lambda(B)$.
The integral inequality follows for nonnegative simple functions by linearity and for every Borel $h\ge0$ by monotone convergence.
\end{proof}

\paragraph{Proof of Proposition~\ref{prop:split}.}
Let $A=\{\rhat_X>0,\rhat_Y>0\}$.
By (H3), $0<\rhat_Y<1$ on $A$ almost surely, so
\[
\rplug^{\,2}\,\mathbf 1_A=\frac{\robs^2}{\rhat_X\rhat_Y}\,\mathbf 1_A\;\ge\;\robs^2\cdot\frac{\mathbf 1_A}{\rhat_X}\qquad\text{almost surely.}
\]
For Borel functions $f,g\ge0$, conditional independence gives $\E[f(\robs)\,g(\rhat_X,\rhat_Y)\mid\thetastar]=\E[f(\robs)\mid\thetastar]\,\E[g(\rhat_X,\rhat_Y)\mid\thetastar]$.
For indicators this is (H1), and it extends to $f,g\ge0$ by linearity and monotone convergence.
Applied to $f(r)=r^2$ and $g(x,y)=x^{-1}\mathbf 1\{x>0,y>0\}$, it gives
\[
\E\big[\rplug^{\,2}\,\mathbf 1_A\,\big|\,\thetastar\big]\;\ge\;\E\big[\robs^2\,\big|\,\thetastar\big]\cdot\E\big[\rhat_X^{-1}\mathbf 1_A\,\big|\,\thetastar\big].
\]
Fix $\thetastar\in\Theta_0$, with its constants $c$ and $\delta$ from (H2), and let $\mu(B)=\Pr(\rhat_X\in B,\,\rhat_Y>0\mid\thetastar)$ for Borel sets $B\subseteq(0,\delta)$.
Condition (H2) is the interval bound of Lemma~\ref{lem:borel} for $\mu$, so
\[
\E\big[\rhat_X^{-1}\mathbf 1_A\,\big|\,\thetastar\big]\;\ge\;\int_{(0,\delta)}t^{-1}\,\mu(dt)\;\ge\;c\int_0^\delta\frac{dt}{t}=\infty.
\]
Since $\E[\robs^2\mid\thetastar]>0$ on $\Theta_0$, the conditional expectation of $\rplug^{\,2}\mathbf 1_A$ is infinite there.
Tonelli's theorem, applied to the disintegration of the joint law over $\thetastar$, then gives
\[
\begin{aligned}
\E\big[\rplug^{\,2}\,\mathbf 1_A\big]&=\E\Big[\E\big[\rplug^{\,2}\,\mathbf 1_A\,\big|\,\thetastar\big]\Big]\\
&\ge\infty\cdot\Pr(\thetastar\in\Theta_0)=\infty.
\end{aligned}
\]
The argument holds for any noise distribution and allows $\rhat_X$ and $\rhat_Y$ to be dependent. \hfill$\square$

Condition (H3) holds for $\ICC{1,1}$ and $\ICC{1,K}$ when, almost surely, the unit means are not all equal and the replicates of some unit differ.

\paragraph{The Gaussian design.}
The next lemma shows that the balanced Gaussian model with three disjoint seed blocks meets (H1)--(H3).

\begin{lemma}[Three disjoint seed blocks]\label{lem:gauss}
For $\ell\in\{X,Y\}$, let $\hat\theta_{\ell,i,s}=\theta^\star_{\ell,i}+e_{\ell,i,s}$ for $N\ge2$ units.
The noise $e_{\ell,i,s}$ is Gaussian with variance $\sigma_\ell^2>0$, independent across units, replicates and measurements, and independent of $\thetastar$.
Split the replicates into three disjoint blocks.
Block 0 gives $\robs$, the Pearson correlation across units of the means of $m\ge1$ replicates of $X$ and of $Y$.
Block $X$ gives $\rhat_X$ from $K\ge2$ replicates of $X$, and block $Y$ gives $\rhat_Y$ from $K$ replicates of $Y$, each as $\ICC{1,K}$ or as $\ICC{1,1}$.
Then the following hold for every value of $\thetastar$.
\textup{(a)} Given $\thetastar$, the statistics $\robs$, $\rhat_X$ and $\rhat_Y$ are mutually independent.
\textup{(b)} Given $\thetastar$, each $\rhat_\ell$ has a continuous density that is positive on $(-\infty,1)$ for $\ICC{1,K}$ and on $(-1/(K-1),1)$ for $\ICC{1,1}$.
\textup{(c)} $\E[\robs^2\mid\thetastar]>0$.
Consequently (H1)--(H3) hold, with $\Theta_0$ equal to the whole space.
\end{lemma}

\begin{proof}
Given $\thetastar$, the three statistics are functions of $\thetastar$ and of three independent noise blocks, which gives (a).
For (b), fix $\thetastar$ and $\ell$, drop the index $\ell$, and let $\bar e_i$ be the mean noise of unit $i$ in the block.
The deviations $e_{i,s}-\bar e_i$ are independent of the means $\bar e_i$ for Gaussian noise, so $\MSW$ and $\MSB$ are independent.
After scaling by $\nu_1/\sigma^2$, $\MSW$ is central $\chi^2$ with $\nu_1$ degrees of freedom.
After scaling by $\nu_2/\sigma^2$, $\MSB$ is non-central $\chi^2$ with $\nu_2$ degrees of freedom and non-centrality $K\sum_i(\theta^\star_i-\bar\theta^\star)^2/\sigma^2$.
Both densities are continuous and positive on $(0,\infty)$, and the density $f_W$ of $\MSW$ is bounded because $\nu_1\ge2$.
The ratio $R=\MSW/\MSB$ therefore has density $f_R(r)=\int_0^\infty b\,f_W(rb)\,f_B(b)\,db$, where $f_B$ is the density of $\MSB$.
This density is positive on $(0,\infty)$, and it is continuous by dominated convergence, with the integrable bound $b\,\sup f_W\,f_B(b)$.
Since $\ICC{1,K}=1-R$ and $\ICC{1,1}=(1-R)/(1+(K-1)R)$ are smooth, strictly decreasing maps of $R$ onto $(-\infty,1)$ and $(-1/(K-1),1)$ with smooth inverses, (b) follows.
For (c), given $\thetastar$ the block-0 means have a non-degenerate Gaussian law on $\mathbb R^{2N}$.
The correlation is zero only if $\sum_i(x_i-\bar x)(y_i-\bar y)=0$, the zero set of a non-zero polynomial, which has Lebesgue measure zero.
Hence $\robs\neq0$ almost surely, and $\E[\robs^2\mid\thetastar]>0$.

Condition (H1) follows from (a).
For (H3), both forms of the ICC are below one exactly when $\MSW>0$, which holds almost surely.
For (H2), fix $\delta\in(0,1)$, let $c_X>0$ be the minimum on $[0,\delta]$ of the conditional density of $\rhat_X$, and let $p_Y^+=\Pr(\rhat_Y>0\mid\thetastar)>0$.
By (a), $\Pr(\rhat_X\in(a,b),\,\rhat_Y>0\mid\thetastar)\ge c_X\,p_Y^+\,(b-a)$ whenever $0<a<b<\delta$, and (c) gives the remaining part of (H2).
\end{proof}

\paragraph{The floored rate.}
With floor $\epsilon$ on each reliability, the plug-in becomes $\hat r_\epsilon=\robs\,\{\max(\rhat_X,\epsilon)\max(\rhat_Y,\epsilon)\}^{-1/2}$.
The proof of its rate uses the mutual independence of all three statistics in Lemma~\ref{lem:gauss}(a), which is more than (H1) states.

\begin{lemma}[Floored rate]\label{lem:floor}
In the setting of Lemma~\ref{lem:gauss}, let $p_\ell(\thetastar)=\Pr(\rhat_\ell\le0\mid\thetastar)$.
Then, for every $\epsilon>0$,
\[
C\,\epsilon^{-2}\;\le\;\E[\hat r_\epsilon^{\,2}]\;\le\;\epsilon^{-2},
\]
where $C=\E\big[\E[\robs^2\mid\thetastar]\,p_X(\thetastar)\,p_Y(\thetastar)\big]\in(0,1]$.
So the second moment of the plug-in with floor $\epsilon$ is $\Theta(\epsilon^{-2})$.
\end{lemma}

\begin{proof}
The upper bound holds because $\robs^2\le1$ and each floored reliability is at least $\epsilon$.
For the lower bound, Lemma~\ref{lem:gauss}(a) gives
\[
\E[\hat r_\epsilon^{\,2}\mid\thetastar]=\E[\robs^2\mid\thetastar]\prod_{\ell\in\{X,Y\}}\E\big[1/\max(\rhat_\ell,\epsilon)\,\big|\,\thetastar\big].
\]
As in Proposition~\ref{prop:known}(ii), each factor of the product is at least $p_\ell(\thetastar)/\epsilon$.
Taking expectations over $\thetastar$ gives the lower bound.
The integrand of $C$ is positive for every $\thetastar$ by Lemma~\ref{lem:gauss}(b) and~(c), so $C>0$.
\end{proof}

This floor acts on each reliability, as in the synthetic check of Appendix~\ref{app:interval}.
The plug-in in the simulations of Appendices~\ref{app:sim-design}--\ref{app:robust} floors the product of the two reliabilities, which is a different rule (Appendix~\ref{app:sim-design}).

\subsection{Proof of Proposition~\ref{prop:magnitude}}\label{app:proof-magnitude}

Write $\rho_\ell$ for the reliability of the $m$-replicate measurement $\ell$, and $G$ for the event that both gates pass.
We assume $\rho_X,\rho_Y>0$, since the true correlation $\rtrue$ is defined only then.
Let $L=1-1/(\qgate\qucb)$, which is positive when $\qgate\qucb>1$.

\paragraph{Part (a).}
Suppose $G$ occurs.
For each $\ell$ the gate passed, so $\MSB_\ell/\MSW_\ell>\qgate$ and $1-\rhat_\ell=\MSW_\ell/\MSB_\ell<1/\qgate$.
The exact-$F$ bound then satisfies $u^F_\ell=1-(1-\rhat_\ell)/\qucb>1-1/(\qgate\qucb)=L$.
If $\MSW_\ell=0$, the gate passes with $\rhat_\ell=1$ and $u^F_\ell=1>L$.
The max rule only raises the bound, so the pre-map bound is above $L$ in both branches.
For $m=K$ no map follows, and $|\rgate|\le|\robs|/\sqrt{u_Xu_Y}\le|\robs|/L$, because clipping only shrinks a value.
Every step uses only the definitions in Algorithm~\ref{alg:gated}, so part (a) needs no distributional assumption.

\paragraph{Part (b).}
For the hypothesis of (b), regard $u_\ell$ as computed from the data whether or not the gates pass; Algorithm~\ref{alg:gated} uses it only on $G$.
Outside $G$ the algorithm returns $\clip(\robs)$, and $|\clip(\robs)|\le|\robs|\le|\robs|/\sqrt{\rho_X\rho_Y}$ because $\rho_X,\rho_Y\le1$.
On $G$, if $u_\ell\ge\rho_\ell$ for both $\ell$, then $\sqrt{u_Xu_Y}\ge\sqrt{\rho_X\rho_Y}$ and
\[
|\rgate|\;\le\;\frac{|\robs|}{\sqrt{u_Xu_Y}}\;\le\;\frac{|\robs|}{\sqrt{\rho_X\rho_Y}}.
\]
The bound can therefore fail only on $G\cap(\{u_X<\rho_X\}\cup\{u_Y<\rho_Y\})$.
Each of the two events has probability at most $\alpha$, so the union bound limits the failure probability to $2\alpha$.
Only $\Pr(G\cap\{u_\ell<\rho_\ell\})\le\alpha$ is used, and we verify this weaker condition in the three cases below. \hfill$\square$

\paragraph{The exact-$F$ bound.}
Under the balanced Gaussian model with $m=K$, the pivot~\eqref{eq:pivot} gives
$\Pr(\rho_K\le u^F)=\Pr\big((1-\rhat_K)/(1-\rho_K)\le\qucb\big)=1-\alpha$,
whether or not the gate passes.

\paragraph{The max rule.}
The event $\{\rho_\ell\le\max(u^F_\ell,u^B_\ell)\}$ is the union of $\{\rho_\ell\le u^F_\ell\}$ and $\{\rho_\ell\le u^B_\ell\}$.
The larger bound therefore covers whenever either bound covers, and its coverage is at least the larger of the two coverages.

\paragraph{The Spearman--Brown map.}
Write $v_\ell$ for the bound before the map and $\mathrm{SB}(u)=(m/K)\,u/\big(1-(1-m/K)\,u\big)$ for the map.
On $G$, part (a) gives $v_\ell>0$, and $v_\ell\le1$ because $\rhat_\ell\le1$ and every bootstrap value of $\ICC{1,K}$ is at most one.
The map fixes $0$ and $1$.
Its derivative, $(m/K)/\big(1-(1-m/K)\,u\big)^2$, is positive on $[0,1]$, so the map is continuous and strictly increasing there.
By~\eqref{eq:rhoK}, the map sends the reliability of a $K$-replicate mean to that of an $m$-replicate mean.
Hence, on $G$, $v_\ell$ covers the $K$-replicate reliability if and only if $u_\ell=\mathrm{SB}(v_\ell)$ covers $\rho_\ell$, and $\Pr(G\cap\{u_\ell<\rho_\ell\})$ equals the probability that $G$ occurs and $v_\ell$ misses.
A mapped bound thus covers on $G$ exactly when the original does.

\begin{remark}[The data-dependent switch]
Algorithm~\ref{alg:gated} uses the max rule only when $|S_\ell|>0.5$.
The bound it divides by is $u^F_\ell$ on one branch and $\max(u^F_\ell,u^B_\ell)$ on the other, and both are at least $u^F_\ell$.
Under the balanced Gaussian model the hypothesis of (b) therefore holds whichever branch is taken.
Without the Gaussian model, coverage of $u^B_\ell$ alone does not suffice, because $u^F_\ell$ is used when $|S_\ell|\le0.5$.
The hypothesis then needs, for each measurement, a bound that covers on the branch where it is used, for example with conditional coverage of at least $1-\alpha$ given the branch.
Table~\ref{tab:scope} lists this condition as, for each measure, a bound that covers on the branch where it is used.
\end{remark}

\subsection{The correction cap}\label{app:cap}

\begin{table}[t]\centering\small\setlength{\tabcolsep}{3pt}
\begin{tabular}{@{}lrrrrrrr@{}}
\toprule
Setting & $N$ & $K$ & $\alpha$ & $\qgate$ & $\qucb$ & $\qgate\qucb$ & cap \\
\midrule
Table~\ref{tab:mse} & 20 & 3 & .1 & 1.615 & 1.730 & 2.793 & 1.56 \\
Table~\ref{tab:mse}, base & 50 & 3 & .1 & 1.357 & 1.392 & 1.889 & 2.12 \\
Table~\ref{tab:mse} & 100 & 3 & .1 & 1.243 & 1.258 & 1.563 & 2.77 \\
Table~\ref{tab:mse} & 500 & 3 & .1 & 1.103 & 1.106 & 1.220 & 5.55 \\
per-class & 500 & 2 & .25 & 1.062 & 1.062 & 1.128 & 8.79 \\
cross-seed size & 50,000 & 2 & .1 & 1.012 & 1.012 & 1.023 & 44.1 \\
\bottomrule
\end{tabular}
\caption{The cap on an accepted correction loosens as the number of units grows. $\qgate$ and $\qucb$ are the upper-$\alpha$ quantiles of $F_{\nu_2,\nu_1}$ and $F_{\nu_1,\nu_2}$, with $\nu_1=N(K-1)$ and $\nu_2=N-1$. Cap: $(1-1/(\qgate\qucb))^{-1}$, the largest factor by which an accepted correction can multiply $\robs$ when $m=K$ (Proposition~\ref{prop:magnitude}(a)). Settings: rows of Table~\ref{tab:mse}; per-class, close to the per-class analysis of Section~\ref{sec:real}; cross-seed size, $N$ and $K$ of the cross-seed data of Section~\ref{sec:real}. The cross-seed correction itself has $m=1$, so its cap is larger and follows from the bounds after this table.}
\label{tab:caps}
\end{table}

Table~\ref{tab:caps} gives the cap of Proposition~\ref{prop:magnitude}(a) at the settings we use.
As $N$ grows, both quantiles fall toward one, so their product approaches one and the cap loosens.
At the cross-seed setting the upper bounds sit just above the estimates, because $\qucb$ is close to one there.

\paragraph{Why the product exceeds one.}
Proposition~\ref{prop:magnitude} assumes $\qgate\qucb>1$, and this holds for every $\alpha<1/2$.
If $W$ follows $F_{\nu_1,\nu_2}$, then $1/W$ follows $F_{\nu_2,\nu_1}$, so $\qucb$ is the reciprocal of the lower-$\alpha$ quantile of $F_{\nu_2,\nu_1}$.
The product $\qgate\qucb$ is therefore the ratio of the upper-$\alpha$ to the lower-$\alpha$ quantile of one continuous distribution, which exceeds one when $\alpha<1/2$.

\paragraph{The cap after a Spearman--Brown map.}
When $\robs$ averages $m\neq K$ replicates, each pre-map bound is above $L=1-1/(\qgate\qucb)$ by part (a), and the map is increasing on $[0,1]$.
The correction factor therefore obeys
\[
\frac{1}{\sqrt{u_Xu_Y}}\;\le\;\frac{1}{\mathrm{SB}_{K\to m}(L)}\;=\;\frac{K-(K-m)\,L}{m\,L},
\]
where $\mathrm{SB}_{K\to m}$ is the map $\mathrm{SB}$ of Appendix~\ref{app:proof-magnitude}, written with its replicate counts.
For $m<K$ the map lowers each bound, so this cap exceeds $1/L$.
For $m>K$ it raises each bound, and the cap in Table~\ref{tab:caps} still applies.

\paragraph{The cap when $\ICC{1,1}$ enters the pivot directly.}
The cross-seed correction of Section~\ref{sec:real} has $m=1$ and $K=2$, and it puts $\ICC{1,1}$ into the $K$-replicate pivot without a map.
A passed gate gives $\MSB_\ell/\MSW_\ell>\qgate$, so $1-\ICC{1,1}=K\MSW_\ell/(\MSB_\ell+(K-1)\MSW_\ell)<K/(\qgate+K-1)$.
Each bound then exceeds $1-K/\big((\qgate+K-1)\qucb\big)$, and the correction factor is at most the reciprocal of this value whenever it is positive.
At the cross-seed size this cap is larger than the table's value for $m=K$ and smaller than the cap after a Spearman--Brown map.

\section{Disattenuation simulations}\label{app:sim}

This appendix gives the design behind Table~\ref{tab:mse} and Figures~\ref{fig:teaser}(a, b) and~\ref{fig:robust}, the full baseline table, and four further checks.

\subsection{Design}\label{app:sim-design}

Each simulation run draws target values for $N$ units from a bivariate normal distribution with unit variances and correlation $\rtrue$.
Each unit then receives $K$ replicates of each measurement.
The noise on a replicate has variance $\sigma^2=K(1-\rho)/\rho$, so the $K$-replicate mean has reliability $\rho$.
We compute $\robs$ as the correlation of the $K$-replicate means across units, and $\rhat_X$, $\rhat_Y$ as $\ICC{1,K}$ on the same replicates.
The numerator and the denominators therefore share replicates.
The split design of Proposition~\ref{prop:split} is checked separately in Appendix~\ref{app:interval}.

The base cell of Table~\ref{tab:mse} has $N=50$, $K=3$, \mbox{$\rho_X=\rho_Y=0.2$}, $\rtrue=0.4$ and Gaussian noise.
Each other row changes one setting.
The asymmetric row sets $\rho_X=0.3$ and $\rho_Y=0.8$, and the $t_3$ row draws the noise from a Student-$t$ distribution with 3 degrees of freedom.
The remaining rows set $\rtrue$ to 0 or 0.8, or $N$ to 20, 100 or 500.
We compute $F$ quantiles by bisection on the regularized incomplete beta function.

Table~\ref{tab:mse} and Table~\ref{tab:full} use 2,000 runs per row.
The study of $\alpha$ (Appendix~\ref{app:alpha}) uses 3,000 runs per cell, and the gate ablation (Appendix~\ref{app:ablation}) uses 4,000.
The robustness sweep of Appendix~\ref{app:robust} also uses 4,000 runs per cell, with $K=3$.
Figure~\ref{fig:teaser}(a, b) shows the Gaussian rows of that sweep.

The plug-in in Figure~\ref{fig:teaser} and in the note to Table~\ref{tab:mse} divides by $\max(\rhat_X\rhat_Y,10^{-9})^{1/2}$.
Its floor therefore acts on the product of the two reliabilities.
At this floor the MSE is finite but roughly proportional to the inverse of the floor.
The product of the two estimates is non-positive in a sizable share of low-reliability runs, so a smaller floor gives a larger MSE.
This product floor differs from the floor on each reliability, $\max(\rhat_\ell,\epsilon)$, used in Proposition~\ref{prop:known}(ii) and Lemma~\ref{lem:floor}.
The best post-hoc floor also floors each reliability, as $\max(\rhat_\ell,c)$.
The Monte Carlo s.e.\ of every Algorithm~\ref{alg:gated} cell in Table~\ref{tab:mse} is at most 0.005, and it is largest at $\rtrue=0.8$.

\subsection{Baselines}\label{app:baselines}

The floor-based baselines all search one grid of eight floors,
\[
\mathcal C=\{0.02,\ 0.06,\ 0.1,\ 0.15,\ 0.2,\ 0.3,\ 0.4,\ 0.5\}.
\]
\begin{itemize}
\item \emph{Uncorrected estimate.} $\clip(\robs)$, with no correction.
\item \emph{Clipped plug-in.} $\clip\big(\robs/\sqrt{\rhat_X\rhat_Y}\big)$.
\item \emph{Joint.} A moment estimator. It estimates the covariance and the two variances of the target values from covariances between different replicates; for the covariance, it pairs replicate $p$ of $X$ with replicate $q\neq p$ of $Y$. It returns their correlation clipped to $[-1,1]$, and 0 when a variance estimate is not positive.
\item \emph{jPSD.} The same $2\times2$ moment covariance, projected onto the positive-semidefinite matrices by eigenvalue truncation and read off as a correlation. The projection minimizes a Frobenius distance, so it need not match a constrained restricted maximum likelihood (REML) or maximum likelihood (ML) fit. At low reliability it pushes the correlation toward $\pm1$. In the low-reliability rows its MSE is no lower than that of the joint estimator and above that of the uncorrected estimate (Table~\ref{tab:full}).
\item \emph{Positive-part.} Replaces each reliability by $\max(\rhat,0)$ and declines to correct when either is not positive. On the balanced one-way design it equals the constrained REML reliability. The constrained ML estimate differs, because ML divides the between-unit sum of squares by $N$ where REML divides it by $N-1$ (Appendix~\ref{app:interval}).
\item \emph{Fiducial Bayesian shrinkage.} Averages the clipped corrected estimate over 24 fiducial draws of each reliability.
\item \emph{Bootstrap floor.} Refits the model at the estimated reliabilities, simulates 120 datasets, and picks the floor in $\mathcal C$ with the lowest MSE in the bootstrap.
\item \emph{Best post-hoc floor$^\dagger$.} The single floor $c\in\mathcal C$ for which $\clip\big(\robs\,\{\max(\rhat_X,c)\max(\rhat_Y,c)\}^{-1/2}\big)$ has the lowest MSE against the known truth, over all runs of a scenario. It is chosen after seeing the truth, so it cannot be deployed. No candidate returns the uncorrected estimate, so it can do worse than that estimate. At $\rtrue=0$ it reaches 0.08 against 0.02.
At $N=20$ it reaches 0.26 against 0.16.
\end{itemize}

\subsection{Full baseline table}\label{app:fulltable}

\begin{table*}[t]\centering\small\setlength{\tabcolsep}{5pt}
\begin{tabular}{@{}lcccccccccc@{}}
\toprule
Scenario & Clipped & Uncorr. & Joint (bad \%) & jPSD & Pos.-part & Bayes & Boot.\ floor & Alg.~\ref{alg:gated} & Best floor$^\dagger$ & Declined (\%) \\
\midrule
Base & .52 & \textbf{.12} & .31 (52) & .56 & .27 & .47 & .22 & \textbf{.12} & .14 & 87 \\
Asym.\ $.3/.8$ & .13 & \textbf{.06} & .11 (14) & .14 & .10 & .12 & .10 & \textbf{.06} & .06 & 42 \\
$t_3$ noise & .46 & \textbf{.12} & .30 (48) & .49 & .27 & .42 & .21 & \textbf{.12} & .13 & 84 \\
$\rtrue{=}0$ & .57 & \textbf{.02} & .28 (50) & .61 & .26 & .54 & .19 & .03 & .08 & 87 \\
$\rtrue{=}.8$ & .37 & .43 & .45 (60) & .45 & .34 & .34 & \textbf{.29} & .42 & .27 & 86 \\
$N{=}20$ & .69 & \textbf{.16} & .33 (68) & .72 & .31 & .65 & .29 & \textbf{.16} & .26 & 94 \\
$N{=}100$ & .34 & \textbf{.11} & .27 (39) & .39 & .22 & .32 & .19 & \textbf{.11} & .10 & 75 \\
$N{=}500$ & .07 & .11 & .08 (6) & .09 & .07 & .07 & .08 & \textbf{.05} & .04 & 12 \\
\bottomrule
\end{tabular}
\caption{Every deployable correction except Algorithm~\ref{alg:gated} is at least 0.10 worse than the uncorrected estimate in some row. MSE against the true correlation, with the runs and settings of Table~\ref{tab:mse}. Clipped: clipped plug-in. Uncorr.: uncorrected estimate. Joint (bad \%): moment estimator, with the share of runs in which its unclipped value is undefined or outside $[-1,1]$. jPSD: moment covariance projected by eigenvalue truncation. Pos.-part: positive-part rule. Bayes: fiducial Bayesian shrinkage. Boot.\ floor: bootstrap floor. $^\dagger$Best post-hoc floor, which cannot be deployed. Declined (\%): percentage of runs in which a gate of Algorithm~\ref{alg:gated} fails. Bold: lowest deployable MSE in the row (ties bold). The largest Monte Carlo s.e.\ over deployable methods and rows is 0.021, and that of Algorithm~\ref{alg:gated} cells is at most 0.005.}
\label{tab:full}
\end{table*}

Table~\ref{tab:full} adds four baselines and the row with 100 units to Table~\ref{tab:mse}.
At three decimals, with 500 units Algorithm~\ref{alg:gated} reaches 0.048 against 0.105 uncorrected and 0.041 for the best post-hoc floor.
In the asymmetric row its gain is small, from 0.063 to 0.058.
Every deployable correction other than Algorithm~\ref{alg:gated} is at least 0.10 worse than the uncorrected estimate in some row.
Positive-part and Bayesian shrinkage are worse than both the uncorrected estimate and Algorithm~\ref{alg:gated} in every row except two.
At $\rtrue=0.8$ both reach 0.34, against 0.43 uncorrected and 0.42 for Algorithm~\ref{alg:gated}.
All four values trail the bootstrap floor at 0.29.
With 500 units both reach 0.07, which beats the uncorrected 0.11 and trails Algorithm~\ref{alg:gated} at 0.05.
The unclipped joint estimator is undefined or outside $[-1,1]$ in 48--68\% of runs in the low-reliability rows.
The largest Monte Carlo s.e.\ over all deployable methods and rows is 0.021.

\subsection{Sensitivity to the level \texorpdfstring{$\alpha$}{alpha}}\label{app:alpha}

\begin{table}[t]\centering\small
\begin{tabular}{@{}lcccc@{}}
\toprule
Cell & $\alpha{=}.05$ & $.1$ & $.25$ & $.4$ \\
\midrule
$N{=}50$, $\rho{=}.1$ & \textbf{.150} & .153 & .160 & .212 \\
$N{=}50$, $\rho{=}.2$ & .124 & \textbf{.122} & .136 & .188 \\
$N{=}50$, $\rho{=}.5$ & \textbf{.051} & .053 & .063 & .076 \\
$N{=}20$, $\rho{=}.2$ & \textbf{.156} & .159 & .169 & .234 \\
$N{=}500$, $\rho{=}.2$ & .054 & \textbf{.044} & .047 & .057 \\
\bottomrule
\end{tabular}
\caption{The three smaller levels give similar MSE, and $\alpha{=}0.4$ is worse in every cell. MSE of Algorithm~\ref{alg:gated} with $K{=}3$, $\rtrue{=}0.4$ and 3,000 runs per cell. $\rho$: reliability of the 3-replicate mean, the same for both measurements. Bold: lowest MSE in the row.}
\label{tab:alpha}
\end{table}

No single level of $\alpha$ is best in every cell (Table~\ref{tab:alpha}).
The level 0.05 is lowest in three cells, and 0.1 is lowest in the other two.
Cell by cell, the three smaller levels stay within about 25\% of each other.
The level 0.4 is worse in every cell, by up to 54\% against the best level.
We use $\alpha=0.1$ throughout and did not tune it per cell.
These values come from a separate simulation with 3,000 runs per cell.
The same cells in Tables~\ref{tab:mse}, \ref{tab:full} and~\ref{tab:ablate} come from other runs and differ from these by Monte Carlo noise.

\subsection{Gate ablation}\label{app:ablation}

\begin{table}[t]\centering\small\setlength{\tabcolsep}{4pt}
\begin{tabular}{@{}rrccccc@{}}
\toprule
$N$ & $\rho$ & Uncorr. & Alg.~\ref{alg:gated} & No gate & Gate + plug-in & Floor .1 \\
\midrule
20 & .1 & \textbf{.184} & .185 & .471 & .187 & .703 \\
20 & .2 & \textbf{.157} & .158 & .347 & .162 & .575 \\
50 & .1 & \textbf{.150} & .152 & .359 & .156 & .519 \\
50 & .2 & \textbf{.123} & .124 & .209 & .131 & .356 \\
50 & .5 & .060 & \textbf{.055} & \textbf{.055} & .075 & .090 \\
500 & .2 & .104 & .046 & \textbf{.040} & .063 & .062 \\
\bottomrule
\end{tabular}
\caption{The gate matters at low reliability and the upper bound at high reliability. MSE against the truth with $\rtrue{=}0.4$, $K{=}3$, $\alpha{=}0.1$ and 4,000 runs per cell. $\rho$: reliability of the 3-replicate mean for both measurements. No gate: always divide by the upper bounds. Gate + plug-in: apply the gate, then divide by the point estimates. Floor .1: clipped plug-in with a fixed floor of 0.1. Bold: lowest deployable MSE in the row, the uncorrected estimate included (ties bold).}
\label{tab:ablate}
\end{table}

Both steps of Algorithm~\ref{alg:gated} matter, in different regimes.
Table~\ref{tab:ablate} removes one step at a time.
The variant without a gate always divides by the upper bounds.
The gate-plus-plug-in variant keeps the gate but divides by the point estimates when both gates pass.
The floor variant is the clipped plug-in with a fixed floor of 0.1.

Removing the gate hurts most at low reliability.
There it raises the MSE by up to 2.5$\times$, from 0.185 to 0.471 at $N=20$ and $\rho=0.1$.
Removing the upper bound costs less, but that variant is worse than Algorithm~\ref{alg:gated} in every row.
When reliability is well estimated, the roles reverse.
With 500 units, removing the gate does not hurt (0.040 against 0.046), while removing the bound does (0.063).
The fixed floor is worse than Algorithm~\ref{alg:gated} in every row.
In the four rows with $\rho\le0.2$ and $N\le50$, the uncorrected estimate is lowest and Algorithm~\ref{alg:gated} is close behind.

\subsection{Non-Gaussian noise and a bootstrap version}\label{app:robust}

\begin{figure*}[t]\centering
\includegraphics[width=\textwidth]{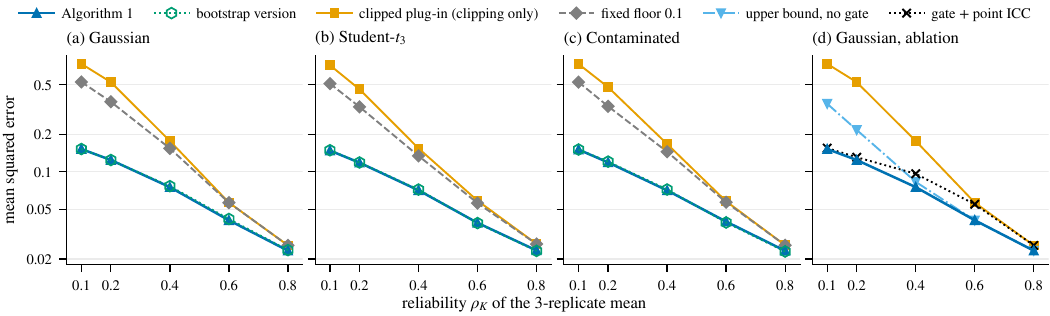}
\caption{Algorithm~\ref{alg:gated} stays stable under heavy-tailed and contaminated noise. Each panel shows the MSE (log scale) against the reliability of the 3-replicate mean, for $N{=}50$, $K{=}3$, $\rtrue{=}0.4$, $\alpha{=}0.1$ and 4,000 runs per point. Panels (a)--(c) use Gaussian, Student-$t_3$ and contaminated noise with equal variance, and compare Algorithm~\ref{alg:gated}, its bootstrap version, the clipped plug-in and the plug-in with a fixed floor of 0.1. Panel (d) takes Algorithm~\ref{alg:gated} apart under Gaussian noise, with the clipped plug-in, the upper bound without the gate, and the gate with the point ICC. Proposition~\ref{prop:magnitude} covers only Gaussian noise, so the Student-$t_3$ and contaminated panels are empirical checks.}
\label{fig:robust}
\end{figure*}

\begin{table}[t]\centering\small
\begin{tabular}{@{}lcc@{}}
\toprule
Method & $\rho{=}.1$ & $\rho{=}.4$ \\
\midrule
Clipped plug-in & .733 & .176 \\
Upper bound, no gate & .351 & .083 \\
Gate + point ICC & .155 & .096 \\
Algorithm~\ref{alg:gated} & \textbf{.152} & \textbf{.075} \\
\bottomrule
\end{tabular}
\caption{Under Gaussian noise the gate gives most of the gain at low reliability, and the upper bound adds a further gain at $\rho{=}0.4$. MSE from the sweep of Figure~\ref{fig:robust}. $\rho$: reliability of the 3-replicate mean. The same variants in Table~\ref{tab:ablate} come from a different run. Bold: lowest MSE in the column.}
\label{tab:robustablate}
\end{table}

Algorithm~\ref{alg:gated} is about as accurate under heavy-tailed and contaminated noise as under Gaussian noise (Figure~\ref{fig:robust}).
We keep the base cell of Table~\ref{tab:mse} ($N=50$, $K=3$, $\rtrue=0.4$) and use $\alpha=0.1$.
The sweep sets the reliability of the 3-replicate mean, for both measurements, to five levels between 0.1 and 0.8, with 4,000 runs per cell.
The replicate noise is Gaussian, Student-$t_3$ scaled by $\sigma/\sqrt3$, or a contaminated Gaussian.
The contaminated noise draws 95\% of errors from a Gaussian with variance $\sigma^2/2.2$ and 5\% from one with variance $25\sigma^2/2.2$.
All three noise types share the variance $\sigma^2=K(1-\rho)/\rho$.
At $\rho=0.4$ the MSE of Algorithm~\ref{alg:gated} is 0.075 with Gaussian noise and 0.071 with each of the other two.

The bootstrap version of Algorithm~\ref{alg:gated} (\texttt{bootstrap\_ucb} in the released results) resamples the $N$ units 500 times and keeps all $K$ replicates of each unit.
This version has its own gate.
The bootstrap version declines to correct when either lower bootstrap quantile of the ICC is not positive, and otherwise divides by the upper bootstrap quantiles.
The max rule of Section~\ref{sec:gated} is a different procedure.
The max rule keeps the exact-$F$ gate and uses a unit bootstrap with 400 resamples only as a second upper bound.
The bootstrap version is competitive, and its largest MSE gain over Algorithm~\ref{alg:gated} is 1.6\%.
Its total runtime is 12.9 times that of Algorithm~\ref{alg:gated}, and the mean of the per-cell runtime ratios is 15.5.
Table~\ref{tab:robustablate} gives the Gaussian ablation at $\rho=0.1$ and $\rho=0.4$.
Proposition~\ref{prop:magnitude} applies only to the Gaussian panels of Figure~\ref{fig:robust}, and the other panels are empirical checks.

\subsection{Intervals for \texorpdfstring{$\rtrue$}{r*} and a check of Proposition~\ref{prop:split}}\label{app:interval}

A Fieller-type interval for $\rtrue$ over-covers in our simulations but says little at low reliability.
We computed it \citep{fieller1954} with $\robs$ and the reliabilities on separate replicates, over 4,000 runs (\texttt{scripts/interval\_compare.py}).
The interval uses normal approximations and is set to $[-1,1]$ when its quadratic has no bounded solution.
The interval covers $\rtrue$ in at least 0.95 of runs at nominal 0.90.
At $\rho=0.2$ its mean width is about 1.9, on a scale that runs from $-$1 to 1.
On the balanced one-way design the positive-part reliability $\max(\rhat,0)$ equals the constrained REML estimate, a known closed form \citep{searle1992}.
The script computes the REML value from that closed form, so its check confirms only the code.
An interval that propagates Bonferroni-adjusted pivot bounds on the two reliabilities under-covers, and we do not report it as valid.

We check Proposition~\ref{prop:split} on synthetic data with \texttt{scripts/\allowbreak prop2\_check.py}.
The check uses separate seed blocks for $\robs$ and for each reliability, and a floor $\epsilon$ on each reliability, as in Appendix~\ref{app:proof-split}.
This per-factor floor differs from the product floor of Appendix~\ref{app:sim-design}.
Across runs, the correlation between $\robs$ and $\rhat$ is below 0.01 in absolute value.
The correlation given the target values is also near zero.
As $\epsilon$ falls from $10^{-2}$ to $10^{-6}$, the simulated second moment of the floored plug-in rises from 3 to $7\times10^{7}$.
This growth is consistent with the $\Theta(\epsilon^{-2})$ rate of Appendix~\ref{app:proof-split}.
The design is set in the script, listed in Appendix~\ref{app:code}.

\section{Difficulty scores on CIFAR}\label{app:scores}

This appendix gives the details behind Section~\ref{sec:real}.

\subsection{Score recipe and reliability}\label{app:score-rel}

We train ResNet-18 \citep{he2016} with a $3\times3$ stem on \mbox{CIFAR-100} and on \mbox{CIFAR-10} \citep{krizhevsky2009}, with eight seeds per dataset.
Training uses SGD with learning rate 0.1, Nesterov momentum 0.9 and weight decay $5\times10^{-4}$.
A cosine schedule runs for 30 epochs with batch size 256.
EL2N \citep{paul2021} is the norm of the prediction error; we take it at epoch 10.
\citet{he2024} define dynamic uncertainty as the standard deviation of the true-class probability within sliding windows of consecutive epochs, averaged over the windows.
We use a single window that covers every epoch after a 5-epoch warm-up, so our score is the variability of \citet{swayamdipta2020} computed over those epochs.
We keep the name dynamic uncertainty for it.

\begin{table}[b]\centering\small
\begin{tabular}{@{}lcc@{}}
\toprule
Score & CIFAR-100 & CIFAR-10 \\
\midrule
EL2N & .58 \ [.55, .60] & .55 \ [.45, .60] \\
Dynamic uncertainty & .84 \ [.83, .85] & .79 \ [.74, .82] \\
\bottomrule
\end{tabular}
\caption{Dynamic uncertainty is more reliable than EL2N on both datasets. Single-seed reliability over eight seeds and the 50,000 training examples of each dataset. Point values: crossed two-way $\ICC{C,1}$, which removes shifts a seed applies to every example. Brackets: 2.5th to 97.5th percentiles of the one-way $\ICC{1,1}$ over 300 random draws of 4 of the 8 seeds.}
\label{tab:scorerel}
\end{table}

Both scores are measured on all 50,000 training examples of each dataset.
Every seed scores every example, so we estimate reliability with the crossed design of examples by seeds.
Two-way variance components from generalizability theory give the consistency coefficient $\ICC{C,1}$ \citep{mcgraw1996}, the single-seed reliability after removing any shift that a seed applies to all examples.
To see how much this value depends on the seeds we happened to train, we recompute the one-way $\ICC{1,1}$ on 300 random draws of 4 of the 8 seeds.
Table~\ref{tab:scorerel} reports both.
On each dataset the subset intervals for EL2N and for dynamic uncertainty do not overlap, so the ordering of the two scores is robust.
The level of each reliability moves modestly with the choice of seeds, most for EL2N on \mbox{CIFAR-10}.
One minus the reliability is the share of seed noise, 42\% of the variance of a single-seed EL2N score on CIFAR-100 and 45\% on CIFAR-10.

\subsection{Per-class demonstration}\label{app:perclass}

The per-class analysis of Section~\ref{sec:real} runs once for each CIFAR-100 class.
The units are that class's roughly 500 training examples, and there are $K=2$ seeds.
We take $\robs$ as the correlation of the two-seed means of EL2N and dynamic uncertainty, and each reliability as $\ICC{1,2}$ on the same two seeds.
The plug-in divides by $\max(\rhat_X\rhat_Y,10^{-9})^{1/2}$, the product floor of Appendix~\ref{app:sim-design}.
The gate uses $\alpha=0.25$ for this two-seed demonstration.
Because the correlation and the reliabilities share the same two seeds, the analysis illustrates on real scores the failures that Proposition~\ref{prop:known} describes.
We do not offer its values as corrected estimates.
At $N\approx500$ the finite-sample bias of the ICC is about $0.004(1-\rho_K)$ (Appendix~\ref{app:bias}).
Recomputing the count of 70 classes requires regenerated CIFAR-100 labels and scores (Appendix~\ref{app:code}).

\subsection{Cross-seed correction on all 50,000 examples}\label{app:pooled}

\begin{table*}[t]\centering\small
\begin{tabular}{@{}llll@{}}
\toprule
Estimator & Estimate & $X$ / $Y$ seeds & Reliabilities: seeds, form \\
\midrule
Observed $\robs$ & .64 & 2 / 3 & -- \\
Alg.~\ref{alg:gated} & .88 & 2 / 3 & 0, 1; $\ICC{1,1}$ \\
Joint & .93 & all four & -- \\
Plug-in over 12 partitions & mean .93, s.d.\ .03, range .89--.98 & all ordered pairs & the other two; $\ICC{1,1}$ \\
Plug-in over 12 partitions, CIFAR-10 & mean .94, s.d.\ .04 & all ordered pairs & the other two; $\ICC{1,1}$ \\
Alg.~\ref{alg:gated}, same seeds & .94 & two-seed means on 2, 3 & 0, 1; $\ICC{1,2}$ \\
\bottomrule
\end{tabular}
\caption{The cross-seed estimates agree, but real data have no ground truth. Correlation of EL2N and dynamic uncertainty over the 50,000 CIFAR-100 training examples, except the row marked CIFAR-10. Partitions: the 12 ordered choices of the $X$ seed and the $Y$ seed among four, with reliabilities from the other two. They share seeds, so their spread is not an error bar. Alg.~\ref{alg:gated} puts $\ICC{1,1}$ into the exact-$F$ pivot for $K{=}2$, a conservative bound.}
\label{tab:pooledest}
\end{table*}

The cross-seed correction pools all 50,000 CIFAR-100 training examples and uses four of the eight seeds.
The analysis computes $\robs$ from EL2N on one seed and dynamic uncertainty on another, so the errors $e_X$ and $e_Y$ come from independent training runs.
The single-seed reliabilities $\ICC{1,1}$ come from the two remaining seeds, so, given the target values, $\robs$ is independent of $\rhat_X$ and $\rhat_Y$.
The point correction needs the first condition, and the design also provides the second.
The two reliabilities share their seeds, which the correction does not require.
Proposition~\ref{prop:split} is a property of the plug-in, shown on synthetic data in Appendix~\ref{app:interval}, and the real correction does not rely on it.

Choosing the $X$ seed and the $Y$ seed in order from the four, with the reliabilities taken from the other two, gives 12 partitions.
Table~\ref{tab:pooledest} lists the estimates.
Over the 12 partitions the plug-in has mean 0.93, s.d.\ 0.03 and range 0.89--0.98 on CIFAR-100.
On CIFAR-10 its mean is 0.94 with s.d.\ 0.04.
The partitions share seeds, so these spreads show how much the estimate moves between partitions and are not error bars.
Algorithm~\ref{alg:gated} divides by upper confidence bounds, which are slightly larger than the point estimates, so its 0.88 lies just below the plug-in range.
At $N=50{,}000$ the quantile $\qucb$ is 1.01, so the bounds stay close to the point estimates (Table~\ref{tab:caps}).
The joint estimator gives 0.93.

Algorithm~\ref{alg:gated} put $\ICC{1,1}$ directly into the exact-$F$ pivot for $K=2$.
This bound is conservative.
Computing the bound for the two-seed mean and mapping it to one seed with the Spearman--Brown step of Algorithm~\ref{alg:gated} gives the same 0.88.

A same-seed version correlates the two-seed means of both scores on seeds 2 and 3.
This version uses $\ICC{1,2}$ from seeds 0 and 1, and Algorithm~\ref{alg:gated} then gives 0.94.
That correlation keeps the errors the two scores share within a seed, which can inflate it (Appendix~\ref{app:errcorr}).

\subsection{Heteroscedasticity and coverage}\label{app:hetero}

\begin{figure*}[t]\centering
\includegraphics[width=\textwidth]{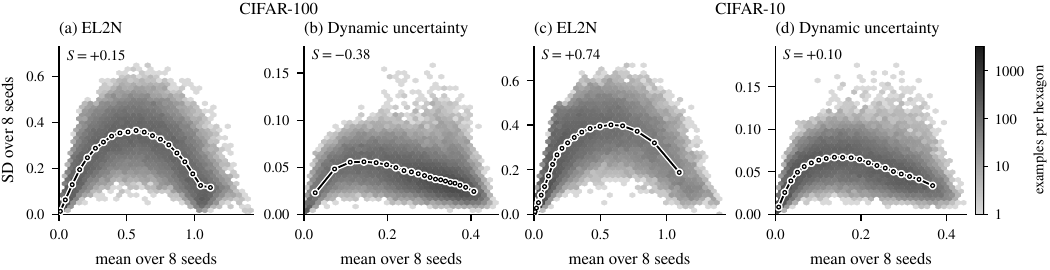}
\caption{Score variance rises with the score mean strongly only for CIFAR-10 EL2N. Each panel bins the 50,000 training examples of one dataset by their mean over eight seeds ($x$) and their standard deviation across the seeds ($y$); color shows the count on a log scale. The line joins the median standard deviation within 20 equal-count bins of the mean. Its markers crowd where examples are dense, so $S$ follows the branch that holds most examples. $S$: Spearman correlation across examples between the variance across seeds and the mean. (a)~CIFAR-100 EL2N; (b)~CIFAR-100 dynamic uncertainty; (c)~CIFAR-10 EL2N; (d)~CIFAR-10 dynamic uncertainty.}
\label{fig:hetero}
\end{figure*}

The exact-$F$ bound assumes homoscedastic noise, whose variance is the same for every unit.
Noise whose variance changes across units is heteroscedastic.
We check it with $S$, the Spearman correlation across examples between an example's variance over the eight seeds and its mean.
The variance uses $K-1$ in the denominator, with $N=50{,}000$ and $K=8$.
On CIFAR-100, $S$ is +0.15 for EL2N and $-$0.38 for dynamic uncertainty.
On CIFAR-10 it is +0.74 for EL2N and +0.10 for dynamic uncertainty (Figure~\ref{fig:hetero}).
The point correction needs only errors with conditional mean zero, so unequal variances affect only the coverage of the exact-$F$ bound.

Coverage here is the share of simulated runs in which the upper bound $u$ is at least the true reliability, at nominal 0.90.
Two calibrations measure it.
In \texttt{scripts/\allowbreak coverage\_hetero.py}, the exact-$F$ bound covers in 0.88--0.90 of runs for $S$ up to about 0.5.
At the strongest simulated heteroscedasticity, slightly above the CIFAR-10 EL2N value of 0.74, its coverage falls to \mbox{0.81--0.84}.
In \texttt{scripts/\allowbreak boot\_calibrate.py}, exact-$F$ covers in 0.79 of runs under strong heteroscedasticity.
A unit bootstrap with 400 resamples raises this to 0.87, with a wider bound.
That bootstrap bound slightly under-covers when the noise is homoscedastic, so exact-$F$ stays the default.
Where the exact-$F$ bound is valid, the max rule only widens it, and a fully bootstrapped version of Algorithm~\ref{alg:gated} took about 13 times the runtime (Appendix~\ref{app:robust}).
The max rule of Algorithm~\ref{alg:gated} uses the same 400 resamples.
The threshold of the max rule, $|S|>0.5$, is pre-specified from the first calibration and is not a theoretical cutoff.

Both calibrations compute $S$ against the simulated target values.
Algorithm~\ref{alg:gated} and the real-data checks compute it against the unit means, which carry replicate noise.
That noise attenuates $S$ when reliability is low, so at low reliability the check $|S_\ell|>0.5$ fires less often than the calibrations suggest.

The calibrations use positive values of $S$.
Negating every score maps a design with variance--mean correlation $S$ to its mirror image with $-S$ and leaves MSW, MSB, the ICC and its bound unchanged.
Each calibration at $+S$ therefore applies equally to its mirror image at $-S$.
On that basis we read CIFAR-100 dynamic uncertainty, at $-$0.38, against the near-nominal range.

The script \texttt{scripts/\allowbreak cond\_coverage.py} measures coverage given that the gate passes, also at nominal 0.90.
The conditional coverage is 1.00 at $\rho_K=0.3$ and $N=50$, and 0.94 at $N=200$.
Where the gate always passes, it is about 0.90.

\subsection{Sensitivity to correlated errors}\label{app:errcorr}

Errors that two scores share within a seed inflate a same-seed correlation, and averaging more seeds shrinks the inflation.
We simulate a single-seed reliability of 0.5 with $\rtrue=0.7$ and $N=2{,}000$.
With $K=2$ seeds, an error correlation of 0.3 inflates the same-seed estimate by +0.15.
An error correlation of 0.6 inflates it by +0.29.
The inflation halves with each doubling of $K$, and it is about zero when the errors are uncorrelated.
Computing $\robs$ across seeds removes this term, which is why Section~\ref{sec:whento} recommends it when $K\le2$ or when errors from the same run correlate strongly.

\subsection{Finite-sample bias of the ICC}\label{app:bias}

The ICC is biased downward by an amount that can be removed in closed form.
From the exact pivot~\eqref{eq:pivot},
\[
\E[\rhat_K]=1-(1-\rho_K)\,\frac{N-1}{N-3},
\]
so the bias is $-2(1-\rho_K)/(N-3)$.
The bias does not depend on $K$ and shrinks as $O(1/N)$.
The bias-corrected estimate (\texttt{scripts/icc\_bias.py}) is
\[
\rhat_K^{\mathrm{bc}}=1-(1-\rhat_K)\,\frac{N-3}{N-1}.
\]
The bias is below $10^{-4}$ for the pooled scores, with \mbox{$N=50{,}000$}.
For the per-class demonstration, with $N\approx500$, it is about $0.004(1-\rho_K)$.

Precision improves separately with the number of seeds.
Take reliability 0.5 and $N=1{,}000$ units.
The standard deviation of the single-seed reliability estimate is then 0.024 at $K=2$ and 0.016 at $K=4$.
The standard deviation falls to 0.013 at $K=10$ and 0.012 at $K=20$.
\section{LLM evaluation: models, prompts and reliability}\label{app:llm}
This appendix lists the models and benchmark items behind Section~\ref{sec:compare} (Appendix~\ref{app:models}).
The later subsections give the reliability analysis summarized in Section~\ref{sec:gcoef} (Appendices~\ref{app:gtheory} and~\ref{app:binary}) and the shot-format control of Section~\ref{sec:results} (Appendix~\ref{app:shotformat}).

\subsection{Models, benchmarks and prompts}\label{app:models}

\begin{table}[t]\centering\footnotesize\setlength{\tabcolsep}{4pt}
\begin{tabular}{@{}llll@{}}
\toprule
Model & Family & Size & Benchmarks \\
\midrule
\mbox{Qwen2.5-0.5B-Instruct} & Qwen & 0.5B & all four \\
\mbox{Qwen2-0.5B-Instruct} & Qwen & 0.5B & all four \\
\mbox{Qwen2.5-1.5B-Instruct} & Qwen & 1.5B & all four \\
\mbox{Qwen2-1.5B-Instruct} & Qwen & 1.5B & all four \\
\mbox{Qwen2.5-3B-Instruct} & Qwen & 3B & ARC \\
\mbox{Qwen2.5-7B-Instruct} & Qwen & 7B & ARC \\
\mbox{Qwen2-7B-Instruct} & Qwen & 7B & ARC \\
\midrule
\mbox{SmolLM2-135M-Instruct} & SmolLM2 & 135M & all four \\
\mbox{SmolLM2-360M-Instruct} & SmolLM2 & 360M & all four \\
\mbox{SmolLM2-1.7B-Instruct} & SmolLM2 & 1.7B & all four \\
\midrule
\mbox{Falcon3-1B-Instruct} & Falcon3 & 1B & all four \\
\mbox{Falcon3-3B-Instruct} & Falcon3 & 3B & all four \\
\mbox{Falcon3-7B-Instruct} & Falcon3 & 7B & ARC \\
\midrule
\mbox{Llama-3.2-1B-Instruct} & \mbox{Llama-3.2} & 1B & all four \\
\mbox{Llama-3.2-3B-Instruct} & \mbox{Llama-3.2} & 3B & all four \\
\bottomrule
\end{tabular}
\caption{The 15 models span four families and 135M to 7B parameters. Qwen2 and Qwen2.5 count as one family. Benchmarks: the benchmarks whose reported comparison family includes the model. ``ARC'' marks the four models included only in the ARC comparison family, and ``all four'' adds HellaSwag, MMLU and OpenBookQA.}
\label{tab:models}
\end{table}

\paragraph{Models and scoring.}
We compare the 15 instruction-tuned open models in Table~\ref{tab:models}, all of them Instruct checkpoints.
A model answers an item by comparing the next-token logits of the option letters after its own chat template, so scoring is deterministic.
Every model sees the same prompt grid of 4 instruction wordings $\times$ 4 option orders in each of two protocols, zero-shot and two-shot.
We chose open checkpoints of this size because the full 32-configuration grid runs on one consumer GPU for every model.
The near-chance \mbox{SmolLM2-135M} and \mbox{SmolLM2-360M} serve as negative controls.
We loaded the \mbox{Llama-3.2} checkpoints from the \texttt{unsloth/} repositories on the Hugging Face Hub and the other models from their developers' repositories, and the scripts in the code release give every repository ID.
Exact model revisions were not recorded, so the correctness tensors in the code release are the reproducibility target (Appendix~\ref{app:code}).

\paragraph{Benchmark items.}
ARC-Easy \citep{clark2018} contributes the 2,365 four-option test questions.
The 11 models marked ``all four'' are also compared on HellaSwag \citep{zellers2019}, MMLU \citep{hendrycks2021} and OpenBookQA \citep{mihaylov2018}, with 2,365 items from each.
For HellaSwag these are the first eligible validation contexts, and for MMLU the first 2,365 four-option test questions.
OpenBookQA items come from its training pool, because its test split has only 500 items.
The ARC comparison family therefore has 105 pairs, and each of the other three has 55.
The full prompt texts are in the code release.

\subsection{One configuration gives item-difficulty reliability of at most 0.65}\label{app:gtheory}

\begin{table}[t]\centering\footnotesize\setlength{\tabcolsep}{2.6pt}
\begin{tabular}{@{}lcrccccc@{}}
\toprule
 & & & & & & & $g(1)$/$g(32)$ \\
Model & acc & spr & wrd & ord & shot & res & (pooled) \\
\midrule
\mbox{Qwen2.5-0.5B} & .610 & 26 & .19 & .00 & .73 & .08 & .36/.95 \\
\mbox{Qwen2-0.5B} & .613 & 8 & .71 & .01 & .04 & .24 & .48/.97 \\
\mbox{Qwen2.5-1.5B} & .886 & 4 & .24 & .03 & .43 & .30 & .58/.98 \\
\mbox{Qwen2-1.5B} & .845 & 2 & .27 & .21 & .16 & .36 & .65/.98 \\
\mbox{Qwen2.5-3B} & .927 & 1.5 & .34 & .13 & .05 & .48 & .59/.98 \\
\mbox{Qwen2.5-7B} & .962 & 1.1 & .09 & .46 & .26 & .19 & .65/.98 \\
\mbox{Qwen2-7B} & .946 & 1.4 & .20 & .05 & .49 & .25 & .65/.98 \\
\midrule
\mbox{SmolLM2-135M} & .248 & 4 & .04 & .63 & .01 & .32 & .04/.54 \\
\mbox{SmolLM2-360M} & .255 & 3 & .03 & .51 & .00 & .46 & .04/.55 \\
\mbox{SmolLM2-1.7B} & .758 & 13 & .42 & .09 & .06 & .42 & .43/.96 \\
\midrule
\mbox{Falcon3-1B} & .762 & 4 & .05 & .04 & .81 & .10 & .51/.97 \\
\mbox{Falcon3-3B} & .871 & 4 & .14 & .21 & .48 & .17 & .55/.97 \\
\mbox{Falcon3-7B} & .945 & 1.6 & .07 & .69 & .01 & .23 & .62/.98 \\
\midrule
\mbox{Llama-3.2-1B} & .493 & 45 & .18 & .00 & .79 & .03 & .15/.85 \\
\mbox{Llama-3.2-3B} & .781 & 52 & .29 & .00 & .41 & .29 & .34/.94 \\
\bottomrule
\end{tabular}
\caption{For every model, the reliability of item difficulty under one configuration, $g(1)$, is at most 0.65, and the 32-configuration mean reaches $g(32)\ge0.85$ for every model above chance. ARC-Easy, 2,365 items, 32 configurations. acc: mean accuracy over the 32 configurations. spr: range of accuracy across them, in points. wrd, ord, shot, res: shares of the variance of configuration-level accuracy due to wording, option order, shot count and the residual (they sum to 1). $g(1)$, $g(32)$: reliability of item difficulty for one configuration and for the 32-configuration mean, pooled over both protocols. \mbox{SmolLM2-135M} and \mbox{SmolLM2-360M} sit near the 0.25 chance floor. The \mbox{Llama-3.2} spreads come from the protocol, because the fixed two-shot block costs them 21--29 points.}
\label{tab:llmscale}
\end{table}

\begin{figure*}[t]\centering
\includegraphics[width=\textwidth]{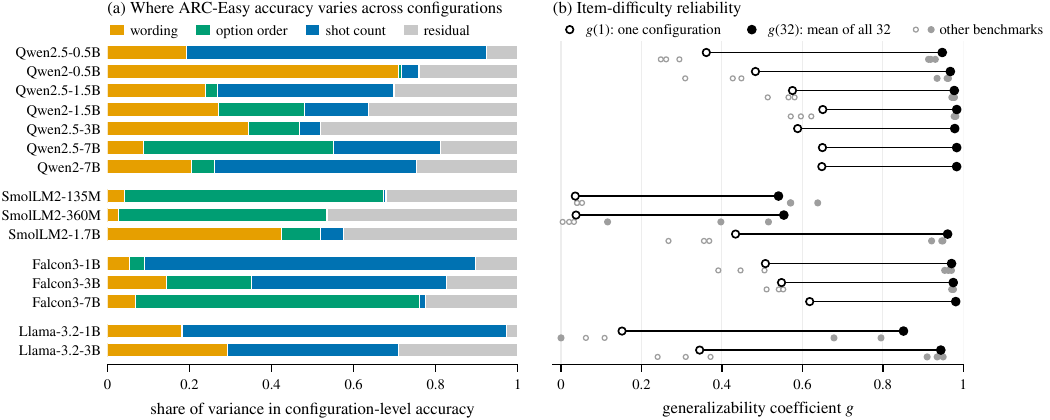}
\caption{One prompt configuration measures item difficulty poorly for every model, and the 32-configuration mean (both protocols pooled) is reliable whenever the model has item signal. (a)~Shares of the variance of configuration-level ARC-Easy accuracy due to wording, option order, shot count and the residual, for the 15 models. (b)~$g(1)$ (open circle) and $g(32)$ (filled circle) on ARC-Easy, joined by a line. Small grey circles give the same two coefficients on HellaSwag, MMLU and OpenBookQA for the 11 models compared there.}
\label{fig:gcoef}
\end{figure*}

For each model we record whether it answers each of the 2,365 ARC-Easy items correctly under each of the 32 configurations.
Table~\ref{tab:llmscale} summarizes this array in two ways.
The first splits the variance of the 32 configuration-level accuracies into wording, option order, shot count and a residual.
The second treats each item's correctness as a measurement of the item's difficulty for that model, with configurations as replicates.
In the crossed item $\times$ configuration design of generalizability theory, $g(1)$ is the reliability of item difficulty under one configuration and $g(32)$ the reliability of its mean over all 32.
The two are linked by the Spearman--Brown formula of Section~\ref{sec:reliability} with $K{=}32$.
Because $g(32)$ averages over both protocols, it is a pooled estimand.
Figure~\ref{fig:gcoef} shows the same quantities and adds the other three benchmarks.

\paragraph{One configuration is a noisy measure for every model.}
On ARC, $g(1)$ ranges from 0.04 to 0.65, and even the 7B models reach at most 0.65.
The mean over all 32 configurations is reliable for every model above chance, with $g(32)\ge0.85$.
On HellaSwag, MMLU and OpenBookQA, $g(1)$ is at most 0.62 for every model, and all 24 model--benchmark combinations with $g(1)\ge0.2$ reach $g(32)\ge0.91$ (Figure~\ref{fig:gcoef}(b)).

\paragraph{Averaging cannot create item signal.}
\mbox{SmolLM2-135M} and \mbox{SmolLM2-360M} answer near the chance floor, so their correctness says little about which items are hard.
For both, $g(1)$ is about 0.04, and averaging all 32 configurations lifts $g(32)$ only to about 0.55.
We use the two models as negative controls.

\paragraph{The coefficients are estimated precisely.}
We resampled items 2,000 times to get 95\% intervals for $g(1)$.
The intervals are narrow and lie well below one.
For \mbox{Qwen2.5-0.5B} and \mbox{SmolLM2-1.7B} the intervals are [0.35, 0.37] and [0.42, 0.45], respectively.
Both Qwen 7B models have intervals inside the range from 0.60 to 0.70.

\paragraph{Accurate models can still be unstable across prompts.}
Within Qwen2.5, the spread falls from 26 points at 0.5B to 4 at 1.5B and 1.1 at 7B, and Falcon3 behaves similarly.
\mbox{Llama-3.2-3B} breaks this pattern, with a zero-shot accuracy of 0.885 and yet a spread of 52 points.
The protocol drives this spread, since the fixed two-shot block drops its accuracy to 0.677.
\mbox{Llama-3.2-1B} falls the same way, from 0.637 to 0.349.

\paragraph{Which source of variation dominates depends on the model.}
Even two models of the same size can differ.
For \mbox{Qwen2.5-0.5B}, shot count explains 0.73 of the variance (item-bootstrap interval [0.69, 0.76]).
For \mbox{Qwen2-0.5B}, wording explains 0.71.
At 7B the spreads are about one point, so every share is noisy.
For \mbox{Qwen2.5-7B}, option order explains 0.46, but its interval runs from 0.10 to 0.75.

\paragraph{The \mbox{Llama-3.2} two-shot drop depends on the wording.}
Two inline examples cost \mbox{Llama-3.2-3B} 10, 12, 11 and 51 points under the four wordings.
Under each of the four option orders the cost is 20 to 21 points.
For \mbox{Llama-3.2-1B} the cost is large under every wording, 22 to 35 points.
The drop is an interaction between the fixed two-shot block and the wording under the chat template, so it is a property of the evaluation harness.
The \ProtDep{} verdicts of Section~\ref{sec:results} describe pairs under this harness, with its inline examples.

\subsection{A binary-scale check agrees with \texorpdfstring{$g(1)$}{g(1)}}\label{app:binary}
The $g$ coefficients fit Gaussian variance components to 0/1 correctness, so the binary scale could distort them.
As a check, we fit a beta-binomial model with one common intraclass correlation per model, by maximum likelihood over the 2,365 items (script in Appendix~\ref{app:code}).
This correlation matches the Gaussian $g(1)$ to within 0.05 for 14 of the 15 models.
The largest of these gaps is 0.046, for \mbox{Falcon3-7B}.
For \mbox{SmolLM2-135M} both methods give 0.04, so the low coefficients of the near-chance models are not a floor artifact.
The exception is \mbox{Llama-3.2-3B}, with a gap of 0.12.
The beta-binomial model assumes that all 32 configurations are exchangeable, and the two-shot drop of \mbox{Llama-3.2-3B} breaks that assumption.
The exception is consistent with treating shot count as a protocol (Section~\ref{sec:protocols}).
A crossed logistic generalized linear mixed model would avoid the Gaussian approximation altogether and could still model item and prompt effects while reporting per-protocol gaps, and we leave it open.

\subsection{Chat-formatted examples remove the two-shot drop}\label{app:shotformat}

\begin{table}[t]\centering\footnotesize\setlength{\tabcolsep}{4pt}
\begin{tabular}{@{}llcccc@{}}
\toprule
 & & \multicolumn{4}{c}{Wording} \\
\cmidrule(l){3-6}
Model & Condition & 1 & 2 & 3 & 4 \\
\midrule
\mbox{Llama-3.2-3B} & zero-shot & .883 & .881 & .883 & .882 \\
 & two-shot, inline & .782 & .752 & .775 & .366 \\
 & two-shot, chat & .895 & .895 & .893 & .895 \\
\midrule
\mbox{Llama-3.2-1B} & zero-shot & .489 & .679 & .693 & .671 \\
 & two-shot, inline & .268 & .322 & .405 & .343 \\
 & two-shot, chat & .741 & .719 & .735 & .730 \\
\midrule
\mbox{Qwen2.5-0.5B} & zero-shot & .704 & .653 & .697 & .683 \\
 & two-shot, inline & .514 & .463 & .592 & .573 \\
 & two-shot, chat & .731 & .729 & .731 & .736 \\
\midrule
\mbox{Falcon3-3B} & zero-shot & .863 & .871 & .860 & .864 \\
 & two-shot, inline & .878 & .878 & .880 & .876 \\
 & two-shot, chat & .881 & .880 & .879 & .882 \\
\bottomrule
\end{tabular}
\caption{Formatting the two examples as chat turns removes the two-shot drop. ARC-Easy accuracy on the 2,365 questions at the identity option order, for each of the four instruction wordings. Inline: the two solved examples written into the prompt before the question, the protocol used throughout the paper. Chat: the same examples given as earlier chat turns.}
\label{tab:shotformat}
\end{table}

\begin{figure}[t]\centering
\includegraphics[width=\linewidth]{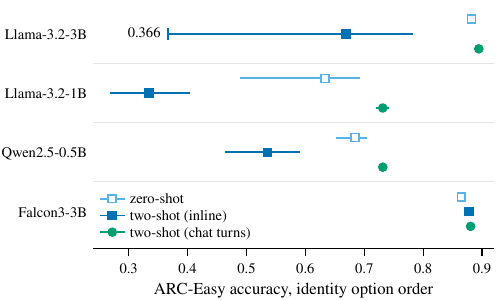}
\caption{Formatting the two examples as chat turns removes the two-shot drop of \mbox{Llama-3.2-3B}. ARC-Easy accuracy at the identity option order under zero-shot, two-shot with the examples written inline (the protocol used throughout), and two-shot with the same examples given as earlier chat turns. Markers show the mean over the four wordings, and bars show their range. The labeled point (0.366) is \mbox{Llama-3.2-3B} with inline examples under wording 4.}
\label{fig:shotformat}
\end{figure}

To separate two-shot prompting from the way the shots are presented, we re-evaluated four models on ARC-Easy with the two solved examples given as earlier chat turns.
Each example becomes a user question followed by the assistant's answer letter (script in Appendix~\ref{app:code}).
All runs use the identity option order and each of the four wordings.
Table~\ref{tab:shotformat} gives accuracy by wording, and Figure~\ref{fig:shotformat} plots the means and ranges.

Chat formatting removes the \mbox{Llama-3.2-3B} drop, including the fall to 0.366 under wording 4.
For \mbox{Llama-3.2-1B} and \mbox{Qwen2.5-0.5B}, the other two models that lose accuracy with inline examples, the chat format also makes the two-shot effect positive.
\mbox{Falcon3-3B}, which is insensitive to the inline block, stays where it was.
The inline format used throughout the paper is one of several ways to present the same two examples.
For \mbox{Llama-3.2-3B} under wording 4, the choice moves accuracy by more than 50 points.
Pooling the two protocols would average over a choice of this size.
\section{Per-protocol comparison: details and error rates}\label{app:alg2}

This appendix gives the details of Algorithm~\ref{alg:perproto} and the error rates behind Tables~\ref{tab:verdicts} and~\ref{tab:pooled}.
Every result uses $\alpha=0.05$ unless stated otherwise.

\begin{figure*}[t]\centering
\includegraphics[width=\textwidth]{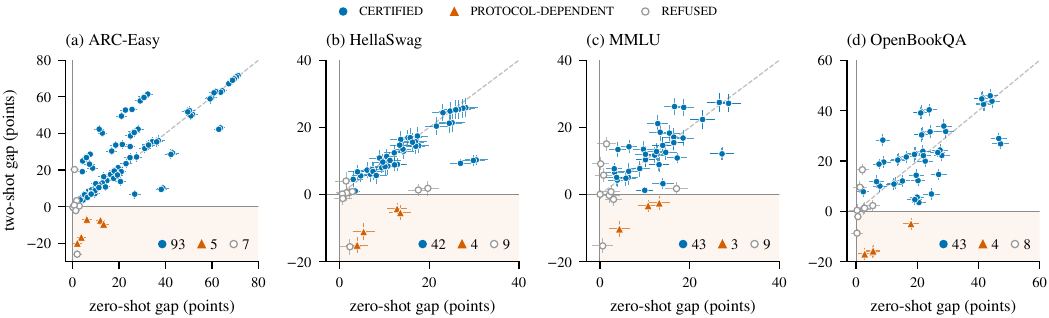}
\caption{Most pairs agree in direction across prompt settings; 16 are significant with opposite signs. One panel per benchmark. Each point is a model pair, with its zero-shot gap on the horizontal axis and its two-shot gap on the vertical axis, and each pair is ordered so that its zero-shot gap is at least zero. A gap is the first-named model's accuracy minus the second-named model's, in percentage points. Filled circles: \Certified{}; triangles: \ProtDep{}; open circles: \Refused{}. Thin lines: simultaneous 95\% intervals of Algorithm~\ref{alg:perproto} in each protocol. The shaded region marks opposite signs in the two settings, and the diagonal marks equal gaps.}
\label{fig:verdictmap}
\end{figure*}

\subsection{How Algorithm~\ref{alg:perproto} computes intervals and rank sets}\label{app:alg2-proc}

For pair $p=(A,B)$ and protocol $\pi$, the per-item gap $d^{p,\pi}_i$ averages the correctness difference over that protocol's 16 fixed configurations.
The standard error $\widehat{\mathrm{se}}_{p,\pi}$ is the sample standard deviation of the $N$ per-item gaps divided by $\sqrt N$.
We stack the centered vectors $d^{p,\pi}_i-\hat\Delta_{p,\pi}$ over all $2P$ pairs and protocols and compute their correlation matrix $\hat R$.
Draws of $Z\sim\mathcal N(0,\hat R)$ approximate the joint limiting distribution of the $2P$ $t$-statistics, and we use 10,000 draws.

The released code handles three numerical details.
First, a contrast whose centered per-item gaps have a sum of squares at or below machine precision is dropped from the calibration.
Second, pair contrasts among $M$ models span at most $M-1$ dimensions in each protocol, so $\hat R$ has rank at most $2(M-1)$ and we draw $Z$ through an eigendecomposition square root.
The square root keeps only the eigenvalues above the dimension of $\hat R$ times machine precision times $\max(\lambda_{\max},1)$, where $\lambda_{\max}$ is the largest eigenvalue.
Third, $\qalpha$ is the empirical $(1-\alpha)$ quantile of $\max_j|Z_j|$ over the draws, taking the larger order statistic when the quantile falls between two draws.

The deployed critical values are 3.513 on ARC, 3.336 on HellaSwag, 3.344 on MMLU and 3.380 on OpenBookQA.
For example, on ARC, \mbox{Qwen2.5-1.5B} minus \mbox{Qwen2-1.5B} is +4.6 points [+2.7, +6.4] zero-shot and +3.6 points [+1.8, +5.5] two-shot (Figure~\ref{fig:alg2}(b)).
Both intervals exclude zero on the positive side, so \mbox{Qwen2.5-1.5B} is certified ahead of \mbox{Qwen2-1.5B} in the sense of Section~\ref{sec:alg2}.
That direction holds at family-wise level $\alpha$, asymptotically in the number of items, for items drawn independently from an item population and the fixed prompt grid.

Within one protocol, a model is shown to be above another when their pair interval excludes zero in its favor.
The model's best possible rank is one plus the number of models shown to be above it, and its worst is $M$ minus the number of models it is shown to be above.
Because the intervals are simultaneous, the rank sets of all models contain the true ranks at once with asymptotic probability at least $1-\alpha$, on the fixed grid.
\citet{avelar2026} construct prompt-specific rank confidence sets for LLMs under a contextual Bradley--Terry model \citep{bradley1952} for human preferences.
Our rank sets adapt that idea to paired correctness on shared items.

\subsection{Other multiplicity corrections and a finite-sample variant}\label{app:alg2-mult}

The verdict rule of Algorithm~\ref{alg:perproto} needs only simultaneous intervals, so other ways to build them fit the same rule.
Single-step Bonferroni intervals \citep{dunn1961} use the normal quantile $z_{1-\alpha/(4P)}$, and \v{S}id\'ak intervals use $z_{(1+(1-\alpha)^{1/(2P)})/2}$ \citep{sidak1967}.
On ARC, with $2P=210$, these quantiles are 3.675 and 3.668, above the max-$t$ value of 3.513.
The intervals keep their centers and become wider, so under the same verdict rule either correction certifies a subset of the pairs that Algorithm~\ref{alg:perproto} certifies.
We computed these two values directly from the normal distribution.
Resampling versions of max-$t$, with bootstrap or permutation draws \citep{westfall1993}, estimate the same kind of quantile and change only $\qalpha$.
Step-down methods \citep{holm1979} can reject more hypotheses, but they do not give simultaneous intervals of this form.

\begin{remark}[A finite-sample variant]
Each per-item gap $d^{p,\pi}_i$ lies in $[-1,1]$.
Betting confidence intervals for the mean of a bounded variable \citep{waudbysmith2024} can therefore be built for every contrast, each with miscoverage $\alpha/(2P)$.
A union bound over the $2P$ intervals gives an FWER of at most $\alpha$ in finite samples for independent items, under the same verdict rule.
The price is power, because these intervals are wider than the max-$t$ intervals.
All results in this paper use the max-$t$ intervals.
\end{remark}

\subsection{Known-truth simulation}\label{app:alg2-sim}

We check the verdicts on simulated benchmarks where the true gaps are known.
Each simulated benchmark has six models and 2,365 paired binary items, with an item difficulty shared by all models.
We draw 1,000 independent simulated benchmarks per scenario.
The scenarios are a global null, in which all models are equal in both protocols, and three kinds of ranking.
Stable rankings keep the same order in both protocols, near-tie stable rankings do the same with small gaps, and reversals invert the order between the protocols.
All rates below are over pairs across simulated benchmarks.

With stable rankings, Algorithm~\ref{alg:perproto} returns the exact verdict, with the correct class and direction, for 89.5\% of pairs and returns \Refused{} for the rest.
The pooled comparison resolves 99.95\% of these pairs.
With near-tie stable rankings, the two rates are 4.5\% and 33.3\%.
With reversals, Algorithm~\ref{alg:perproto} returns \ProtDep{} for 68.3\% of pairs, never names a stable winner, and returns \Refused{} for the rest.
The pooled comparison cannot return \ProtDep{}, and it names a false stable winner for 64.6\% of reversed pairs.

Under the global null, the FWER is 5.1\% for Algorithm~\ref{alg:perproto} and 4.3\% for the pooled comparison, with a Monte Carlo s.e.\ of about 0.7\%.
Here the FWER is the share of simulated benchmarks with at least one false resolution.
A separate null simulation of Algorithm~\ref{alg:perproto} alone, over 2,000 simulated benchmarks, gives $86/2{,}000=4.3\%$ (Monte Carlo s.e.\ 0.5\%).
That second null is the binary null quoted in Section~\ref{sec:results}.
These simulations test the fixed-grid target under item sampling, and Appendix~\ref{app:transfer} measures transfer to new prompts.

\subsection{ARC verdicts}\label{app:verdicts}

Figure~\ref{fig:verdictmap} plots every pair on all four benchmarks.
On ARC, 93 of the 105 pairs are \Certified{}, and the other 12 are listed below.
For each \ProtDep{} pair we give the gap, the first-named model's accuracy minus the second-named model's, in points, written as zero-shot / two-shot, with the simultaneous intervals of the first pair in brackets.
Two \Refused{} pairs carry their gaps in the same form, and the other gaps and intervals are in the released results (Appendix~\ref{app:code}).
\begin{itemize}\setlength{\itemsep}{1pt}
\item \ProtDep{}, 5 pairs:
  \begin{itemize}\setlength{\itemsep}{0pt}
  \item \mbox{Qwen2.5-0.5B} vs.\ \mbox{Qwen2-0.5B}, +6.1 [+3.9, +8.4] / $-$6.9 [$-$9.2, $-$4.5]
  \item \mbox{Qwen2-1.5B} vs.\ \mbox{Llama-3.2-3B}, $-$3.8 / +16.6
  \item \mbox{SmolLM2-1.7B} vs.\ \mbox{Llama-3.2-3B}, $-$12.0 / +7.4
  \item \mbox{Falcon3-1B} vs.\ \mbox{Llama-3.2-3B}, $-$13.3 / +9.5
  \item \mbox{Falcon3-3B} vs.\ \mbox{Llama-3.2-3B}, $-$2.0 / +19.9
  \end{itemize}
\item \Refused{}, unresolved in both protocols, 2 pairs:
  \begin{itemize}\setlength{\itemsep}{0pt}
  \item \mbox{SmolLM2-135M} vs.\ \mbox{SmolLM2-360M}, both near the chance floor
  \item \mbox{Qwen2-7B} vs.\ \mbox{Falcon3-7B}, $-$0.1 / +0.3
  \end{itemize}
\item \Refused{}, resolved in only one protocol, 5 pairs:
  \begin{itemize}\setlength{\itemsep}{0pt}
  \item \mbox{Qwen2-0.5B} vs.\ \mbox{Llama-3.2-1B}
  \item \mbox{Qwen2.5-1.5B} vs.\ \mbox{Falcon3-3B}, +2.8 / +0.4
  \item \mbox{Qwen2.5-1.5B} vs.\ \mbox{Llama-3.2-3B}
  \item \mbox{Qwen2-1.5B} vs.\ \mbox{Falcon3-3B}
  \item \mbox{SmolLM2-1.7B} vs.\ \mbox{Falcon3-1B}
  \end{itemize}
\end{itemize}
Five of the 93 certified pairs flip sign on some individual configurations, on at most 8 of the 32, and each of the five involves a \mbox{Llama-3.2} model.
Certification concerns the mean direction within each protocol.

\subsection{Permutation null on real data}\label{app:permnull}

A permutation null keeps the real items and makes all models equal.
For each item, we reassign the models' complete two-protocol scores to model labels at random, which keeps the dependence among models and protocols within the item.
Every pair then has a true gap of zero in both protocols, so any interval that excludes zero is a false resolution.
We draw 200 permutations of each benchmark.
The share of permutations with at least one false resolution is 6.5\% on ARC, 6.0\% on HellaSwag, 7.0\% on MMLU and 5.0\% on OpenBookQA.
The Monte Carlo s.e.\ is 1.5--1.8\%, so these FWERs are consistent with the nominal 5\% within Monte Carlo error.

A second diagnostic works on one pair at a time and flips the sign of each item's gap at random, in both protocols at once, which makes that pair's true gap zero.
The diagnostic uses 200 flips per pair and max-$t$ critical values from an earlier calibration of the same comparison families.
It finds 2/21,000 false stable winners on ARC, 2/11,000 on HellaSwag and on MMLU, and 1/11,000 on OpenBookQA.
Because each pair is tested alone, the diagnostic measures per-pair errors, and the FWER comes from the permutation null above.

\subsection{Transfer to held-out and new prompts}\label{app:transfer}

\paragraph{Held-out wordings and orders.}
We choose two of the four wordings in each of the 6 possible ways and rerun Algorithm~\ref{alg:perproto} on the chosen two, with all four option orders.
We then check every certified direction on the other two wordings.
A check fails when the held-out mean gap, averaged over the held-out configurations within a protocol, has the opposite sign in either protocol.
The held-out-wording error rate is the share of failed checks.
On ARC, 11/555 checks fail (2.0\%), and on HellaSwag 4/257 fail (1.6\%).
On MMLU, 6/251 fail (2.4\%), and on OpenBookQA 13/261 fail (5.0\%).
Splitting the option orders in the same way gives 0/553, 0/248, 1/246 and 0/258 failures.
Full rank sets transfer less well.
Simultaneous rank coverage over all models, in both protocols, holds for 0--2 of the six wording splits and for 5--6 of the six order splits.
We therefore claim rank sets only for the fixed grid.

\paragraph{Four new wordings.}
After all verdicts were fixed, we wrote four new instruction wordings whose structure differs from the original four (the texts are in the code release).
All 15 models answered the 2,365 ARC-Easy questions under each new wording at the identity option order, in both protocols, which gives 8 new configurations per model.
Of the 93 certified ARC directions, 1 fails on the new wordings (Figure~\ref{fig:newwordings}).
The failing pair is \mbox{Qwen2.5-0.5B} vs.\ \mbox{Llama-3.2-1B}, two models with large configuration spreads (26 and 45 points).
\mbox{Qwen2.5-0.5B} is certified ahead, but on the new wordings the gap is $-$2.6 points zero-shot and +6.2 two-shot.
The pair is also one of the five certified pairs that already flipped on some original configurations (Appendix~\ref{app:verdicts}).
Across all certified pairs, 5 of the 744 individual new configurations flip, and all 5 protocol-dependent sign patterns replicate.
These are measured error rates, outside the guarantee.

\begin{figure}[tb]\centering
\includegraphics[width=\linewidth]{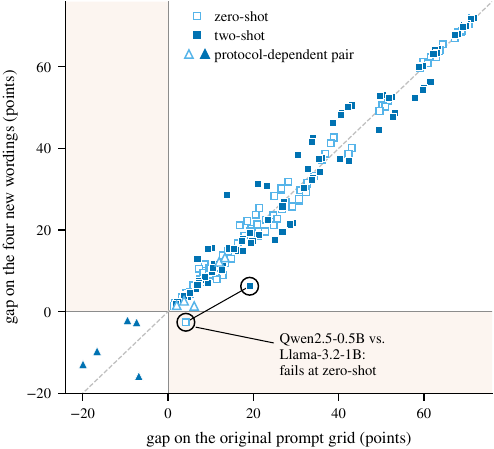}
\caption{All but one of the 93 certified ARC directions keep their sign on four wordings written after all verdicts were fixed, a measured rate outside the guarantee. Each point is one ARC pair in one protocol. The horizontal axis gives its gap on the original prompt grid, and the vertical axis its mean gap over the four new wordings at the identity option order. A gap is the first-named model's accuracy minus the second-named model's, in percentage points. Squares: \Certified{} pairs; triangles: the five \ProtDep{} pairs. Open markers: zero-shot; filled markers: two-shot. The highlighted pair, \mbox{Qwen2.5-0.5B} vs.\ \mbox{Llama-3.2-1B}, is the only certified direction that fails. Points in the shaded quadrants changed direction, and the diagonal marks equal gaps.}
\label{fig:newwordings}
\end{figure}

\subsection{Split-grid self-comparison}\label{app:selfcomp}

A split-grid self-comparison is a stress test of configuration bias.
We split a model's 16 configurations into two disjoint halves of 8 at random, use the same split in both protocols, and compare the model on one half with itself on the other.
The true gap is zero, so any stable winner is false.
We draw 200 random splits per model, and each method uses its critical value from the real comparison family.
The two halves form different grids, so this test lies outside the guarantee, which needs one grid shared by all models (Section~\ref{sec:harness}).

Algorithm~\ref{alg:perproto} names a false stable winner in 157/3,000 splits on ARC, 35/2,200 on HellaSwag, 32/2,200 on MMLU and 118/2,200 on OpenBookQA.
The pooled comparison does so in 581/3,000, 292/2,200, 254/2,200 and 670/2,200 splits.
In total, Algorithm~\ref{alg:perproto} names a false stable winner in 342/9,600 splits (3.6\%) and the pooled comparison in 1,797/9,600 (18.7\%).
Pooling produces 5.3 times as many false stable winners.
Figure~\ref{fig:selfcomp} shows that the rate grows with a model's configuration spread, the range of its accuracy across the 32 configurations.

\begin{figure}[tb]\centering
\includegraphics[width=\linewidth]{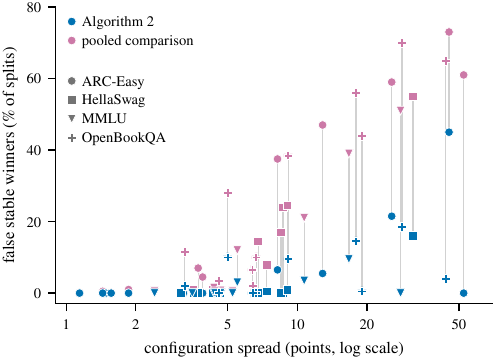}
\caption{False winners in split-grid self-comparisons grow with a model's configuration spread, and pooling multiplies them. Each point is one model on one benchmark. The horizontal axis is the model's configuration spread, the range of its accuracy across the 32 configurations in points (log scale). The vertical axis is the share of 200 random splits of its grid into two halves in which a method names the model a stable winner over itself, which is always an error. Thin vertical lines join the two methods' rates for the same model. Colors distinguish Algorithm~\ref{alg:perproto} from the pooled comparison, and marker shapes distinguish the benchmarks.}
\label{fig:selfcomp}
\end{figure}

\subsection{The pooled comparison}\label{app:pooled-baseline}

\begin{figure}[tb]\centering
\includegraphics[width=\linewidth]{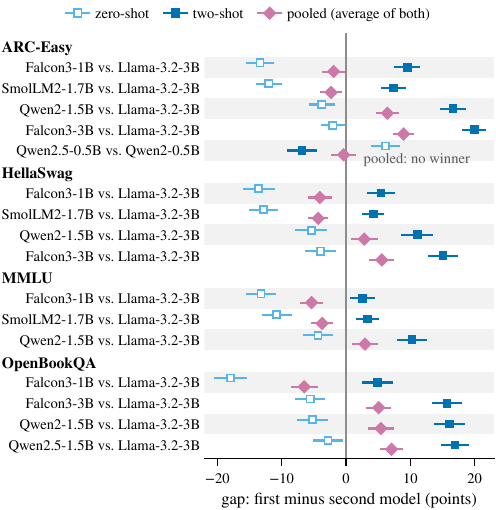}
\caption{Pooling the zero-shot and two-shot protocols names a stable winner for 15 of the 16 reversed pairs. Each row is one pair, and rows are grouped by benchmark and sorted by the zero-shot gap. A gap is the first-named model's accuracy minus the second-named model's, in percentage points. Open and filled squares: simultaneous 95\% intervals of Algorithm~\ref{alg:perproto} at zero-shot and two-shot. Diamonds: the pooled comparison's interval for the average of the two protocols.}
\label{fig:reversals}
\end{figure}

The pooled comparison averages each pair's per-item gaps over the two protocols with equal weight.
The pooled pair intervals are then calibrated jointly by max-$t$ from their estimated correlation, as in Algorithm~\ref{alg:perproto}.
The two methods share the paired design and the multiplicity control and differ only in the target, which for the pooled comparison is one average gap per pair.
On real data the pooled comparison resolves 100/105 ARC pairs and 254/270 pairs overall.
The resolved pairs include 15 of the 16 pairs whose protocol intervals have opposite signs (Figure~\ref{fig:reversals}).
For \mbox{Qwen2-1.5B} minus \mbox{Llama-3.2-3B}, it reports +6.4 points, the average of $-$3.8 zero-shot and +16.6 two-shot.
Under the simulated global null its FWER is 4.3\% (Appendix~\ref{app:alg2-sim}), so its own test is calibrated.
The pooled comparison errs on reversals because averaging removes the protocol distinction before any interval is built.

Three variants, none of which we ran, could recover some of the power that Algorithm~\ref{alg:perproto} gives up when the protocols agree, each at a cost.
Testing for an interaction and pooling only when none is found chooses the analysis after seeing the data, so the family-wise guarantee is lost.
A pooled contrast fixed in advance can join the same max-$t$ family with its own verdict, but it raises $\qalpha$ for every pair and answers only the average-gap question.
Empirical-Bayes shrinkage controls the average error across pairs, a weaker target than family-wise error.

\subsection{Sensitivity to \texorpdfstring{$\alpha$}{alpha}}\label{app:alg2-alpha}

Verdict counts barely change with $\alpha$ (Table~\ref{tab:alg2alpha}).
On every benchmark, the counts at $\alpha=0.01$, 0.05 and 0.1 differ by at most two pairs.
The command that reproduces the table is listed in Appendix~\ref{app:code}.

\begin{table}[tb]\centering\small
\begin{tabular}{@{}lccc@{}}
\toprule
 & $\alpha{=}.01$ & $\alpha{=}.05$ & $\alpha{=}.1$ \\
\midrule
ARC & 92/4/9 & 93/5/7 & 93/5/7 \\
HellaSwag & 42/4/9 & 42/4/9 & 42/5/8 \\
MMLU & 43/3/9 & 43/3/9 & 43/3/9 \\
OpenBookQA & 43/4/8 & 43/4/8 & 45/4/6 \\
\bottomrule
\end{tabular}
\caption{Verdict counts change by at most two pairs between $\alpha=0.01$ and $\alpha=0.1$. Each entry gives the number of \Certified{} / \ProtDep{} / \Refused{} pairs. ARC has 105 pairs among 15 models, and the other benchmarks have 55 pairs among 11 models. Table~\ref{tab:verdicts} reports $\alpha=0.05$.}
\label{tab:alg2alpha}
\end{table}

\subsection{Replication on three more benchmarks}\label{app:crossbench}

The pattern seen on ARC repeats on HellaSwag, MMLU and OpenBookQA (Table~\ref{tab:verdicts} and Figure~\ref{fig:verdictmap}).
Their comparison families contain the same 11 models, with 55 pairs each.
The three 7B models and \mbox{Qwen2.5-3B} are included only in the reported ARC comparison family.
Every protocol-dependent pair on these three benchmarks involves \mbox{Llama-3.2-3B}.
As on ARC, \mbox{Llama-3.2-3B} leads its partner zero-shot and trails it two-shot.
For \mbox{Qwen2-1.5B} minus \mbox{Llama-3.2-3B}, the zero-shot and two-shot gaps are $-$5.4 and +11.1 points on HellaSwag.
The same gaps are $-$4.3 and +10.2 on MMLU, and $-$5.2 and +16.1 on OpenBookQA.
\mbox{Llama-3.2-3B} also has the largest configuration spread on each of these benchmarks, at 31.6, 27.9 and 43.8 points.
Across all four benchmarks, 44 pairs are certified on at least two benchmarks, and each keeps the same direction wherever it is certified.
\section{Embedding details}\label{app:embed}

\paragraph{Training and features.}
We train ResNet-18 on \mbox{CIFAR-100} with standard augmentation, using a $3\times3$ stem and no max-pooling.
The optimizer is SGD with learning rate 0.1, Nesterov momentum 0.9, weight decay $5\times10^{-4}$ and a cosine learning-rate schedule.
We train three seeds with this recipe.
For each seed we keep the 512-dimensional penultimate features, taken after global average pooling, for every test image.
We repeat the study on CIFAR-10.

\begin{table}[t]\centering\small\setlength{\tabcolsep}{3pt}
\begin{tabular}{@{}lcc@{}}
\toprule
Cross-seed measure & CIFAR-100 & CIFAR-10 \\
\midrule
per-dimension, raw & .003 & .005 \\
per-dimension, after held-out Procrustes & .813 & .902 \\
pairwise distances & .823 & -- \\
\midrule
agreement: linear CKA & .852 & .924 \\
\bottomrule
\end{tabular}
\caption{Seeds disagree on the coordinates of the features and agree on the representation. Penultimate features of three \mbox{ResNet-18} seeds, evaluated on the half of the test images not used to fit the alignment. Rows above the rule are reliabilities. Raw: before alignment. CKA is a similarity index, shown as agreement. --: not computed.}
\label{tab:embed}
\end{table}

\paragraph{Why the coordinates are free.}
The network's logits are $Wh+c$, where $h$ is the penultimate feature vector.
For any orthogonal matrix $Q$, the features $Qh$ with output weights $WQ^\top$ give the same logits.
The loss therefore cannot prefer one basis over another through the output layer.
The ReLU at the end of the last block makes every feature nonnegative before pooling, and a rotation of a nonnegative vector usually has negative coordinates.
A network with this architecture cannot produce every rotated version of its features, so some bases are favored and the symmetry is approximate.
Section~\ref{sec:embed} measures how large the coordinate effect is in practice.

\paragraph{Measures.}
Let $Z_a$ hold the features of seed $a$, with one row per image and one column per dimension.
For per-dimension reliability and for the alignment, we standardize each column across images.
Per-dimension reliability between seeds $a$ and $b$ is
\[
\bar r(Z_a,Z_b)=\frac{1}{512}\sum_{j=1}^{512}\operatorname{corr}_i\big(Z_{a,ij},\,Z_{b,ij}\big),
\]
the correlation across images of dimension $j$ under the two seeds, averaged over dimensions.
With two seeds as replicates, each correlation estimates the single-replicate reliability $\rho_1$ of that dimension (Section~\ref{sec:reliability}).
We split the test images at random into a fit half $F$ and an evaluation half $E$.
On $F$ we compute the orthogonal Procrustes map $R=\arg\min_{R^\top R=I}\lVert Z_a^{F}R-Z_b^{F}\rVert$ in the Frobenius norm \citep{schonemann1966}.
Its solution is $R=UV^\top$, where $(Z_a^{F})^\top Z_b^{F}=USV^\top$ is a singular value decomposition.
The aligned reliability is $\bar r(Z_a^{E}R,\,Z_b^{E})$, computed only on images that played no part in the fit.
For pairwise distances we draw random pairs of evaluation images and correlate their Euclidean distance under seed $a$ with their distance under seed $b$, as in representational similarity analysis \citep{kriegeskorte2008}.
Linear CKA is $\lVert A^\top B\rVert^2/(\lVert A^\top A\rVert\,\lVert B^\top B\rVert)$ in the Frobenius norm, where $A$ and $B$ are the column-centered evaluation features as extracted \citep{kornblith2019}.
Distances and linear CKA are unchanged by any orthogonal map of either seed's features, so they need no alignment.
The CIFAR-100 values in Table~\ref{tab:embed} are computed on the evaluation half and averaged over all pairs of the three seeds.

\paragraph{The aligned value is stable.}
The reference seed is the seed whose coordinates the others are mapped onto.
Changing it keeps the aligned CIFAR-100 reliability within [0.812, 0.814].
Fitting the alignment on 25--75\% of the test images gives 0.807--0.817.
On CIFAR-10, per-dimension reliability is 0.005 before alignment and 0.902 after it, and linear CKA is 0.924.
We did not compute pairwise-distance reliability for CIFAR-10.
We report linear CKA as agreement, because it is a similarity index that is not calibrated as a reliability.

\paragraph{Regenerating the features.}
The three CIFAR-100 feature matrices behind Table~\ref{tab:embed} are not in the released files.
The script \texttt{scripts/\allowbreak gauge\_probe.py} retrains the three seeds on public CIFAR-100 and extracts the features again (Appendix~\ref{app:code}).
\section{Code and data}\label{app:code}

Code, data and prompt texts are at \url{https://github.com/ziangni-sys/measured-or-real}.

\begin{table*}[!t]\centering\scriptsize\setlength{\tabcolsep}{4pt}\renewcommand{\arraystretch}{1.12}
\begin{tabular}{@{}p{4.6cm}p{7.9cm}p{4.7cm}@{}}
\toprule
\raggedright Result & \raggedright Command & \raggedright Data in the repository \tabularnewline
\midrule
\raggedright Tab.~\ref{tab:mse}, Tab.~\ref{tab:full}: simulation MSE &
\raggedright \texttt{erel disatt2} &
\raggedright -- \tabularnewline
\raggedright Tab.~\ref{tab:ablate}: gate ablation &
\raggedright \texttt{erel ablate} &
\raggedright -- \tabularnewline
\raggedright Fig.~\ref{fig:teaser}(a, b), Tab.~\ref{tab:robustablate}, Fig.~\ref{fig:robust}: robustness sweep &
\raggedright \texttt{scripts/disatt\_robustness.py -{}-repetitions 4000 -{}-bootstrap-repetitions 500 -{}-seed 20270720}; figures: \texttt{fig1\_teaser.py}, \texttt{figD1\_robustness.py} &
\raggedright \texttt{results/disatt\_robustness\_full\_\allowbreak 4000x500\_seed20270720/} \tabularnewline
\raggedright Prop.~\ref{prop:split}: simulation check (App.~\ref{app:interval}) &
\raggedright \texttt{scripts/prop2\_check.py} &
\raggedright -- \tabularnewline
\raggedright Fieller interval and REML closed form (App.~\ref{app:interval}) &
\raggedright \texttt{scripts/interval\_compare.py} &
\raggedright -- \tabularnewline
\raggedright Persistence study (App.~\ref{app:persist}, Fig.~\ref{fig:persist}) &
\raggedright \texttt{erel synth data/synth.csv}; figure: \texttt{figB1\_persistence.py} &
\raggedright \texttt{data/synth.csv} \tabularnewline
\raggedright Score reliabilities and variance--mean check (Tab.~\ref{tab:scorerel}, App.~\ref{app:hetero}, Fig.~\ref{fig:hetero}) &
\raggedright \texttt{scripts/difficulty\_gtheory.py}, then again with \texttt{DATASET=cifar10}; figure: \texttt{figE1\_hetero.py} &
\raggedright \texttt{data/prune\_scores/}, \texttt{data/prune\_scores\_cifar10/} \tabularnewline
\raggedright Cross-seed correction (Section~\ref{sec:real}, Tab.~\ref{tab:pooledest}) &
\raggedright \texttt{scripts/real\_recover.py}, then again with \texttt{DATASET=cifar10} &
\raggedright \texttt{data/prune\_scores/}, \texttt{data/prune\_scores\_cifar10/} \tabularnewline
\raggedright Coverage calibrations (App.~\ref{app:hetero}) &
\raggedright \texttt{scripts/coverage\_hetero.py}, \texttt{scripts/cond\_coverage.py}, \texttt{scripts/boot\_calibrate.py} &
\raggedright -- \tabularnewline
\raggedright Finite-sample bias of the ICC (App.~\ref{app:bias}) &
\raggedright \texttt{scripts/icc\_bias.py} &
\raggedright -- \tabularnewline
\raggedright Per-class count (Section~\ref{sec:real}, App.~\ref{app:perclass}) &
\raggedright \texttt{scripts/prune\_cifar.py}, then \texttt{scripts/real\_disatt.py} &
\raggedright needs regenerated CIFAR-100 labels \tabularnewline
\raggedright Verdicts and the pooled comparison on real data (Tab.~\ref{tab:verdicts}, \ref{tab:pooled}; App.~\ref{app:alg2-proc}, \ref{app:verdicts}, \ref{app:pooled-baseline}; Fig.~\ref{fig:teaser}(c), \ref{fig:reversals}, \ref{fig:verdictmap}) &
\raggedright \texttt{python -m scripts.algo2\_simrank.run}; figures: \texttt{fig1\_teaser.py}, \texttt{fig2\_reversals.py}, \texttt{figG1\_verdictmap.py} &
\raggedright \texttt{results/algo2\_reported/} \tabularnewline
\raggedright Held-out wordings and orders, permutation null (Tab.~\ref{tab:verdicts}; App.~\ref{app:permnull}, \ref{app:transfer}) &
\raggedright \texttt{python -m scripts.algo2\_simrank.validate\_extended} &
\raggedright \texttt{results/algo2\_reported/} \tabularnewline
\raggedright Split-grid self-comparison (Tab.~\ref{tab:verdicts}, \ref{tab:pooled}; App.~\ref{app:selfcomp}; Fig.~\ref{fig:selfcomp}) &
\raggedright stored output; figure: \texttt{figG3\_selfcomp.py} &
\raggedright \texttt{results/algo2\_reported/\allowbreak cross\_harness\_native.json} \tabularnewline
\raggedright Simulated benchmarks with known truth (Tab.~\ref{tab:pooled}, App.~\ref{app:alg2-sim}) &
\raggedright \texttt{scripts/validate\_decisions.py} &
\raggedright \texttt{results/algo2\_reported/} \tabularnewline
\raggedright New ARC wordings (App.~\ref{app:transfer}, Fig.~\ref{fig:newwordings}) &
\raggedright \texttt{scripts/holdout\_wordings.py -{}-analyze-only}; figure: \texttt{figG2\_newwordings.py} &
\raggedright \texttt{data/llm\_holdout\_*.npy} \tabularnewline
\raggedright Level sensitivity of Algorithm~\ref{alg:perproto} (Tab.~\ref{tab:alg2alpha}) &
\raggedright \texttt{python -m scripts.algo2\_simrank.run -{}-alpha}, once per level &
\raggedright -- \tabularnewline
\raggedright $g(1)$ and $g(32)$ (Section~\ref{sec:gcoef}, Tab.~\ref{tab:llmscale}, Fig.~\ref{fig:gcoef}) &
\raggedright \texttt{scripts/llm\_gtheory.py}, \texttt{scripts/analyze\_scale.py}; figure: \texttt{figF1\_gcoef.py} &
\raggedright \texttt{data/llm\_decomp\_*.npy} \tabularnewline
\raggedright Beta-binomial check (App.~\ref{app:binary}) &
\raggedright \texttt{scripts/binary\_gtheory.py} &
\raggedright \texttt{data/llm\_decomp\_*.npy} \tabularnewline
\raggedright Shot-format control (Fig.~\ref{fig:shotformat}, Tab.~\ref{tab:shotformat}) &
\raggedright \texttt{scripts/shot\_format.py -{}-analyze-only}; figure: \texttt{fig3\_shotformat.py} &
\raggedright \texttt{data/llm\_chatshot\_*.npy}, \texttt{data/llm\_decomp\_*.npy} \tabularnewline
\raggedright Embeddings (Tab.~\ref{tab:embed}, App.~\ref{app:embed}) &
\raggedright \texttt{scripts/gauge\_probe.py} (retrains the three seeds), then \texttt{scripts/gauge\_probe2.py} &
\raggedright needs regenerated features \tabularnewline
\raggedright Schematics (Fig.~\ref{fig:overview}, \ref{fig:gate}, \ref{fig:alg2}) &
\raggedright \texttt{pdflatex} on \texttt{figscripts/\allowbreak schematics/\allowbreak schem\_overview.tex}, \texttt{schem\_gate.tex} and \texttt{schem\_alg2.tex} &
\raggedright values written in the sources, interval positions from \texttt{results/\allowbreak algo2\_reported/} \tabularnewline
\bottomrule
\end{tabular}
\caption{Most results run from the released files, and two need regenerated CIFAR-100 inputs. Each row gives a result, the command that produces it, and the data it reads or writes. Every Python command is prefixed with \texttt{uv run}, and a bare script path stands for \texttt{uv run python} followed by that path. \texttt{erel} stands for \texttt{cargo run -{}-release -p erel-cli -{}-}. Flags and seeds not shown here are in \texttt{REPRODUCIBILITY.md}. Figure scripts are in \texttt{figscripts/}, and the schematic sources in \texttt{figscripts/\allowbreak schematics/}. --: the command generates its own data. Stored output: the file is released, but no released script writes it.}
\label{tab:repro}
\end{table*}

\paragraph{Contents.}
The repository holds a Rust core with the gate and the exact-$F$ bounds of Algorithm~\ref{alg:gated} and the simulations of Table~\ref{tab:mse}.
The repository also holds Python scripts for the other analyses, the data files named in Table~\ref{tab:repro}, and one plotting script per data figure in \texttt{figscripts/}.
The three schematic figures are TikZ sources in \texttt{figscripts/\allowbreak schematics/}, and each compiles on its own with \texttt{pdflatex} and needs the \texttt{standalone} class and the \texttt{tikz} and \texttt{newtx} packages.
Python commands run from the repository root after \texttt{uv sync}, in the environment pinned by \texttt{uv.lock}.
Table~\ref{tab:repro} lists the command for each result it covers.
The repository's \texttt{REPRODUCIBILITY.md} gives the flags and seeds of every command, and \texttt{scripts/\allowbreak smoke\_test.py} runs shortened versions of the CPU Python commands.

\paragraph{What runs from the released files.}
The LLM correctness tensors are included, so the analyses of Algorithm~\ref{alg:perproto}, the LLM reliability analyses and the prompt controls run without a GPU.
The disattenuation simulations and the robustness sweep generate their own data on a CPU.
The per-example difficulty scores of all eight seeds are included for both datasets.
The score reliabilities and the cross-seed correction of Section~\ref{sec:real} therefore also run from the released files.
The split-grid self-comparison counts come from a stored output file (Table~\ref{tab:repro}), and no released script regenerates that file.
No released script computes the joint estimate of Table~\ref{tab:pooledest} either.
The moment estimator of Appendix~\ref{app:baselines}, applied to seeds 0--3 of the released scores, gives 0.931 on \mbox{CIFAR-100} and 0.930 on \mbox{CIFAR-10}.

\paragraph{What needs regenerated inputs.}
Two inputs are missing from the released files, the CIFAR-100 class labels and the embedding features.
The per-class count of Section~\ref{sec:real} needs the labels.
The script \texttt{scripts/\allowbreak prune\_cifar.py} downloads public CIFAR-100 and recreates the label cache, and \texttt{scripts/\allowbreak real\_disatt.py} then recomputes the count.
Table~\ref{tab:embed} needs the features.
The script \texttt{scripts/\allowbreak gauge\_probe.py} regenerates them by retraining the three ResNet-18 seeds on a GPU (Appendix~\ref{app:embed}).

\paragraph{LLM checkpoints.}
We did not record the exact revisions of the model checkpoints.
The correctness tensors are therefore the reproducibility target for every LLM result.

\paragraph{Algorithm~\ref{alg:gated} in the code.}
Every reported number for Algorithm~\ref{alg:gated} uses the exact-$F$ branch, and the released implementation of Algorithm~\ref{alg:gated} has only that branch.
It lacks the in-line variance--mean check, the unit-bootstrap bound of the max rule and the Spearman--Brown map behind the check in Appendix~\ref{app:pooled}.
Stand-alone versions of the first two are in \texttt{scripts/\allowbreak difficulty\_gtheory.py} and \texttt{scripts/\allowbreak boot\_calibrate.py}.

\paragraph{Compute and data.}
LLM scoring, the difficulty scores and the embeddings ran on one consumer GPU.
The simulations and the statistical analyses run on a CPU.
All datasets and model checkpoints are public, and we collected no new data and no data about people.


{\small\begin{thebibliography}{99}
\bibitem[Angrist \& Krueger(1995)]{angrist1995} J.~D.~Angrist and A.~B.~Krueger. Split-sample instrumental variables estimates of the return to schooling. \emph{Journal of Business \& Economic Statistics}, 13(2):225--235, 1995.
\bibitem[Avelar Menendez et al.(2026)]{avelar2026} A.~R.~Avelar Menendez, Y.~Liu, and X.~Dai. Prompt-dependent ranking of large language models with uncertainty quantification. arXiv:2603.03336, 2026.
\bibitem[Bancroft(1944)]{bancroft1944} T.~A.~Bancroft. On biases in estimation due to the use of preliminary tests of significance. \emph{Annals of Mathematical Statistics}, 15(2):190--204, 1944.
\bibitem[Baranchik(1964)]{baranchik1964} A.~J.~Baranchik. Multiple regression and estimation of the mean of a multivariate normal distribution. Technical Report 51, Department of Statistics, Stanford University, 1964.
\bibitem[Ben Allal et al.(2025)]{benallal2025} L.~Ben Allal et al. SmolLM2: when smol goes big -- data-centric training of a small language model. arXiv:2502.02737, 2025.
\bibitem[Berger(1982)]{berger1982} R.~L.~Berger. Multiparameter hypothesis testing and acceptance sampling. \emph{Technometrics}, 24(4):295--300, 1982.
\bibitem[Bound et al.(1995)]{bound1995} J.~Bound, D.~A.~Jaeger, and R.~M.~Baker. Problems with instrumental variables estimation when the correlation between the instruments and the endogenous explanatory variable is weak. \emph{Journal of the American Statistical Association}, 90(430):443--450, 1995.
\bibitem[Bouthillier et al.(2021)]{bouthillier2021} X.~Bouthillier et al. Accounting for variance in machine learning benchmarks. \emph{MLSys}, 2021.
\bibitem[Bradley \& Terry(1952)]{bradley1952} R.~A.~Bradley and M.~E.~Terry. Rank analysis of incomplete block designs: I. The method of paired comparisons. \emph{Biometrika}, 39(3/4):324--345, 1952.
\bibitem[Brennan(2001)]{brennan2001} R.~L.~Brennan. \emph{Generalizability Theory}. Springer, 2001.
\bibitem[Brown(1910)]{brown1910} W.~Brown. Some experimental results in the correlation of mental abilities. \emph{British Journal of Psychology}, 3(3):296--322, 1910.
\bibitem[Carroll et al.(2006)]{carroll2006} R.~J.~Carroll, D.~Ruppert, L.~A.~Stefanski, and C.~M.~Crainiceanu. \emph{Measurement Error in Nonlinear Models}. 2nd ed., Chapman \& Hall, 2006.
\bibitem[Clark et al.(2018)]{clark2018} P.~Clark et al. Think you have solved question answering? Try ARC, the AI2 Reasoning Challenge. arXiv:1803.05457, 2018.
\bibitem[Cronbach et al.(1972)]{cronbach1972} L.~J.~Cronbach, G.~C.~Gleser, H.~Nanda, and N.~Rajaratnam. \emph{The Dependability of Behavioral Measurements: Theory of Generalizability for Scores and Profiles}. Wiley, 1972.
\bibitem[D'Amour et al.(2022)]{damour2022} A.~D'Amour et al. Underspecification presents challenges for credibility in modern machine learning. \emph{Journal of Machine Learning Research}, 23(226):1--61, 2022.
\bibitem[Ding et al.(2021)]{ding2021} F.~Ding, J.-S.~Denain, and J.~Steinhardt. Grounding representation similarity through statistical testing. \emph{NeurIPS}, 2021.
\bibitem[Dror et al.(2018)]{dror2018} R.~Dror, G.~Baumer, S.~Shlomov, and R.~Reichart. The hitchhiker's guide to testing statistical significance in natural language processing. \emph{ACL}, pp.~1383--1392, 2018.
\bibitem[Dunn(1961)]{dunn1961} O.~J.~Dunn. Multiple comparisons among means. \emph{Journal of the American Statistical Association}, 56(293):52--64, 1961.
\bibitem[Falcon-LLM Team(2024)]{falcon2024} Falcon-LLM Team. The Falcon 3 family of open models. Technology Innovation Institute, \url{https://huggingface.co/blog/falcon3}, December 2024.
\bibitem[Fieller(1954)]{fieller1954} E.~C.~Fieller. Some problems in interval estimation. \emph{Journal of the Royal Statistical Society, Series B}, 16(2):175--185, 1954.
\bibitem[Fuller(1987)]{fuller1987} W.~A.~Fuller. \emph{Measurement Error Models}. Wiley, 1987.
\bibitem[Gail \& Simon(1985)]{gailsimon1985} M.~Gail and R.~Simon. Testing for qualitative interactions between treatment effects and patient subsets. \emph{Biometrics}, 41(2):361--372, 1985.
\bibitem[Gao et al.(2023)]{gao2023} L.~Gao et al. A framework for few-shot language model evaluation. Version v0.4.0, Zenodo, doi:10.5281/zenodo.10256836, 2023.
\bibitem[Grattafiori et al.(2024)]{grattafiori2024} A.~Grattafiori et al. The Llama 3 herd of models. arXiv:2407.21783, 2024.
\bibitem[He et al.(2016)]{he2016} K.~He, X.~Zhang, S.~Ren, and J.~Sun. Deep residual learning for image recognition. \emph{CVPR}, pp.~770--778, 2016.
\bibitem[He et al.(2024)]{he2024} M.~He, S.~Yang, T.~Huang, and B.~Zhao. Large-scale dataset pruning with dynamic uncertainty. \emph{CVPR Workshops}, pp.~7713--7722, 2024. arXiv:2306.05175.
\bibitem[Hendrycks et al.(2021)]{hendrycks2021} D.~Hendrycks et al. Measuring massive multitask language understanding. \emph{ICLR}, 2021.
\bibitem[Holm(1979)]{holm1979} S.~Holm. A simple sequentially rejective multiple test procedure. \emph{Scandinavian Journal of Statistics}, 6(2):65--70, 1979.
\bibitem[Hothorn et al.(2008)]{hothorn2008} T.~Hothorn, F.~Bretz, and P.~Westfall. Simultaneous inference in general parametric models. \emph{Biometrical Journal}, 50(3):346--363, 2008.
\bibitem[Jacobs \& Wallach(2021)]{jacobs2021} A.~Z.~Jacobs and H.~Wallach. Measurement and fairness. \emph{FAccT}, 2021.
\bibitem[Kinal(1980)]{kinal1980} T.~W.~Kinal. The existence of moments of k-class estimators. \emph{Econometrica}, 48(1):241--249, 1980.
\bibitem[Kornblith et al.(2019)]{kornblith2019} S.~Kornblith, M.~Norouzi, H.~Lee, and G.~Hinton. Similarity of neural network representations revisited. \emph{ICML}, 2019.
\bibitem[Kriegeskorte et al.(2008)]{kriegeskorte2008} N.~Kriegeskorte, M.~Mur, and P.~Bandettini. Representational similarity analysis -- connecting the branches of systems neuroscience. \emph{Frontiers in Systems Neuroscience}, 2:4, 2008.
\bibitem[Krizhevsky(2009)]{krizhevsky2009} A.~Krizhevsky. Learning multiple layers of features from tiny images. Technical report, University of Toronto, 2009.
\bibitem[Li et al.(2016)]{li2016} Y.~Li, J.~Yosinski, J.~Clune, H.~Lipson, and J.~Hopcroft. Convergent learning: do different neural networks learn the same representations? \emph{ICLR}, 2016.
\bibitem[Liang et al.(2023)]{liang2023} P.~Liang et al. Holistic evaluation of language models. \emph{Transactions on Machine Learning Research}, 2023.
\bibitem[Maia Polo et al.(2024)]{polo2024} F.~Maia Polo et al. Efficient multi-prompt evaluation of LLMs. \emph{NeurIPS}, 2024.
\bibitem[Mart\'inez-Plumed et al.(2019)]{martinez2019} F.~Mart\'inez-Plumed, R.~B.~C.~Prud\^encio, A.~Mart\'inez-Us\'o, and J.~Hern\'andez-Orallo. Item response theory in AI: analysing machine learning classifiers at the instance level. \emph{Artificial Intelligence}, 271:18--42, 2019.
\bibitem[McCoy et al.(2020)]{mccoy2020} R.~T.~McCoy, J.~Min, and T.~Linzen. BERTs of a feather do not generalize together: large variability in generalization across models with similar test set performance. \emph{BlackboxNLP}, pp.~217--227, 2020.
\bibitem[McGraw \& Wong(1996)]{mcgraw1996} K.~O.~McGraw and S.~P.~Wong. Forming inferences about some intraclass correlation coefficients. \emph{Psychological Methods}, 1(1):30--46, 1996.
\bibitem[Meta(2024)]{meta2024} Meta. Llama 3.2: revolutionizing edge AI and vision with open, customizable models. \url{https://ai.meta.com/blog/llama-3-2-connect-2024-vision-edge-mobile-devices/}, September 2024.
\bibitem[Mihaylov et al.(2018)]{mihaylov2018} T.~Mihaylov, P.~Clark, T.~Khot, and A.~Sabharwal. Can a suit of armor conduct electricity? A new dataset for open book question answering. \emph{EMNLP}, pp.~2381--2391, 2018.
\bibitem[Miller(2024)]{miller2024} E.~Miller. Adding error bars to evals: a statistical approach to language model evaluations. arXiv:2411.00640, 2024.
\bibitem[Mizrahi et al.(2024)]{mizrahi2024} M.~Mizrahi et al. State of what art? A call for multi-prompt LLM evaluation. \emph{Transactions of the Association for Computational Linguistics}, 12:933--949, 2024.
\bibitem[Mogstad et al.(2024)]{mogstad2024} M.~Mogstad, J.~P.~Romano, A.~M.~Shaikh, and D.~Wilhelm. Inference for ranks with applications to mobility across neighbourhoods and academic achievement across countries. \emph{Review of Economic Studies}, 91(1):476--518, 2024.
\bibitem[Paul et al.(2021)]{paul2021} M.~Paul, S.~Ganguli, and G.~K.~Dziugaite. Deep learning on a data diet: finding important examples early in training. \emph{NeurIPS}, 2021.
\bibitem[Picard(2021)]{picard2021} D.~Picard. \texttt{torch.manual\_seed(3407)} is all you need: on the influence of random seeds in deep learning architectures for computer vision. arXiv:2109.08203, 2021.
\bibitem[Rodriguez et al.(2021)]{rodriguez2021} P.~Rodriguez et al. Evaluation examples are not equally informative: how should that change NLP leaderboards? \emph{ACL}, 2021.
\bibitem[Sch\"onemann(1966)]{schonemann1966} P.~H.~Sch\"onemann. A generalized solution of the orthogonal Procrustes problem. \emph{Psychometrika}, 31(1):1--10, 1966.
\bibitem[Sclar et al.(2024)]{sclar2024} M.~Sclar, Y.~Choi, Y.~Tsvetkov, and A.~Suhr. Quantifying language models' sensitivity to spurious features in prompt design. \emph{ICLR}, 2024.
\bibitem[Searle et al.(1992)]{searle1992} S.~R.~Searle, G.~Casella, and C.~E.~McCulloch. \emph{Variance Components}. Wiley, 1992.
\bibitem[Sellam et al.(2022)]{sellam2022} T.~Sellam et al. The MultiBERTs: BERT reproductions for robustness analysis. \emph{ICLR}, 2022.
\bibitem[Shrout \& Fleiss(1979)]{shrout1979} P.~E.~Shrout and J.~L.~Fleiss. Intraclass correlations: uses in assessing rater reliability. \emph{Psychological Bulletin}, 86(2):420--428, 1979.
\bibitem[\v{S}id\'ak(1967)]{sidak1967} Z.~\v{S}id\'ak. Rectangular confidence regions for the means of multivariate normal distributions. \emph{Journal of the American Statistical Association}, 62(318):626--633, 1967.
\bibitem[Sorscher et al.(2022)]{sorscher2022} B.~Sorscher, R.~Geirhos, S.~Shekhar, S.~Ganguli, and A.~S.~Morcos. Beyond neural scaling laws: beating power-law scaling via data pruning. \emph{NeurIPS}, 2022.
\bibitem[Spearman(1904)]{spearman1904} C.~Spearman. The proof and measurement of association between two things. \emph{American Journal of Psychology}, 15(1):72--101, 1904.
\bibitem[Spearman(1910)]{spearman1910} C.~Spearman. Correlation calculated from faulty data. \emph{British Journal of Psychology}, 3(3):271--295, 1910.
\bibitem[Staiger \& Stock(1997)]{staiger1997} D.~O.~Staiger and J.~H.~Stock. Instrumental variables regression with weak instruments. \emph{Econometrica}, 65(3):557--586, 1997.
\bibitem[Swayamdipta et al.(2020)]{swayamdipta2020} S.~Swayamdipta et al. Dataset cartography: mapping and diagnosing datasets with training dynamics. \emph{EMNLP}, 2020.
\bibitem[Toneva et al.(2019)]{toneva2019} M.~Toneva et al. An empirical study of example forgetting during deep neural network learning. \emph{ICLR}, 2019.
\bibitem[Waudby-Smith \& Ramdas(2024)]{waudbysmith2024} I.~Waudby-Smith and A.~Ramdas. Estimating means of bounded random variables by betting. \emph{Journal of the Royal Statistical Society, Series B}, 86(1):1--27, 2024.
\bibitem[Westfall \& Young(1993)]{westfall1993} P.~H.~Westfall and S.~S.~Young. \emph{Resampling-Based Multiple Testing: Examples and Methods for p-Value Adjustment}. Wiley, 1993.
\bibitem[Williams et al.(2021)]{williams2021} A.~H.~Williams, E.~Kunz, S.~Kornblith, and S.~W.~Linderman. Generalized shape metrics on neural representations. \emph{NeurIPS}, 2021.
\bibitem[Yang et al.(2024a)]{yang2024a} A.~Yang et al. Qwen2 technical report. arXiv:2407.10671, 2024.
\bibitem[Yang et al.(2024b)]{yang2024b} A.~Yang et al. Qwen2.5 technical report. arXiv:2412.15115, 2024.
\bibitem[Zellers et al.(2019)]{zellers2019} R.~Zellers, A.~Holtzman, Y.~Bisk, A.~Farhadi, and Y.~Choi. HellaSwag: can a machine really finish your sentence? \emph{ACL}, pp.~4791--4800, 2019.
\bibitem[Zheng et al.(2024)]{zheng2024} C.~Zheng, H.~Zhou, F.~Meng, J.~Zhou, and M.~Huang. Large language models are not robust multiple choice selectors. \emph{ICLR}, 2024.
\bibitem[Zimmerman \& Williams(1997)]{zimmerman1997} D.~W.~Zimmerman and R.~H.~Williams. Properties of the Spearman correction for attenuation for normal and realistic non-normal distributions. \emph{Applied Psychological Measurement}, 21(3):253--270, 1997.
\end{thebibliography}
}
\end{document}